\documentclass{article}
\usepackage{graphicx} 
\usepackage{enumerate}
\usepackage{url} 
\usepackage{boxinz-macros}
\usepackage{authblk}
\usepackage{dsfont}
\usepackage{float}

\usepackage[colorlinks]{hyperref}
\usepackage{natbib}
\usepackage{xcolor}
\usepackage{placeins}
\usepackage{enumitem}
\usepackage{amsmath}
\usepackage{xspace}
\usepackage{mathabx}
\usepackage[nodisplayskipstretch]{setspace}
\usepackage{subcaption}

\usepackage[colorlinks]{hyperref}

\usepackage{xcolor}

\usepackage{placeins}

\usepackage{enumitem}

\usepackage[compact]{titlesec}
\usepackage{bibunits}
\usepackage{wrapfig}

\usepackage{amsmath}

\usepackage{xspace}

\newcommand{\myalg}{Trans-Glasso\xspace}
\newcommand{\myalgMT}{Trans-MT-Glasso\xspace}
\usepackage{mathabx}
\newcommand{\shared}{\Omega^{\star}}
\newcommand{\unique}{\Gamma^{(k)\star}}
\newcommand{\supp}{\textrm{supp}}
\newcommand{\diff}{\Psi^{(k)}}
\newcommand{\diffestimate}{\widehat{\Psi}^{(k)}}
\newcommand{\uniqueestimate}{\widehat{\Gamma}^{(k)}}
\newcommand{\steponeoutput}{\widecheck{\Omega}^{(k)}}
\newcommand{\uniqueparameter}{\Gamma^{(k)}}
\newcommand{\supportshared}{\mathcal{S}_{\Omega}}
\newcommand{\supportunique}{\mathcal{S}_{\Gamma^{(k)}}}
\newcommand{\diffproj}{\widehat{\Psi}^{(k)}_{\mathrm{proj}}}
\newcommand{\targetproj}{\widehat{\Omega}^{(0)}_{\mathrm{proj}}}

\usepackage{color}
\definecolor{cm}{RGB}{0,0,200}

\newcommand{\newText}{}
\newcommand{\newTextTwo}{}

\allowdisplaybreaks
\title{\myalg: A Transfer Learning Approach to \\ Precision Matrix Estimation}

\author[1]{Boxin Zhao}
\author[2]{Cong Ma}
\author[3]{Mladen Kolar}

\affil[1]{Booth School of Business, University of Chicago}
\affil[2]{Department of Statistics, University of Chicago}
\affil[3]{Department of Data Sciences and Operations, Marshall School of Business, University of Southern California}

\date{}

\begin{document}

\begin{bibunit}[plainnat]

\maketitle

\begin{abstract}
\normalsize

Precision matrix estimation is essential in various fields; yet, it is challenging when samples for the target study are limited. Transfer learning can enhance estimation accuracy by leveraging data from related source studies. We propose \myalg, a two-step transfer learning method for precision matrix estimation. First, we obtain initial estimators using a multi-task learning objective that captures both shared and unique features across studies. Then, we refine these estimators through differential network estimation to adjust for structural differences between the target and source precision matrices. Under the assumption that most entries of the target precision matrix are shared with those of the source matrices, we derive non-asymptotic error bounds and show that \myalg achieves minimax optimality under certain conditions. Extensive simulations demonstrate \myalg\!\!’s superior performance compared to baseline methods, particularly in small-sample settings. We further validate \myalg in applications to gene networks across brain tissues and protein networks for various cancer subtypes, showcasing its effectiveness in biological contexts. Additionally, we derive the minimax optimal rate for differential network estimation, representing the first such guarantee in this area.
The Python implementation of \myalg, along with code to reproduce all experiments in this paper, is publicly available at \url{https://github.com/boxinz17/transglasso-experiments}.

\end{abstract}


\section{Introduction}

Estimating the precision matrix, i.e., the inverse covariance matrix, is a fundamental task in statistical analysis and has broad applications, including in portfolio optimization, speech recognition, and genomics~\citep{lauritzen1996graphical}. 
The precision matrix is closely tied to Gaussian graphical models: estimating the support of the precision matrix corresponds to uncovering the network structure of conditional dependencies between multivariate normal variables~\citep{lauritzen1996graphical}.
However, estimating a precision matrix accurately is often challenging when the sample size is small compared to the dimension---a typical scenario in high-dimensional settings.

In many applications, the target study has a limited sample size, while data from related studies may be available.
Transfer learning~\citep{pan2009survey} provides a promising approach in these scenarios by leveraging information from related source studies to improve estimation accuracy in the target study. 
For example, in gene expression studies across different tissue types, sample sizes may be small for specific tissues, but data from related tissues can help improve estimates~\citep{li2023transfer}. 
Similarly, protein network studies for different cancer subtypes can benefit from transfer learning, as leveraging data from related subtypes can enhance estimation for a particular subtype with limited data~\citep{peterson2015bayesian}.

A critical aspect of transfer learning is establishing similarity between the target and source tasks. Here, we assume that most entries of the target precision matrix are shared with those of the source matrices, with only a few differences. Based on this assumption, we propose \myalg, a novel two-step transfer learning method for precision matrix estimation. First, we obtain initial estimators through a multi-task learning objective that captures shared and unique dependencies across datasets. Second, we refine these estimators using differential network estimation to adjust for differences between the target and source matrices~\citep{zhao2014direct, yuan2017differential}.

We provide a theoretical analysis of \myalg, deriving non-asymptotic error bounds and establishing that the method achieves minimax optimality in a wide range of parameter settings. Through extensive simulations, we demonstrate that \myalg~outperforms several baseline methods, particularly in scenarios where the target sample size is small. We also apply \myalg~to gene networks across brain tissues and protein networks for various cancer subtypes, showing its practical effectiveness in biological applications. Additionally, as a byproduct of our analysis, we derive the minimax optimal rate for differential network estimation---to our knowledge, the first such guarantee in this area.

\subsection{Related Work}

\paragraph{Precision matrix estimation.} Estimation of sparse precision matrices in a single study is well studied. Common methods include penalized M-estimator~\citep{yuan2007model,friedman2008sparse,rothman2008sparse,lam2009sparsistency,ravikumar2011high} and constrained $L_1$ minimization~\citep{cai2011constrained,cai2016estimating,ren2015asymptotic}. There is also extensive literature on multi-task precision matrix estimation, which estimates multiple related but non-identical precision matrices from multiple studies~\citep{guo2011joint, danaher2014joint,zhu2014structural}. See~\cite{tsai2022joint} for a survey. While related to transfer learning, multi-task learning aims to estimate parameters of all studies, whereas transfer learning focuses solely on the target study.

\paragraph{Transfer learning.} Transfer learning has a long history~\citep{pan2009survey}. Recently, interest in transfer learning for statistical problems has grown. \citet{li2022transfer, he2024transfusion} studied high-dimensional linear regression. The fused regularizer in~\cite{he2024transfusion} is similar to our multi-task objective; however, \citet{he2024transfusion} focuses on linear regression, while we focus on precision matrix estimation, leading to distinct technical analyses. \citet{li2023estimation, tian2023transfer} studied high-dimensional generalized linear regression. 
\citet{cai2021transfer} studied nonparametric classification and \citet{cai2024transfer} investigated multi-armed bandit problems. 
\citet{lin2022transfer,cai2024transferfunctional} studied transfer learning for functional data analysis.
\citet{kpotufe2021marginal, pathak2022new, ma2023optimally, ge2023maximum, pathak2024design} addressed covariate shift---a special form of transfer learning---under different statistical models. 

The most relevant work to this paper is~\cite{li2023transfer}, which also studied transfer learning for precision matrix estimation. The key difference is the similarity assumption. \cite{li2023transfer} assumes that the divergence matrices between the target and source precision matrices are sparse. We assume that most entries of the target precision matrix are shared across source precision matrices, with few different entries. Although the assumption in \cite{li2023transfer} is motivated by the KL divergence between Gaussian distributions, ours is a structural assumption, making it applicable beyond Gaussian data and easier to interpret. Consequently, our method differs significantly from that of \cite{li2023transfer}.

\paragraph{Differential Network Estimation.} Our approach leverages differential network estimation techniques, which aim to directly estimate the difference between two precision matrices without the need to estimate the individual ones~\citep{zhao2014direct,yuan2017differential,liu2014direct,ma2021inter}.
\citet{fazayeli2016generalized} explored this concept in the context of Ising models. Additionally, \citet{zhao2019direct, zhao2022fudge} extended differential network estimation methods to functional data, while \citet{tugnait2023estimation} broadened its application to multi-attribute data.

\subsection{Organization and Notation}

The rest of the paper is organized as follows. 
In Section~\ref{sec:ProblemSetup}, we introduce the problem setup.
In Section~\ref{sec:TransGlassoAlgorithm}, we introduce the methodology of the paper.
Section~\ref{sec:implementation} describes practical implementation details of our method.
The theoretical results are developed in Section~\ref{sec:theoreticalAnalysis}.
We also conduct extensive simulation experiments in Section~\ref{sec:simulation}.
Furthermore, in Section~\ref{sec:real-world-data}, we apply our method to two real-world datasets.
Finally, we conclude our paper with Section~\ref{sec:conclusion}.
The technical proofs and details about optimization algorithms are provided in the appendix.
The Python implementation of our method, along with code to reproduce all experiments in this paper, is publicly available at
\begin{center}
\url{https://github.com/boxinz17/transglasso-experiments}
\end{center}

\paragraph{Notation.}

For a vector $v \in \mathbb{R}^d$, we use $\Vert v \Vert_p$ to denote its $L_p$-norm.
More specifically, we have $\Vert v \Vert_p=(\sum^d_{i=1} \vert v_i \vert^p)^{\frac{1}{p}}$ for $1 \leq p \leq \infty$, where $\Vert v \Vert_{\infty}=\max_i \vert v_i \vert$. For a matrix $A \in \mathbb{R}^{d \times d}$, we use $\vert \cdot \vert$ to denote its elementwise norm and $\Vert \cdot \Vert$ to denote its operator norm. For example, $\vert A \vert_1=\sum^d_{i=1} \sum^d_{j=1} \vert A_{ij} \vert$, $\vert A \vert_0=\sum^d_{i=1} \sum^d_{j=1} \mathds{1} \left\{ A_{ij} \neq 0 \right\}$, $\vert A \vert_{\infty}=\max_{1 \leq i,j \leq d} \vert A_{ij} \vert$; $\Vert A \Vert_1=\max_{1 \leq j \leq d}\sum^d_{i=1} \vert A_{ij} \vert$, $\Vert A \Vert_{\infty}=\max_{1 \leq i \leq d} \sum^d_{j=1} \vert A_{ij} \vert$, and $\Vert A \Vert_2$ denote the largest singular value of $A$. We use $\Vert A \Vert_{\mathrm{F}}=(\sum^d_{i=1} \sum^d_{j=1} \vert A_{ij} \vert^2)^{1/2}$ to denote the Frobenius norm of $A$. In addition, we use $\langle A,B\rangle=\tr(A^{\top}B) = \sum_{i,j}A_{ij}B_{ij}$ for $A,B\in \mathbb{R}^{d \times d}$ to define the inner product between two matrices. We use $\text{vec}(A)$ to denote the $d^2$-vector obtained by stacking the columns of $A$. When $A$ is symmetric, we let $\gamma_{\min}(A)$ and $\gamma_{\max}(A)$ denote its smallest and largest eigenvalue. For $A \in \mathbb{R}^{n_1 \times n_2}$ and $B \in \mathbb{R}^{m_1 \times m_2}$, we let $A \otimes B = [A_{ij}B_{lm}]_{i,j,l,m} \in \mathbb{R}^{n_1 m_1 \times n_2 m_2}$ denote the Kronecker product of two matrices. We define $\mathbb{S}^{d \times d}$ as the set of symmetric matrices with dimension $d$. The universal constants may vary from one line to another without further clarification.
Finally, following~\cite{ravikumar2011high}, we say a random vector $X \in \mathbb{R}^d$ with $\mathbb{E}[X]=0$ and $\Sigma=\textrm{Cov}(X)$ is sub-Gaussian if there exists a constant $\sigma>0$ such that
\begin{equation*}
\mathbb{E}\left[ \exp \left( \lambda X_j /\sqrt{\Sigma_{jj}} \right) \right] \leq \exp \left( \sigma^2 \lambda^2 / 2 \right) \quad \text{for all } \lambda \in \mathbb{R} \text{ and } 1 \leq j \leq d.
\end{equation*}

In addition, we use the following standard notation in the paper. For two positive sequences $\{ f(n) \}_{n \geq 1}$ and $\{ g(n) \}_{n \geq 1}$, $f(n)=O(g(n))$ or $f(n) \lesssim g(n)$ means that there exists a universal constant $c>0$ such that $f(n) \leq c g(n)$ holds for sufficiently large $n$; $f(n)=\Omega(g(n))$ or $f(n) \gtrsim g(n)$ means that there exists a universal constant $c>0$ such that $f(n) \geq c g(n)$ holds for sufficiently large $n$; $f(n) = \Theta (g(n))$ or $f(n) \asymp g(n)$ means that there exist universal constants $c_1,c_2>0$ such that $c_1 g(n) \leq f(n) \leq c_2 g(n)$ holds for sufficiently large $n$; $f(n)=o(g(n))$ indicates that $f(n)/g(n) \to 0$ as $n \to \infty$.

\section{Problem Setup}

\label{sec:ProblemSetup}

Imagine that we observe $n_0$ i.i.d.~samples $\{\mathbf{x}^{(0)}_i\}^{n_0}_{i=1} \coloneqq \mathcal{D}_0$ from a sub-Gaussian target distribution $\mathcal{P}_0$. Each sample $\mathbf{x}^{(0)}_i \in \mathbb{R}^d$ is assumed to have zero mean and covariance matrix $\Sigma^{(0)}$. 
Our goal is to estimate the target precision matrix $\Omega^{(0)} = ( \Sigma^{(0)} )^{-1}$. Additionally, we have access to $K$ sub-Gaussian source distributions $\{\mathcal{P}_k\}^K_{k=1}$, each with $n_k$ i.i.d.~samples $\{\mathbf{x}^{(k)}_i\}^{n_k}_{i=1} \coloneqq \mathcal{D}_k$. For $1 \leq k \leq K$, each $\mathbf{x}^{(k)}_i \in \mathbb{R}^d$ also has zero mean, covariance matrix $\Sigma^{(k)}$, and corresponding precision matrix $\Omega^{(k)} = ( \Sigma^{(k)} )^{-1}$. The goal is to leverage samples from both the target and source distributions to accurately estimate $\Omega^{(0)}$.

To facilitate transfer learning, we assume structural similarity between the target and source precision matrices, whereby most entries in $\Omega^{(0)}$ are shared with those in $\Omega^{(k)}$, with relatively few differences. This assumption enables us to efficiently utilize source samples to enhance the estimation of the target precision matrix.
Formally, we characterize the relationship between the target and source precision matrices by the following assumption.

\begin{assump}
\label{assump:model-structure}
For each $0 \leq k \leq K$, there exists a shared component $\shared$ and a unique component $\unique$ with disjoint supports such that 
\begin{equation}
\Omega^{(k)} = \shared + \unique,
\end{equation}
where $|\shared|_0 \leq s$, $|\unique|_0 \leq h$, and $|\unique|_1 \leq M_\Gamma$. The parameters $s$ and $h$ satisfy $s \gg h$, indicating that the majority of the structure is shared, while unique components are minimal.
\end{assump}

\medskip

See Figure~\ref{fig:demonstration-assumption-1} for a visual illustration of Assumption~\ref{assump:model-structure}.
Assumption~\ref{assump:model-structure} is inspired by the assumptions widely used in the differential network estimation literature~\citep{zhao2014direct}.  
To see this connection, for each $1 \leq k \leq K$, define  $\diff=\Omega^{(k)}-\Omega^{(0)}$ to be the differential network between the target $\Omega^{(0)}$ and the source $\Omega^{(k)}$. Two immediate implications of  Assumption~\ref{assump:model-structure} are 
\begin{equation*}
\left\vert \diff \right\vert_0 \leq  2 h, \qquad \text{and} \qquad 
\left\vert \diff \right\vert_1 \leq  2 M_{\Gamma}, \quad \text{for all } 1 \leq k \leq K.
\end{equation*}
This  exactly matches Condition~1 in the paper~\citep{zhao2014direct}. 

{\newTextTwo 

We also note that the structural similarity in Assumption~1 is conceptually related to the heterogeneity assumption in \cite{bastani2021predicting}, which models the discrepancy between proxy and true predictive tasks as a sparse bias; both frameworks exploit the idea that differences across tasks are concentrated in a small set of coordinates, thereby enabling efficient transfer.

}

In addition, Assumption~\ref{assump:model-structure} is naturally interpretable within Gaussian graphical models. Suppose we have an undirected graphical model \( G = (V, E) \) where nodes represent variables and edges represent conditional dependencies. In this model, an edge exists between nodes \( i \) and \( j \) if and only if \( \Omega^{(k)}_{ij} \neq 0 \)~\citep{lauritzen1996graphical}. Under Assumption~\ref{assump:model-structure}, the target precision matrix \( \Omega^{(0)} \) and the source matrices \( \left\{ \Omega^{(k)} \right\} \) share a large subset of edges, with only a small number of unique edges in each source, corresponding to sparse deviations \( \diff \).

\begin{figure}[t]
\centering
\includegraphics[width=0.8\linewidth]{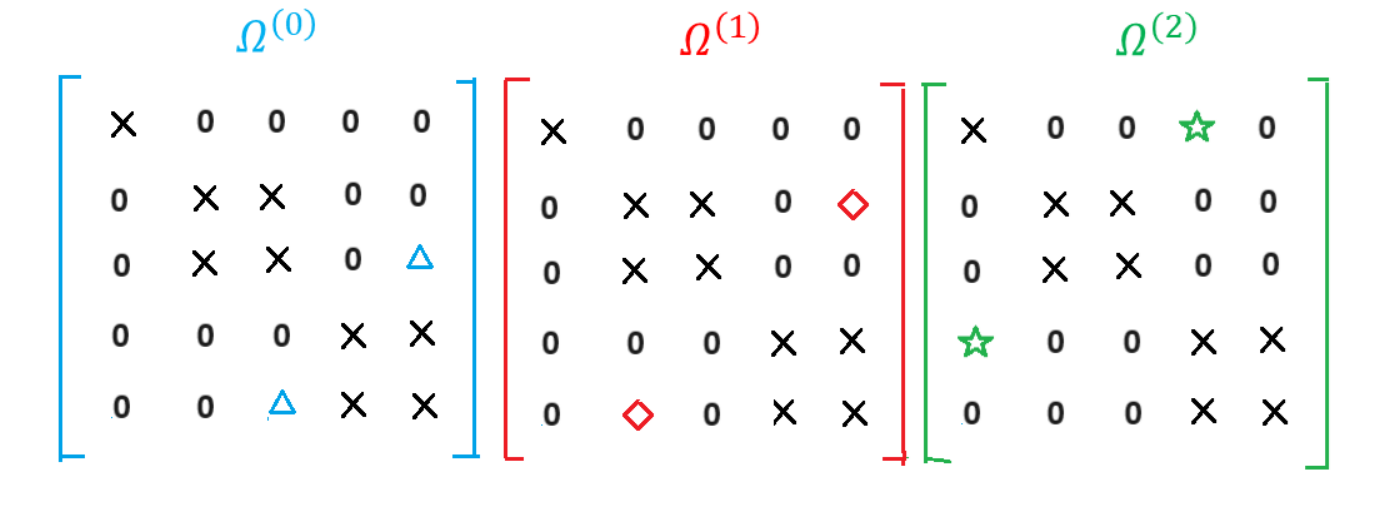}
\vspace{-0.5cm}
\caption{Illustration of Assumption~\ref{assump:model-structure}. The target precision matrix, $\Omega^{(0)}$,  is shown alongside two source precision matrices, $\Omega^{(1)}$ and $\Omega^{(2)}$. Black crosses represent the shared entries across the matrices, while colored shapes indicate individual, unique entries.}
\label{fig:demonstration-assumption-1}
\end{figure}

{\newText

Finally, we compare Assumption~\ref{assump:model-structure} with the similarity assumption based on divergence measures adopted in~\cite{li2023transfer}. Specifically, \citet{li2023transfer} assume that the divergence matrix $\Upsilon^{(k)} = \Omega^{(0)} \Sigma^{(k)} - I_d$ is sparse, a condition motivated by controlling the KL divergence between two multivariate Gaussian distributions. In contrast, our structural assumption does not rely on Gaussian-specific properties and naturally extends beyond Gaussian graphical models, providing a more intuitive interpretation of shared structure. We emphasize that we do not claim \myalg consistently outperforms Trans-CLIME, the method proposed in~\cite{li2023transfer}, and we leave a detailed theoretical comparison to future work. Further discussion is provided in Appendix~\ref{sec:comparison-li2023}.
}

\section{\myalg~Algorithm}
\label{sec:TransGlassoAlgorithm}

In this section, we introduce \myalg, a transfer learning method for precision matrix estimation. It comprises two steps: (1) initializing precision matrix estimators via a \mbox{multi-task} learning problem and (2) refining them with differential network estimation to account for structural differences between target and source matrices.

\subsection{Initialization via Multi-Task Learning}

To leverage shared structure across the target and source matrices, we begin by jointly estimating precision matrices for both the target and source distributions. Based on Assumption~\ref{assump:model-structure}, we employ a multi-task variant of the graphical lasso estimator~\citep{friedman2008sparse}, which we refer to as \myalgMT (\textbf{Trans}fer \textbf{M}ulti-\textbf{T}ask \textbf{G}raphical \textbf{lasso}).

Let $\widehat{\Sigma}^{(k)} =\frac{1}{n_k} \sum^{n_k}_{i=1} \mathbf{x}^{(k)}_i \mathbf{x}^{(k) \top}_i$ denote the sample covariance matrix for $0 \leq k \leq K$. We define $\Theta \coloneqq ( \Omega, \left\{ \uniqueparameter \right\}^K_{k=0} )$ for the shared component and the sparse unique components. The \myalgMT objective is then given by
{
\begin{equation}
\label{eq:Glasso-mutli-task}
\widehat{\Theta} = \left( \widehat{\Omega}, \{ \uniqueestimate \}^K_{k=0} \right) \in \,\,\arg \min_{ \Theta \in \mathcal{C} \left( M_{\mathrm{op}} \right) } \left\{ \mathcal{L} \left( \Theta \right) + \lambda_{\text{M}} \Phi \left( \Theta \right)  \right\},
\end{equation}
where
\begin{align*}
\mathcal{L} \left( \Theta \right) & \coloneqq \sum^K_{k=0 } \alpha_k \left\{ \left\langle \Omega + \uniqueparameter, \widehat{\Sigma}^{(k)} \right\rangle - \log \det \left( \Omega + \uniqueparameter \right) \right\}, \numberthis \label{eq:loss-fun-def} \\
\Phi \left( \Theta \right) & \coloneqq \left\vert \Omega \right\vert_1 + \sum^K_{k=0} \sqrt{\alpha_k} \, \left\vert \uniqueparameter \right\vert_1, \qquad \text{and}\numberthis \label{eq:regularizer-def} \\
\mathcal{C} \left( M_{\mathrm{op}} \right) & \coloneqq \left\{ \Theta = \left( \Omega, \left\{ \uniqueparameter \right\}^K_{k=0} \right) \, : \, \Omega + \uniqueparameter \succ 0 \ \text{and} \ \max_{0 \leq k \leq K} \left\Vert \Omega + \uniqueparameter \right\Vert_2 \leq M_{\mathrm{op}} \right\}. \numberthis \label{eq:opt-domain}
\end{align*}
Here, $\alpha_k=n_k / N$ with  $N=\sum^K_{k=0} n_k$ controls the contribution of each source, $\lambda_{\text{M}}>0$ is a regularization parameter, and $M_{\mathrm{op}} > 0$ is a predefined constant. The constraint $\left\Vert \Omega + \uniqueparameter \right\Vert_2 \leq M_{\mathrm{op}}$ is primarily included to facilitate theoretical analysis.
While we theoretically need to set $M_{\mathrm{op}}$ to be a sufficiently large constant, as detailed in a later section, in practice, we simply set $M_{\mathrm{op}} = \infty$ to effectively remove this constraint.
}

The first term of~\eqref{eq:Glasso-mutli-task} measures parameter fitness with observed data, while the second term $\Phi(\Theta)$ promotes sparsity in both the shared and individual components. 
The sparsity penalization level for $\uniqueparameter$ is proportional to $\sqrt{n_k/N}$, a factor that is crucial to balance contributions from the target and sources. See Section~\ref{sec:thm-trans-mt-glasso} for details.

In the end, we construct the initial estimators of $\Omega^{(k)}$ as $\steponeoutput=\widehat{\Omega}+\uniqueestimate$ for $0 \leq k \leq K$.

\subsection{Refinement via Differential Network Estimation}

To further enhance accuracy, we refine these initial estimators by estimating the differential networks $\diff=\Omega^{(k)}-\Omega^{(0)}$, which capture structural differences between each source and the target. This refinement step adjusts for potential biases in the initial estimates.

Estimating $\diff$ has been extensively studied in the  differential network estimation literature~\citep{zhao2014direct, yuan2017differential}, with a variety of good estimators. For instance, 
\citet{yuan2017differential} proposed estimation by solving
\begin{equation}
\label{eq:diff-network-est-Dtrace}
\diffestimate \in \,\,\arg \min_{\Psi} L_D \left( \Psi ; \widehat{\Sigma}^{(0)} , \widehat{\Sigma}^{(k)} \right) + \lambda^{(k)}_{\Psi} \left\vert \Psi \right\vert_1, 
\end{equation}
where $\lambda^{(k)}_{\Psi}$ is a tuning parameter, and 
\begin{equation*}
L_D \left( \Psi ; \widehat{\Sigma}^{(0)} , \widehat{\Sigma}^{(k)} \right) = \frac{1}{4} \left( \left\langle \widehat{\Sigma}^{(0)} \Psi, \Psi \widehat{\Sigma}^{(k)}  \right\rangle + \left\langle \widehat{\Sigma}^{(k)} \Psi , \Psi \widehat{\Sigma}^{(0)}   \right\rangle  \right) - \left\langle \Psi, \widehat{\Sigma}^{(0)} - \widehat{\Sigma}^{(k)} \right\rangle.
\end{equation*}
For our purpose, any reasonable differential network estimator can correct for the bias. Thus, we treat differential network estimation as a black-box algorithm, outputting $\diffestimate$.

With the initial estimator $\steponeoutput$ and refined differential network estimators $\diffestimate$, we construct the final transfer learning estimator for $\Omega^{(0)}$ as
\begin{equation}
\label{eq:Trans-GLasso}
\widehat{\Omega}^{(0)} = \sum^K_{k=0} \alpha_k \left( \steponeoutput - \diffestimate \right),
\end{equation}
where $\widehat{\Psi}^{(0)}=0$ by definition. The final estimator \myalg~(\textbf{Trans}fer learning \textbf{G}raphical \textbf{lasso}) integrates both shared information and source-specific refinements, yielding a transfer learning approach that leverages structural similarities across datasets for improved precision matrix estimation.
Compared to~\cite{li2023transfer}, our method is more sample efficient as it does not require sample splitting between the steps.

\section{Implementation in Practice}
\label{sec:implementation}

In this section, we provide practical guidelines for implementing \myalg, including optimization techniques, hyperparameter selection, and a method for identifying the informative set when not all sources are useful.

\subsection{Optimization Algorithms}
\label{sec:Optimization}

To implement \myalg, we first solve the \myalgMT objective from Equation~\eqref{eq:Glasso-mutli-task} to obtain initial estimates for $\shared$ and $\unique$. This is a constrained optimization problem that can be efficiently solved using the Alternating Direction Method of Multipliers (ADMM)~\citep{boyd2011distributed}.
In this section, we slightly abuse notation by using superscripts to denote the iteration round and subscripts to represent the population.

Define
\begin{equation*}
X =
\begin{bmatrix}
\Omega_0 \\
\Omega_1 \\
\vdots \\
\Omega_K
\end{bmatrix}
, \
Y =
\begin{bmatrix}
\Omega \\
\Gamma_0 \\
\Gamma_1 \\
\vdots \\
\Gamma_K
\end{bmatrix}
, \ \text{and} \ B = 
\begin{bmatrix}
I_d & I_d & 0 & \cdots & 0 \\
I_d & 0 & I_d & \cdots & 0 \\
\vdots & \vdots & \vdots &  & \vdots \\
I_d & 0 & 0 & \cdots & I_d
\end{bmatrix},
\end{equation*}
where $X \in \mathbb{R}^{(K+1)d \times d},Y \in \mathbb{R}^{(K+2)d \times d},B \in \mathbb{R}^{(K+1)d \times (K+2) d}$.
With this notation, the objective in Equation~\eqref{eq:Glasso-mutli-task} can be reformulated as 
\begin{align*}
\text{\minimize} \quad f(X) + g(Y), \qquad \text{subject to} \, \, X = B Y,
\end{align*}
where
\begin{align*}
f (X) &= \sum^K_{k=0} f_k (\Omega_k), \quad f_k (\Omega_k) = \alpha_k \left\{ - \log \det \left( \Omega_k \right) + \left\langle \widehat{\Sigma}^{(k)}, \Omega_k \right\rangle  \right\} + \mathbb{I} \left( \Omega_k \succ 0 \right), \quad 0 \leq k \leq K, \\
g (Y) &= \lambda_{\text{M}} \left\vert \Omega \right\vert_1 + \lambda_{\text{M}} \sum^K_{k=0} \sqrt{\alpha_k} \left\vert \Gamma_k \right\vert_1,
\end{align*}
where $\mathbb{I} \left( \Omega \succ 0 \right) = 0$ if $\Omega \succ 0$ and $\mathbb{I} \left( \Omega \succ 0 \right) = \infty$ otherwise. 

The augmented Lagrangian for this problem is: 
\begin{equation*}
L_{\rho} (X, Y, Z) = f(X) + g(Y) + \rho \left\langle Z, X - BY \right\rangle + \frac{\rho}{2} \left\Vert X - BY \right\Vert^2_{\mathrm{F}},
\end{equation*}
where $Z^{\top}=[Z^{\top}_0,Z^{\top}_1,\ldots,Z^{\top}_K]\in \mathbb{R}^{d \times (K+1)d}$
is a dual variable and $\rho$ is a penalty parameter. 

After initializing $Y^{(0)}$ and $Z^{(0)}$ such that $\Omega^{(0)}$, $\Gamma^{(0)}_k$, and $Z^{(0)}_k$ are symmetric, ADMM iteratively updates:
\begin{align*}
& X^{(t)} = \arg \min_{X} L_{\rho} \left( X, Y^{(t-1)}, Z^{(t-1)} \right), \numberthis \label{eq:ADMM-update-1} \\ 
& Y^{(t)} \in \arg \min_{Y} L_{\rho} \left( X^{(t)}, Y, Z^{(t-1)} \right), \numberthis \label{eq:ADMM-update-2} \\
& Z^{(t)} = Z^{(t-1)} + \rho (  X^{(t)} - B Y^{(t)}) . \numberthis \label{eq:ADMM-update-3}
\end{align*}
Note that~\eqref{eq:ADMM-update-3} is equivalent to $Z^{(t)}_k = Z^{(t-1)}_k + \rho \left( \Omega^{(t)}_k - \Omega^{(t)} - \Gamma^{(t)}_k \right)$ for $0 \leq k \leq K$.
See Appendix~\ref{sec:opt-addition} for detailed steps and stopping criteria.

To refine the initial estimators, we solve the D-Trace loss objective~\eqref{eq:diff-network-est-Dtrace} for differential network estimation.
Here, we use a different proximal gradient descent algorithm~\citep{Parikh2014Proximal} following~\cite{zhao2022fudge, zhao2019direct}.

For simplicity, let $L_D (\Psi) \coloneqq L_D \left( \Psi \,;\, \widehat{\Sigma}^{(0)} , \widehat{\Sigma}^{(k)} \right)$ for the chosen $k$. In iteration $t$, we update $\Psi^{(t-1)}$ by solving
\begin{equation}
\label{eq:iterorign}
\Psi^{(t)}=\arg \min_{\Psi} \left\{ \frac{1}{2}\left\Vert \Psi -\left(\Psi^{(t-1)}-\eta\nabla L_D \left(\Psi^{(t-1)}\right)\right)\right\Vert_{F}^{2}+\eta\cdot\lambda^{(k)}_{\Psi} \left\vert \Psi \right\vert_1 \right\},
\end{equation}
where $\eta$ is a user-specified step size. Note that $\nabla L_D (\cdot)$ is Lipschitz continuous with constant $\Vert \widehat{\Sigma}^{(0)} \otimes \widehat{\Sigma}^{(k)} \Vert_2 = \Vert \widehat{\Sigma}^{(0)} \Vert_2 \Vert \widehat{\Sigma}^{(k)} \Vert_2$. Thus, for $0< \eta \leq \Vert \widehat{\Sigma}^{(0)} \Vert^{-1}_2 \Vert \widehat{\Sigma}^{(k)} \Vert^{-1}_2$, the proximal gradient method converges~\citep{Beck2009Fast}. The update in~\eqref{eq:iterorign} has a closed-form solution:
\begin{equation}
\label{eq:itersimplified}
\Psi^{(t)}_{jl}=\left[ \left\vert A^{(t-1)}_{jl} \right\vert - \lambda^{(k)}_{\Psi} \eta \right]_{+} \cdot A^{(t-1)}_{jl} \,/\, \left\vert A^{(t-1)}_{jl} \right\vert, \qquad 1 \leq j,l \leq d,
\end{equation}
where $A^{(t-1)}=\Psi^{(t-1)}-\eta\nabla L_D (\Psi^{(t-1)})$ and $x_{+}=\max\{0,x\}$ for any  $x\in{\mathbb{R}}$.
Details on the optimization algorithms, including stopping criteria and descriptions, are in Appendix~\ref{sec:opt-addition}.

\subsection{Hyperparameter Selection}
\label{sec:hyperparameterTuning}

This section covers the selection of hyperparameters, specifically $\lambda_{\text{M}}$ in \myalgMT~\eqref{eq:Glasso-mutli-task} and $\lambda^{(k)}_{\Psi}$ in D-Trace loss~\eqref{eq:diff-network-est-Dtrace}.

We choose $\lambda^{(k)}_{\Psi}$ for $k \in [K]$ to minimize the Bayesian information criterion (BIC) of D-Trace loss:
\begin{equation}
\label{eq:bic-d-trace}
\text{BIC}^{(k)}_{\Psi} = \left( n_0 + n_k \right) \left\Vert \frac{1}{2} \left( \widehat{\Sigma}^{(0)} \diffestimate \widehat{\Sigma}^{(k)} + \widehat{\Sigma}^{(k)} \diffestimate \widehat{\Sigma}^{(0)}  \right) - \widehat{\Sigma}^{(0)} + \widehat{\Sigma}^{(k)} \right\Vert_{\mathrm{F}} + \log \left( n_0 + n_k \right) \cdot \left\vert \diffestimate \right\vert_0,
\end{equation}
following~\cite{yuan2017differential}. After selecting $\lambda^{(k)}_{\Psi}$ and obtaining $\diffestimate$ for all $k \in [K]$, $\widehat{\Omega}^{(0)}$ depends on $\lambda_{\text{M}}$. Recall that $N=\sum^K_{k=0} n_k$.
We choose $\lambda_{\text{M}}$ to minimize the BIC of \myalg, defined as
\begin{equation}
\label{eq:bic-trans-glasso}
\text{BIC}_{\mathrm{Trans}} = N \cdot \left[ \left\langle \widehat{\Sigma}^{(0)} \, , \, \widehat{\Omega}^{(0)} \right\rangle - \log \det \left( \widehat{\Omega}^{(0)} \right)  \right] + \log N \cdot \left\vert \widehat{\Omega}^{(0)} \right\vert_0 .
\end{equation}

\subsection{Identifying the Informative Set}
\label{sec:UnknownInformSet}

In practice, it is not necessarily true that all source distributions are structurally similar to the target.  We propose a data-driven method to estimate the informative set $\mathcal{A} \subseteq [K]$.

We obtain differential network estimations $\diffestimate$ for all $k \in [K]$, with the hyperparameter $\lambda^{(k)}_{\Psi}$ chosen by minimizing the BIC criterion in~\eqref{eq:bic-d-trace}. We then rank sources according to the sparsity level of $\diffestimate$. Let $R_k$ be the rank of the source $k$. For any $1 \leq k_1, k_2 \leq K$, $\vert \widehat{\Psi}^{(k_1)} \vert_0 \leq \vert \widehat{\Psi}^{(k_2)} \vert_0$ implies $R_{k_1} \leq R_{k_2}$. After ranking sources, we input samples into \myalg and determine the number of sources based on the cross-validation (CV) error. For $K_{\mathrm{chosen}}=0,1,\ldots,K$, we select sources $k$ with $R_k \leq K_{\mathrm{chosen}}$. When $K_{\mathrm{chosen}}=0$, we obtain $\widehat{\Omega}^{(0)}$ from graphical lasso~\citep{friedman2008sparse} using the target data alone. We then compute the CV error of $K_{\mathrm{chosen}}$ by the following procedure:
\begin{enumerate}[label=(\roman*)]
\item We randomly split the \emph{target samples} into $M$-fold. 
\item For $m=1,\ldots,M$, we select the $m$-th fold as the validation set, $\mathcal{D}_{\text{val}}$, and the rest as the training set. We input the training set and chosen source samples into \myalg to obtain $\widehat{\Omega}^{(0)}$. We compute the CV error for the $m$-th fold as
\begin{equation}
\label{eq:cv-error}
\text{CV}_m = \frac{1}{2 d} \left\{ \frac{1}{\vert \mathcal{D}_{\text{val}} \vert} \sum_{i \in \mathcal{D}_{\text{val}} } \tr \left( \mathbf{x}^{(0)}_i \mathbf{x}^{(0) \top}_i \widehat{\Omega}^{(0)} \right) - \log \det \left( \widehat{\Omega}^{(0)} \right) \right\} + \frac{1}{2} \log \pi,
\end{equation}
and define $\text{CV} (K_{\mathrm{chosen}})=\frac{1}{M}\sum^M_{m=1}\text{CV}_m$.
\end{enumerate}

We set the estimated informative set as $\widehat{\mathcal{A}}=\{ k \in [K] : R_k \leq K^{\star}_{\text{chosen}} \}$, where $K^{\star}_{\text{chosen}}=\arg \min_{k=0,1,\ldots,K}\text{CV} (k)$. Source samples from $\widehat{\mathcal{A}}$ are used to estimate $\Omega^{(0)}$. If $\widehat{\mathcal{A}}=\emptyset$, we obtain $\widehat{\Omega}^{(0)}$ via graphical lasso on target samples. We note that sample splitting is not necessary between estimating $\mathcal{A}$ and $\Omega^{(0)}$. {\newText This procedure is called \myalg-CV. We evaluate its empirical performance in Section~\ref{sec:simulation} and defer the study of its theoretical properties to future research. In the next section, we establish the theoretical results under the assumption that $\mathcal{A} = [K]$.
}

\section{Theoretical Analysis}
\label{sec:theoreticalAnalysis}

In this section, we establish theoretical guarantees for the \myalg~algorithm. We begin by analyzing the initial estimation step using \myalgMT, followed by an error bound for the complete \myalg~estimator. Finally, we derive a minimax lower bound, demonstrating that \myalg~is minimax optimal in a wide range of parameter regimes.

To simplify the theoretical statements, we assume the following condition throughout this section. 

\begin{assump}
\label{assump:eig-element-upp-bd}
Assume that 
\begin{equation}
\label{eq:def-Msigma-Momega}
M_\Sigma \coloneqq \max_{ 0 \leq k \leq K } \left\vert \Sigma^{(k)} \right\vert_{\infty} = O(1)
\qquad\text{and}\qquad
M_\Omega \coloneqq \max_{ 0 \leq k \leq K} \left\Vert \Omega^{(k)} \right\Vert_2 = O(1).
\end{equation}
\end{assump}
\subsection{Analysis of \myalgMT}
\label{sec:thm-trans-mt-glasso}
We first provide error bounds for the initial multi-task estimation step, \myalgMT. This method estimates both the shared precision matrix component $\shared$ and the deviation matrices $\unique$ based on the structural similarity outlined in Assumption~\ref{assump:model-structure}.

{\newText

The following theorem provides a high probability upper bound on the Frobenius norm error for the \myalgMT~estimator. Recall that $N=\sum^K_{k=0} n_k$ and let $\widebar{n}=N / (K+1)$.

}

\begin{theorem}
\label{thm:multi-task-frob}
Suppose Assumptions~\ref{assump:model-structure} and \ref{assump:eig-element-upp-bd} hold. 
Fix a failure probability $\delta \in (0,1]$.
Suppose that the local sample size is large enough so that
\begin{equation}
\label{eq:min-nk-large}
\min_{ 0 \leq k \leq K } n_k \geq 2  \log (2 \left( K + 2 \right) d^2 / \delta).
\end{equation}
Set $M_{\mathrm{op}} \geq M_{\Omega}$ and the penalty parameter $\lambda_{\text{M}}$ such that 
\begin{equation}
\label{eq:mlut-Frob-lambda-large-enough}
\lambda_{\text{M}} \geq 160  M_\Sigma  \sqrt{ \frac{ \log (2 \left( K + 2 \right) d^2 / \delta) }{ 2 N } }.
\end{equation}
Then with probability at least $1-\delta$, the estimator satisfies 
\begin{equation}
\label{eq:TransMT-error-bd}
\sum^K_{k=0} \alpha_k \Vert \steponeoutput - \Omega^{(k)} \Vert^2_{\mathrm{F}} \leq \frac{18 \left( s +  (K+1) h \right) \lambda^2_{\text{M}}}{\kappa^2},
\end{equation}
where $\kappa=\left( 2 M_{\Omega} + M_{\mathrm{op}}  \right)^{-2}$.
\end{theorem}

\noindent Note that the loss function defined in~\eqref{eq:loss-fun-def} is not strongly convex with respect to the Euclidean norm. 
This prevents us from obtaining error guarantee for individual precision matrix. 
Nevertheless, we make a key observation that the loss function exhibits strong convexity with respect to the weighted norm used in~\eqref{eq:TransMT-error-bd}.
See Appendix~\ref{sec:proof-thm-multi-task-frob} for a detailed proof. 

\medskip

When we choose $\lambda_{M} \asymp \sqrt{(\log d )/ N}$, the rate shown in Theorem~\ref{thm:multi-task-frob} consists of two parts. The first part, which is of the order $(s \log d)/{ N }$, refers to the estimation error of the shared component. 
In words, \myalg uses all the samples to estimate the shared component.
The second part, of the order $( K h \log d ) / { N }$, relates to the estimation error of the individual components, i.e., on average, there are $N/ K$ samples to estimate each individual component.

\subsection{Analysis of \myalg}
\label{sec:thm-trans-glasso}

After the initial estimates are obtained via \myalgMT, the differential network estimation step refines these estimates by isolating the deviations $\diff$. This yields the final \myalg~estimator $\widehat{\Omega}^{(0)}$.

As discussed in Section~\ref{sec:TransGlassoAlgorithm}, any differential network estimator can be used in Step 2 for refinement. The differential network estimates $\diffestimate$ are treated as the result of a black-box algorithm, obeying 
\begin{equation}
\label{eq:diff-network-high-prob-err-Frob}
\left\Vert \diffestimate - \diff \right\Vert_{\mathrm{F}} \lesssim  g^{(k)}_{\mathrm{F}} (  n_0, n_k, d, h, M_{\Gamma}, \delta ) \coloneqq g^{(k)}_{\mathrm{F}} 
\end{equation}
simultaneously for all $k=1,\ldots,K$ with probability at least $1-\delta$.
We now establish a non-asymptotic error bound for this estimator, which combines the initial estimation error with the error from differential network estimation.

\begin{corollary}
\label{corollary:convg-Trans-GLasso-Frob}
Let $\widehat{\Omega}^{(0)}$ be the \myalg~estimator obtained in Equation~\eqref{eq:Trans-GLasso}. Under the same conditions as in Theorem~\ref{thm:multi-task-frob}, and assuming Equation~\eqref{eq:diff-network-high-prob-err-Frob} holds for the differential network estimators,
with probability at least $1-2\delta$, one has 
\begin{align*}
\left\Vert \widehat{\Omega}^{(0)} - \Omega^{(0)} \right\Vert^2_{\mathrm{F}} \lesssim \left( \frac{ s }{ N } + \frac{ h }{ \widebar{n} } \right)  \log (2 \left( K + 2 \right) d^2 / \delta) + \sum^K_{k=0} \alpha_k g^{(k)}_{\mathrm{F}}.
\end{align*}
\end{corollary}
The error rates depend on differential network estimators' performance. Next, we provide specific error rates using the D-Trace loss estimator.

\subsubsection{A differential network estimator: D-Trace loss minimization}

We characterize $g^{(k)}_{\mathrm{F}}$ in the case when the D-Trace loss estimator~\eqref{eq:diff-network-est-Dtrace} is used. We use the tighter D-Trace loss estimator analysis from~\cite{zhao2019direct, zhao2022fudge, tugnait2023estimation} to obtain the following lemma, which is derived directly from  Corollary~\ref{corollary:analysis-Dtrace} in  Appendix~\ref{sec:proof-corollary:analysis-Dtrace}.
\begin{lemma}
\label{lemma:Dtrace}
Suppose that Assumptions~\ref{assump:model-structure}-\ref{assump:eig-element-upp-bd} hold. Assume that $$\min_{0 \leq k \leq K} n_k \gg h^2 \log \left( 2 (K+1) d^2 / \delta \right).$$ Set 
\begin{equation*}
\lambda^{(k)}_{\Psi} =  C \sqrt{ \frac{ \log \left( 2 (K+1)  d^2 / \delta \right)  }{ \min \{ n_k, n_0 \} } } \quad\text{for all } k \in [K]
\end{equation*}
for some large constant $C>0$. Then 
\begin{align*}
\left\Vert \diffestimate - \diff \right\Vert_{\mathrm{F}} \lesssim \sqrt{h} M_{\Gamma} \sqrt{ \frac{ \log \left( 2 (K+1)  d^2 / \delta \right)  }{ \min \{ n_k, n_0 \} } } 
\end{align*}
holds simultaneously for all $k \in [K]$ with probability at least $1-\delta$.
\end{lemma}

By Lemma~\ref{lemma:Dtrace}, we have
\begin{align*}
g^{(k)}_{\mathrm{F}}  = \sqrt{h} M_{\Gamma} \sqrt{ \frac{ \log \left( 2 (K+1)  d^2 / \delta \right)  }{ \min \{ n_k, n_0 \} } } 
\end{align*}
for the D-Trace loss estimator. Plugging the above results into Corollary~\ref{corollary:convg-Trans-GLasso-Frob}, we have the following corollary.

\begin{corollary}
\label{corollary:convg-Trans-GLasso-Frob-Dtrace}
Let $\widehat{\Omega}^{(0)}$ be obtained by \myalg~\eqref{eq:Trans-GLasso} with the D-Trace loss estimator used in Step~2. Assuming the assumptions in Theorem~\ref{thm:multi-task-frob} and Lemma~\ref{lemma:Dtrace} hold. 
For a given $\delta \in (0,1]$, letting 
\begin{equation*}
\lambda_{\text{M}} \asymp \sqrt{ \frac{ \log (2 \left( K + 2 \right) d^2 / \delta) }{ N } }, \quad \lambda^{(k)}_{\Psi} \asymp M_{\Gamma} \sqrt{ \frac{ \log \left( 2 (K+1)  d^2 / \delta \right)  }{ \min \{ n_k, n_0 \} } } \quad \text{ for all } k \in [K],
\end{equation*}
we have 
\begin{align*}
\left\Vert \widehat{\Omega}^{(0)} - \Omega^{(0)} \right\Vert_{\mathrm{F}} & \lesssim \left( \sqrt{\frac{ s }{N } } + (1 + M_{\Gamma} ) \sqrt{\frac{h}{ \widebar{n} }} + M_{\Gamma} \sqrt{\frac{h}{ n_0 }} \right) \sqrt{ \log (2 \left( K + 2 \right) d^2 / \delta) }
\end{align*}
with probability at least $1-2\delta$.
\end{corollary}
The estimation error consists of three parts: shared component estimation, individual component estimation, and differential network estimation.
If $\widebar{n} \geq n_0$ and $M_{\Gamma}$ is bounded by a universal constant, the error scales as $\sqrt{\frac{s \log d}{N}} + \sqrt{\frac{h \log d}{n_0}}$. 
{\newText
Note that the rate obtained when using only target samples is on the order of $\sqrt{\frac{s \log d}{n_0}} + \sqrt{\frac{h \log d}{n_0}}$. 
Thus, when $s \gg h$ and $s / N \ll h / n_0$, transfer learning can significantly reduce the error rate. 
In Appendix~\ref{sec:upp-bd-expe-err}, we also derive an upper bound for the expected error, which follows the same order.
}

\subsection{Minimax Lower Bounds and Optimality}
\label{sec:thm-lower-bd}

To evaluate the theoretical performance of \myalg, we derive the minimax lower bound for estimating the target precision matrix \( \Omega^{(0)} \) over the parameter space defined by Assumptions~\ref{assump:model-structure}-\ref{assump:eig-element-upp-bd}. 
More precisely, we define the relevant parameter space: 
\begin{multline}
\label{eq:param-space}
\mathcal{G} \left( s, h \right) \coloneqq \left\{ \left\{ \Omega^{(k)} \right\}^K_{k=0} : \, \Omega^{(k)} \succ 0 , \, \Omega^{(k)} = \shared + \unique , \, \text{supp}\left( \shared \right) \cap \text{supp}\left( \unique \right) = \emptyset \ \forall k , \right.\\
\left. \left\vert \shared  \right\vert_0 \leq s, \, \max_{0 \leq k \leq K} \left\vert \unique \right\vert_0 \leq h, \, \max_{0 \leq k \leq K} \left\vert \unique \right\vert_1 \leq M_{\Gamma} , \, \max_{0 \leq k \leq K} \left\Vert \Omega^{(k)} \right\Vert_2 \leq M_{\Omega} , \, \max_{0 \leq k \leq K} \left\vert \Sigma^{(k)} \right\vert_{\infty} \leq M_{\Sigma} \right\}, 
\end{multline}
where $M_{\Gamma} > 0$, $M_{\Omega},M_{\Sigma} > 1$ are universal constants.

Intuitively, the performance limit of estimating the target precision matrix is dictated by the information-theoretic lower bounds of estimating two parts, namely the shared component and the individual component. This motivates us to develop lower bounds for estimating these two parts.

\paragraph{Lower bound for estimating shared component.}
The following lemma provides the minimax lower bound for estimating the shared component when all the distributions are the same.
\begin{lemma}
\label{lemma:low-bd-shared}
Assume that we have $n$  i.i.d.~samples $X_1,\ldots,X_n$ from $N(0,\Omega^{-1})$, where
\begin{equation*}
\Omega \in \mathcal{G}_1 = \left\{ \Omega \in \mathbb{S}^{d \times d} \, : \, \Omega \succ 0, \, \left\vert \Omega \right\vert_0 \leq s , \, 0 < c_1 \leq  \gamma_{\min} ( \Omega ) \leq \gamma_{\max} ( \Omega ) \leq c_2 < \infty  \right\},
\end{equation*}
and $c_1,c_2$ are universal constants. 
In addition, assume that $s \geq d \geq c^{\prime} n^{\beta}$ for some universal constants $\beta>1$ and $c^{\prime}>0$, and
\begin{equation*}
\lceil s / d \rceil = o \left( \frac{ n }{ \left( \log d \right)^{\frac{3}{2}} } \right) .
\end{equation*}
We then have
\begin{equation*}
\inf_{\widehat{\Omega}} \sup_{\Omega \in \mathcal{G}_1} \mathbb{E} \left[ \left\Vert \widehat{\Omega} - \Omega \right\Vert^2_{\mathrm{F}} \right] \gtrsim \frac{s \log d}{n}.
\end{equation*}
\end{lemma}

\noindent The proof is based on Theorem 6.1 in~\cite{cai2016estimating}. See 
Appendix~\ref{sec:proof-thm-low-bd-shared} for more details. 
By Lemma~\ref{lemma:low-bd-shared}, it is easy to see that the squared Frobenius error of estimating the shared component is lower bounded by ${s \log d}/{N}$ where $N$ is the total number of samples.

\paragraph{Lower bound for estimating the individual components.}
We also provide a lower bound for estimating the individual component. When $\Omega^{(k)}=I_d$ for all $1 \leq k \leq K$ and $\Omega^{(0)}=I_d+\Delta$ with $\textrm{diag}(\Delta)=0$, the source samples are not helpful to estimate $\Omega^{(0)}$ at all. Thus, the minimax lower bound for estimating $\Delta$ provides a valid minimax lower bound for the transfer learning problem.
\begin{lemma}
\label{lemma:lower-bd-diff-net}
Assume that we have $n$  i.i.d.~samples $X_1,\ldots,X_n$ from $N(0,\Omega^{-1})$, where
\begin{equation}
\label{eq:low-bd-param-space-2}
\begin{aligned}
\Omega \in \mathcal{G}_2 = & \left\{ \Omega \in \mathbb{S}^{d \times d} : \, \Omega \succ 0 , \, \Omega = I_d + \Delta, \, \Delta_{jj} = 0 \text{ for all } 1 \leq j \leq d, \right. \\
& \quad\quad\quad\quad \left. \left\vert \Delta \right\vert_0 \leq h, \,  \left\vert \Delta \right\vert_1 \leq C_{\Gamma}, \, 0 < c_1 \leq \gamma_{\min} (\Omega) \leq \gamma_{\max} (\Omega) \leq c_2 < \infty \right\},
\end{aligned}
\end{equation}
where $C_{\Gamma}>0$, $c_1 < 1$ and $c_2 > 1$ are constants.
In addition, assume that
\begin{equation}
\label{eq:lower-bd-diff-net-cond}
d \geq 4h, \, h \log d \geq 8 \log 3, \, \frac{h \log d}{n} \leq \min \left\{ 2, 8 (1-c_1)^2, 8(1-c_2)^2 \right\}, \, h \sqrt{\frac{\log d}{n}} \leq 4 C_{\Gamma}.
\end{equation}
We then have
\begin{equation*}
\inf_{\widehat{\Omega}} \sup_{\Omega \in \mathcal{G}_2} \mathbb{E} \left[ \left\Vert  \widehat{\Omega} - \Omega \right\Vert^2_{\mathrm{F}} \right] \gtrsim \frac{h \log d}{n}.
\end{equation*}
\end{lemma}

\noindent See Appendix~\ref{sec:proof-thm-lower-bd-diff-net} for the proof, which 
relies on a novel construction of the packing set of the parameter space and on the celebrated Fano's method~\citep[Section 15.3]{wainwright2019high}. 

\medskip 
It is worth noting that Lemma~\ref{lemma:lower-bd-diff-net} also provides a minimax lower bound for estimating the differential network $\Omega_X-\Omega_Y$ for two precision matrices $\Omega_X$ and $\Omega_Y$ when the $L_1$-norm of the differential network is bounded. To our knowledge, this is also the first lower bound for differential network estimation~\citep{zhao2014direct,yuan2017differential}. As a result, we can derive the first minimax optimal rate for differential network estimation.
See Appendix~\ref{sec:minimax-opt-diff-network} for a more detailed discussion.

\paragraph{Combining pieces together. }
Combining Lemma~\ref{lemma:low-bd-shared} and Lemma~\ref{lemma:lower-bd-diff-net}, we have the following lower bound for estimating the target precision matrix in this transfer learning setup.
\begin{theorem}
\label{thm:minimax-optimal-rate}
Suppose that we have $n_k$ i.i.d.~samples from a sub-Gaussian distribution $\mathcal{P}_k$ with zero mean and precision matrix $\Omega^{(k)}$ for all $0 \leq k \leq K$. Besides, assume that $s \geq d \geq c^{\prime} N^{\beta}$ for some universal constants $\beta>1$, $c^{\prime}>0$ and
\begin{equation*}
\lceil s / d \rceil = o \left( \frac{ N }{ \left( \log d \right)^{\frac{3}{2}} } \right).
\end{equation*}
In addition, assume that
\begin{equation*}
d \geq 4h, \, h \log d \geq 8 \log 3, \, h \sqrt{\frac{\log d}{n_0}} \leq 4 M_{\Gamma}, \frac{h \log d}{n_0} \leq \min \left\{ 2, \, 8 \left( 1 - M_{\Omega} \right)^2, \, 8 \left( 1 - \frac{1}{M_{\Sigma}} \right)^2  \right\},
\end{equation*}
where $M_{\Gamma}>0, M_{\Omega}, M_{\Sigma}>1$ are universal constants defined in~\eqref{eq:param-space}.
We then have
\begin{equation*}
\inf_{ \widehat{\Omega} } \sup_{ \{ \Omega^{(k)} \}^K_{k=0} \in \mathcal{G}(s,h ) } \mathbb{E} \left[ \left\Vert \widehat{\Omega} - \Omega^{(0)} \right\Vert^2_{\mathrm{F}} \right] \gtrsim  \frac{ s \log d }{ N } + \frac{h \log d}{n_0} \, .
\end{equation*}

\end{theorem}

Theorem~\ref{thm:minimax-optimal-rate} demonstrates that \myalg achieves minimax optimality for the parameter space specified in~\eqref{eq:param-space} when $\widebar{n} \geq n_0$. The obtained minimax optimal rate is reasonable given $N$ samples for estimating the shared component with $s$ non-zero entries and $n_0$ samples for estimating the individual component with $h$ non-zero entries. Furthermore, from a practical viewpoint, the rate suggests that the target sample size only needs to be sufficiently large in relation to the sparsity level $h$ of the individual component. In contrast, if we only have target samples, the target sample size needs to be large enough to match the sparsity level $s + h$ of the entire precision matrix, which can be significantly larger.

\section{Simulations}
\label{sec:simulation}

In this section, we demonstrate the empirical performance of \myalg~through a series of simulations. We evaluate its accuracy in comparison with several baseline methods under different settings, including varying sample sizes and sparsity levels.

We set the dimensionality $d=100$ and the number of source distributions $K=5$ for all experiments. In each experiment, we vary parameters such as the target sample size $n_0$, the source sample size $n_k=n_{\mathrm{source}}$, and the sparsity level $h$ to assess the robustness of \myalg across diverse conditions. The data are simulated under three different models, each reflecting a specific structure for the shared and individual components of the precision matrices.

\subsection{Data Generation Models}
\label{sec:data-generate-models}
We generate data from three distinct models to assess the flexibility of Trans-Glasso.
Each model starts with a shared component, followed by individual components. The final precision matrices are made positive definite by adding a diagonal matrix. Specifically, each model is set as follows.
\begin{itemize}

\item \textbf{Model I:} The shared component is a banded matrix with bandwidth 1, where each entry $\tilde{\Omega}_{ij}=5 \times 0.6^{\vert i-j \vert} \mathds{1}( \vert i -j \vert \leq 1)$ for $1 \leq i,j \leq d$. For a given $h$ and each $k=0,1,\ldots,K$, we uniformly choose $\lceil h/2 \rceil$ entries $(i,j)$ such that $1 \leq i \leq \lfloor \frac{d}{2} \rfloor$ and $\lfloor \frac{d}{2} \rfloor + 1\leq j \leq d$, denoted as $\mathcal{S}_{\Gamma^{(k)},\text{up}}$. We let $\tilde{\Gamma}^{(k)}_{ij}=u_{ij}\mathds{1}\{ (i,j)\in \mathcal{S}_{\Gamma^{(k)},\text{up}} \}$, $1 \leq i,j \leq d$, where $u_{ij}$'s are from $\text{Unif}[-3,3]$. Then $\unique=\tilde{\Gamma}^{(k)} + (\tilde{\Gamma}^{(k)})^{\top}$. Finally, we let $\Omega^{(k)}=\tilde{\Omega}+\unique+\sigma I_d$, where $\sigma$ ensures $\gamma_{\min}(\Omega^{(k)}) \geq 0.1$ for $0 \leq k \leq K$.

\item \textbf{Model II:} Model II is similar to Model I but with a wider bandwidth of $5$, introducing a more connected structure in the shared component.

\item \textbf{Model III:} We generate the shared component from an Erdos–Renyi graph. Specifically, let $\tilde{\Omega}_{ii}=5$, $1 \leq i \leq d$, and $U_{ij} \sim \text{Bernoulli}(0.02)$, $1 \leq i < j \leq d$. If $U_{ij}=1$, let $\tilde{\Omega}_{ij}=\tilde{\Omega}_{ji} \sim \text{Unif}[-3,3]$; otherwise, set $\tilde{\Omega}_{ij}=\tilde{\Omega}_{ji}=0$. Let $\mathcal{S}_{\tilde{\Omega}}=\{ (i,j) \in [d] \times [d] : \tilde{\Omega}_{ij} \neq 0 \}$ be the support of $\tilde{\Omega}$. For given $h$ and $0 \leq k \leq K$, uniformly choose $h$ entries ($h+1$ if $h$ is odd) from $[d] \times [d] \setminus \tilde{\Omega}$, denoted as $\mathcal{S}_{\Gamma^{(k)}}$, such that $(i,j) \in \mathcal{S}_{\Gamma^{(k)}}$ if and only if $(j,i) \in \mathcal{S}_{\Gamma^{(k)}}$. Let $\unique_{ij}=u_{ij}\mathds{1}\{ (i,j)\in \mathcal{S}_{\Gamma^{(k)}} \}$, $1 \leq i,j \leq d$, where $u_{ij} \sim \text{Unif}[-3,3]$. Finally, let $\Omega^{(k)}=\tilde{\Omega}+\unique+\sigma I_d$, where $\sigma$ ensures $\gamma_{\min}(\Omega^{(k)}) \geq 0.1$ for $0 \leq k \leq K$.

\end{itemize}

\subsection{Experimental Design}

\begin{figure}[t]
\centering
\includegraphics[width=.8\textwidth]{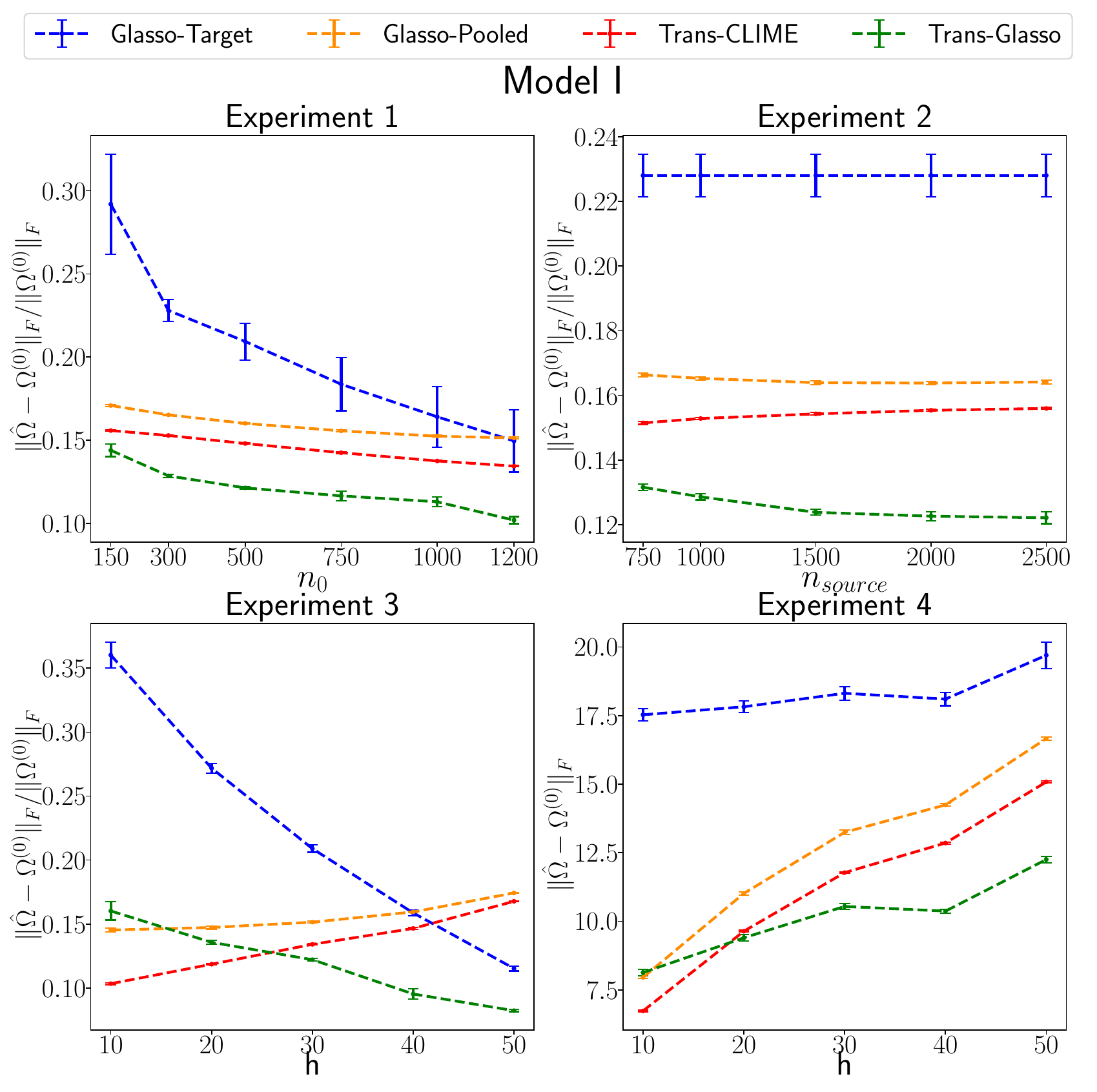}
\caption{Simulation results for Model I. 
}
\label{fig:simu-exp-model1}
\end{figure}

For each model, we conduct four main experiments to evaluate Trans-Glasso under various conditions, repeating each 30 times to ensure reliable averages and standard errors.

\begin{enumerate}
    \item \textbf{Experiment 1}: Vary the target sample size \( n_0 \) while keeping the source sample size \( n_{\text{source}} \) and sparsity \( h \) fixed.
    \item \textbf{Experiment 2}: Vary the source sample size \( n_{\text{source}} \) while keeping \( n_0 \) and \( h \) fixed.
    \item \textbf{Experiment 3}: Increase both \( n_0 \), \( n_{\text{source}} \), and \( h \) proportionally to examine scalability.
    \item \textbf{Experiment 4}: Fix \( n_0 \) and \( n_{\text{source}} \) while increasing \( h \), assessing performance as sparsity in deviations increases.
\end{enumerate}

\subsection{Comparison Methods}
We compare \myalg with the following baseline methods:

\begin{itemize}
    \item \textbf{Glasso-Target}: Applies graphical lasso~\citep{friedman2008sparse} to target data only.
    \item \textbf{Glasso-Pooled}: Pools target and source data, then applies graphical lasso.
    \item \textbf{Trans-CLIME}: A transfer learning method by~\cite{li2023transfer} that assumes a sparse divergence matrix across sources.
\end{itemize}

\subsection{Results}

\begin{figure}[t]
\centering
\includegraphics[width=.8\textwidth]{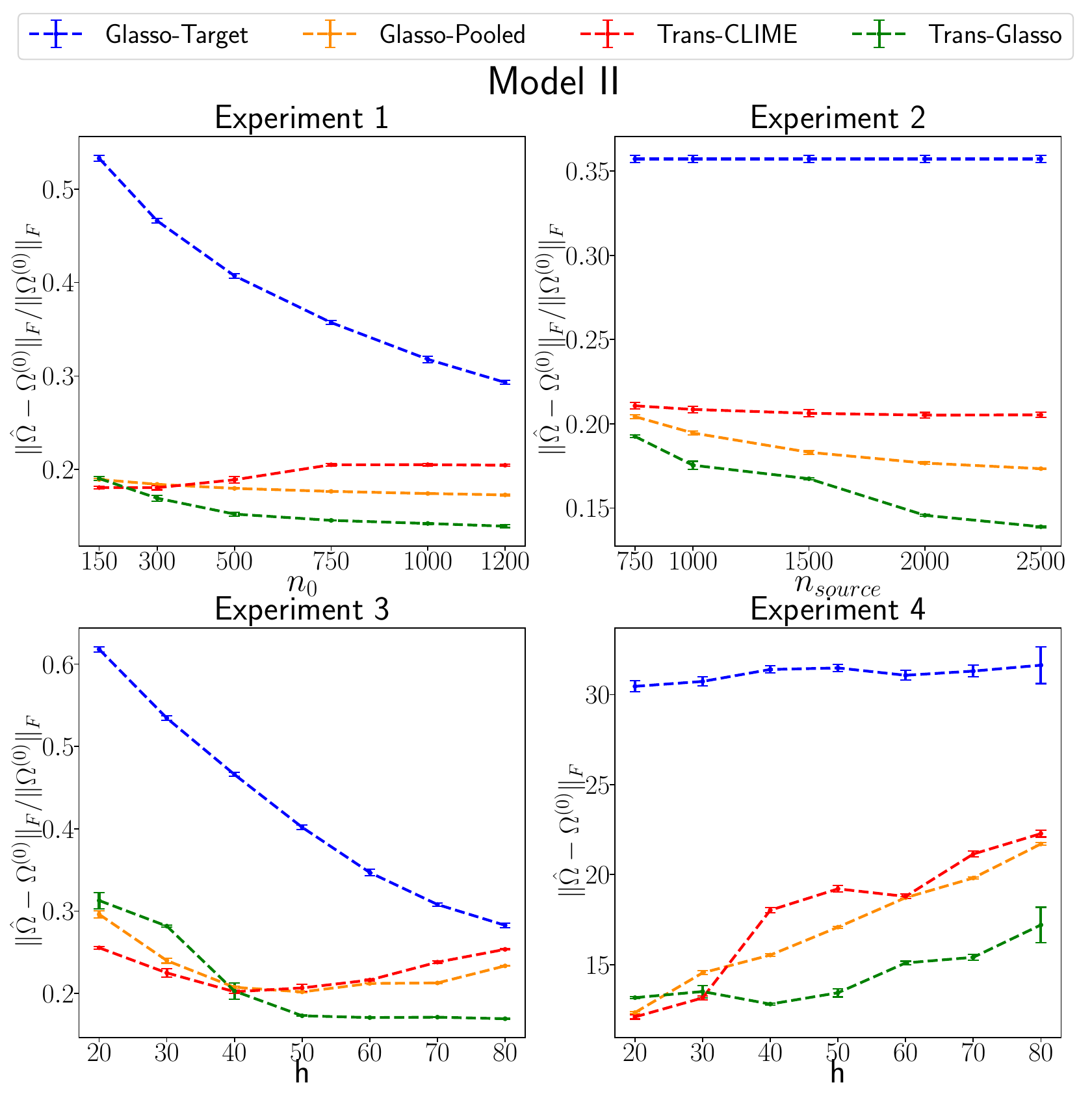}
\caption{Simulation results for Model II. 
}
\label{fig:simu-exp-model2}
\end{figure}

\begin{figure}[t]
\centering
\includegraphics[width=.8\textwidth]{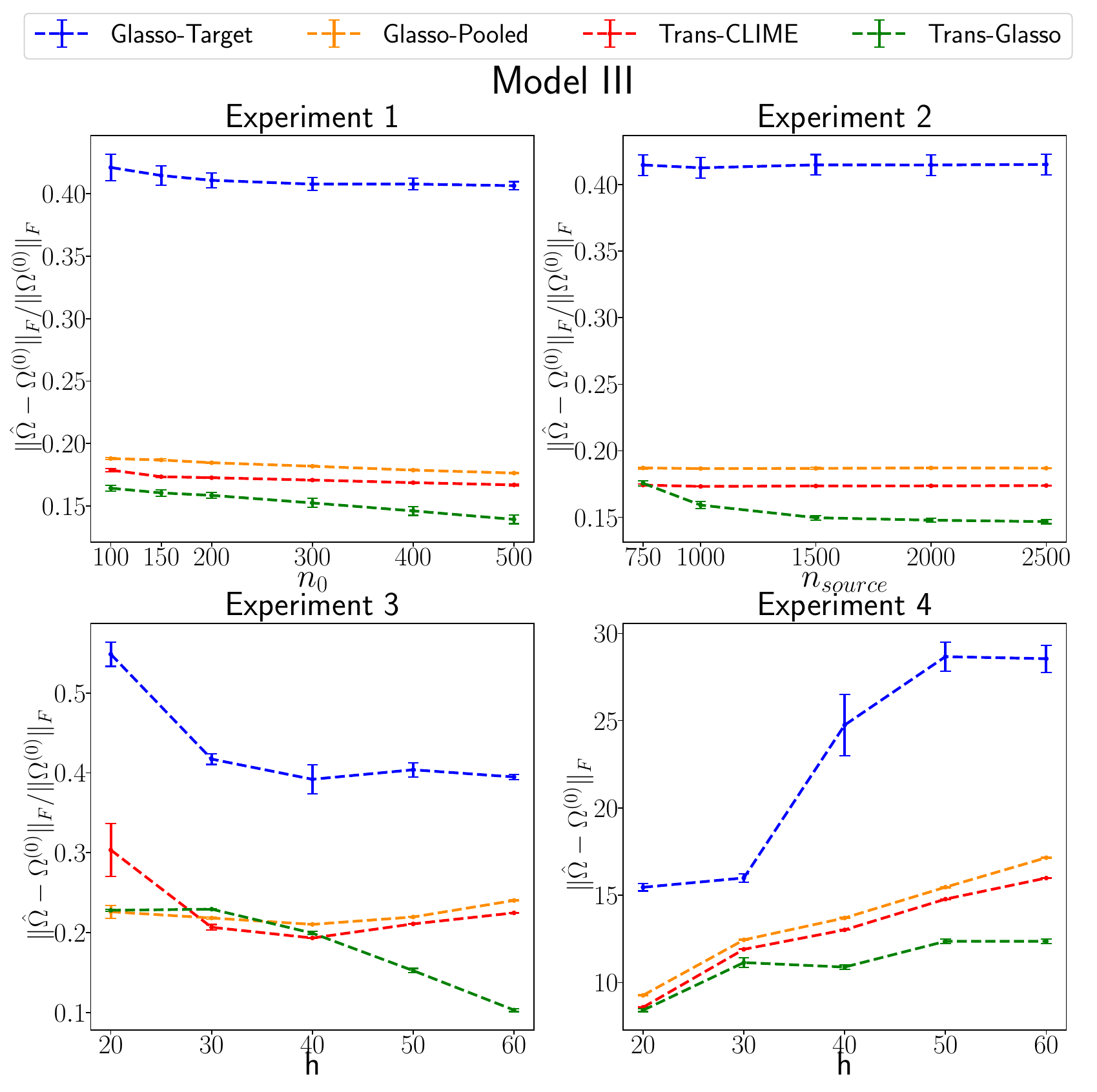}
\caption{Simulation results for Model III. 
}
\label{fig:simu-exp-model3}
\end{figure}

The results for Models I -- III are shown in Figure~\ref{fig:simu-exp-model1}--\ref{fig:simu-exp-model3}. \myalg generally outperforms baseline methods. In Experiment 3, \myalg shows consistency across all models, whereas Glasso-Pooled and Trans-CLIME do not. Trans-CLIME performs better with small $h$, but its performance deteriorates as $h$ increases. This is because a small increase in $h$ can significantly increase the sparsity of the divergence matrix when the covariance matrix is not sparse, as discussed in Appendix~\ref{sec:comparison-li2023}. Therefore, when the precision matrix is sparse but the covariance matrix is not, \myalg is more reliable and robust.

\subsection{Experiments with Unknown Informative Set}

\begin{figure}[t]
\centering
\includegraphics[width=.8\textwidth]{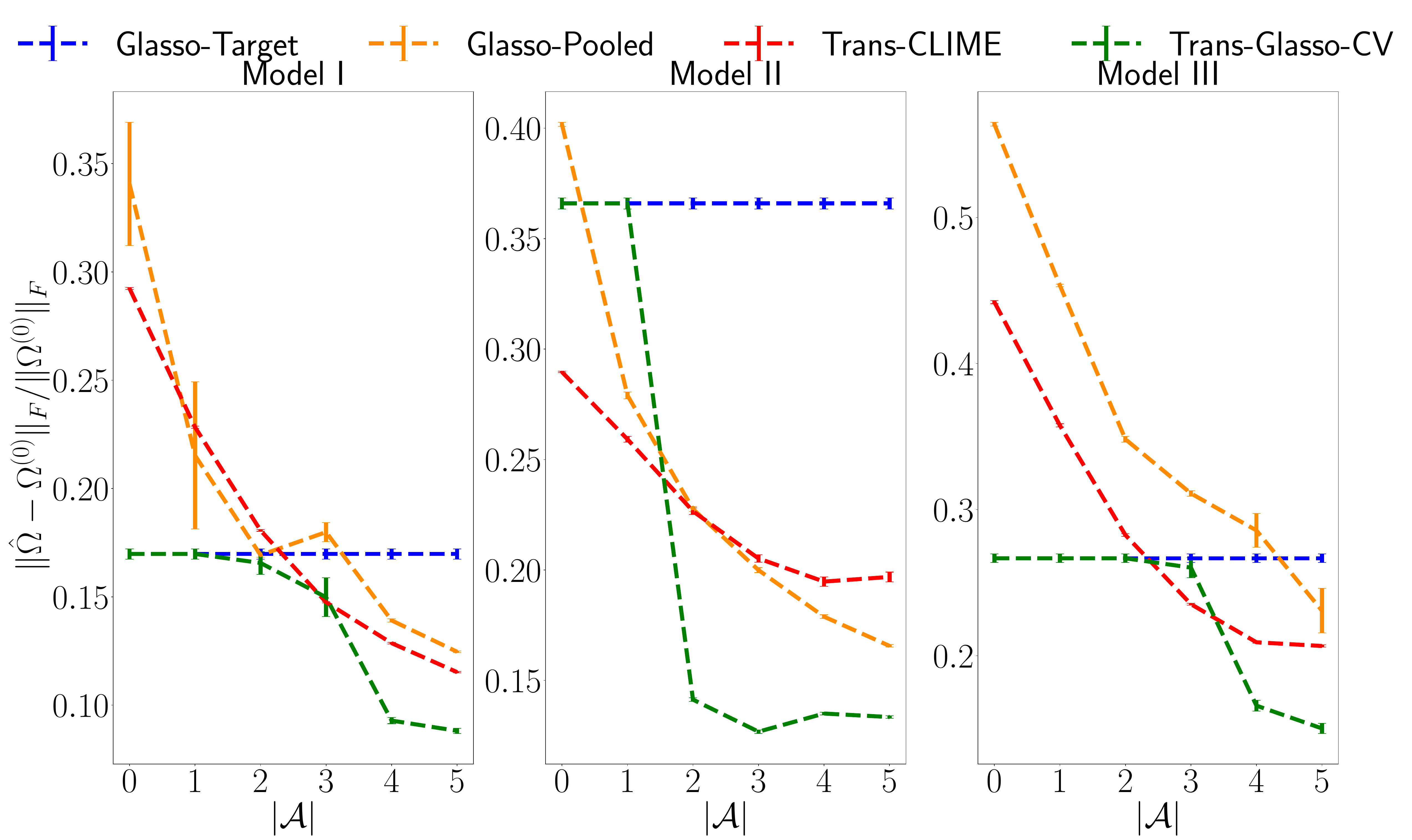}
\caption{Simulation results when the informative set $\mathcal{A}$ is unknown. 
}
\label{fig:simu-exp-unknownA}
\end{figure}

We perform simulation experiments with unknown $\mathcal{A}$, using the same three models. We divide $[K]$ into $[K]=\mathcal{A} \cup \mathcal{A}^c$. For $k \in \mathcal{A}$, we set the sparsity level $h$ to be small, and for $k \in \mathcal{A}^c$, $h$ to be large. Specifically, for Model I, $h=20$ for $k \in \mathcal{A}$ and $h=600$ for $k \in \mathcal{A}^c$; for Model II, $h=30$ for $k \in \mathcal{A}$ and $h=600$ for $k \in \mathcal{A}^c$; for Model III, $h=10$ for $k \in \mathcal{A}$ and $h=300$ for $k \in \mathcal{A}^c$. We implement the \myalg-CV algorithm (Section~\ref{sec:UnknownInformSet}) and compare it with other methods. We vary $\vert \mathcal{A} \vert$ from $0$ to $K$ to observe performance changes. Each experiment is repeated $30$ times with different random seeds.

Figure~\ref{fig:simu-exp-unknownA} shows that \myalg-CV generally outperforms baseline methods. Notably, it never performs worse than Glasso-Target, indicating no ``negative transfer'' of knowledge. In contrast, both Glasso-Pooled and Trans-CLIME can underperform compared to Glasso-Target. Additionally, as $\vert \mathcal{A} \vert$ increases, \myalg-CV achieves the best performance.

\section{Real-World Data Analysis}
\label{sec:real-world-data}

We apply the \myalg algorithm to two real-world datasets. In Section~\ref{sec:real-world-data-gene}, we use it on gene networks with different brain tissues. In Section~\ref{sec:real-world-protein}, we use it on protein networks for various cancer subtypes.

\subsection{Gene Networks Data for Brain Tissues}
\label{sec:real-world-data-gene}

\begin{figure}[t]
\centering
\includegraphics[width=0.8\textwidth]{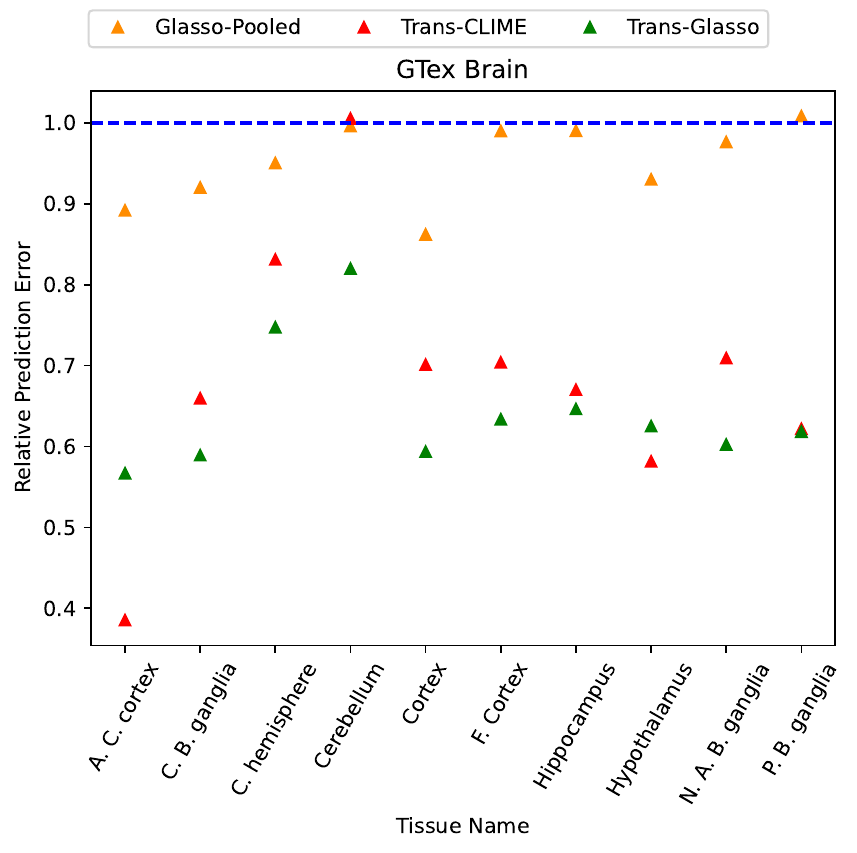}
\caption{Cross-validation prediction error of different methods on GTEx brain tissue datasets, relative to Glasso-Target.}
\label{fig:brain_gtex}
\end{figure}

We apply \myalg to detect gene networks across tissues using the Genotype-Tissue Expression (GTEx) data \footnote{\url{https://gtexportal.org/home/}}. Following~\cite{li2023transfer}, we focus on genes involved in central nervous system neuron differentiation (GO:0021953) across 13 brain tissues, treating one as the target and the remaining 12 as sources. To avoid small sample sizes, we exclude three tissues from the target set, using only 10. A complete list of tissues is in Table~3 of~\cite{li2023transfer}. We remove genes with missing values and compare \myalg to baselines via cross-validation prediction error~\eqref{eq:cv-error}.

Figure~\ref{fig:brain_gtex} presents the results, benchmarking \myalg against Glasso-Target, Glasso-Pooled, and Trans-CLIME, as in Section~\ref{sec:simulation}. Prediction errors are reported relative to Glasso-Target for cross-tissue comparison. As shown in Figure~\ref{fig:brain_gtex}, \myalg consistently outperforms baselines, with a relative prediction error well below 1, demonstrating superior performance and robustness to negative transfer. In contrast, Glasso-Pooled and Trans-CLIME sometimes perform comparably or worse than Glasso-Target.

\begin{figure}[t]
\centering
\includegraphics[width=0.8\textwidth]{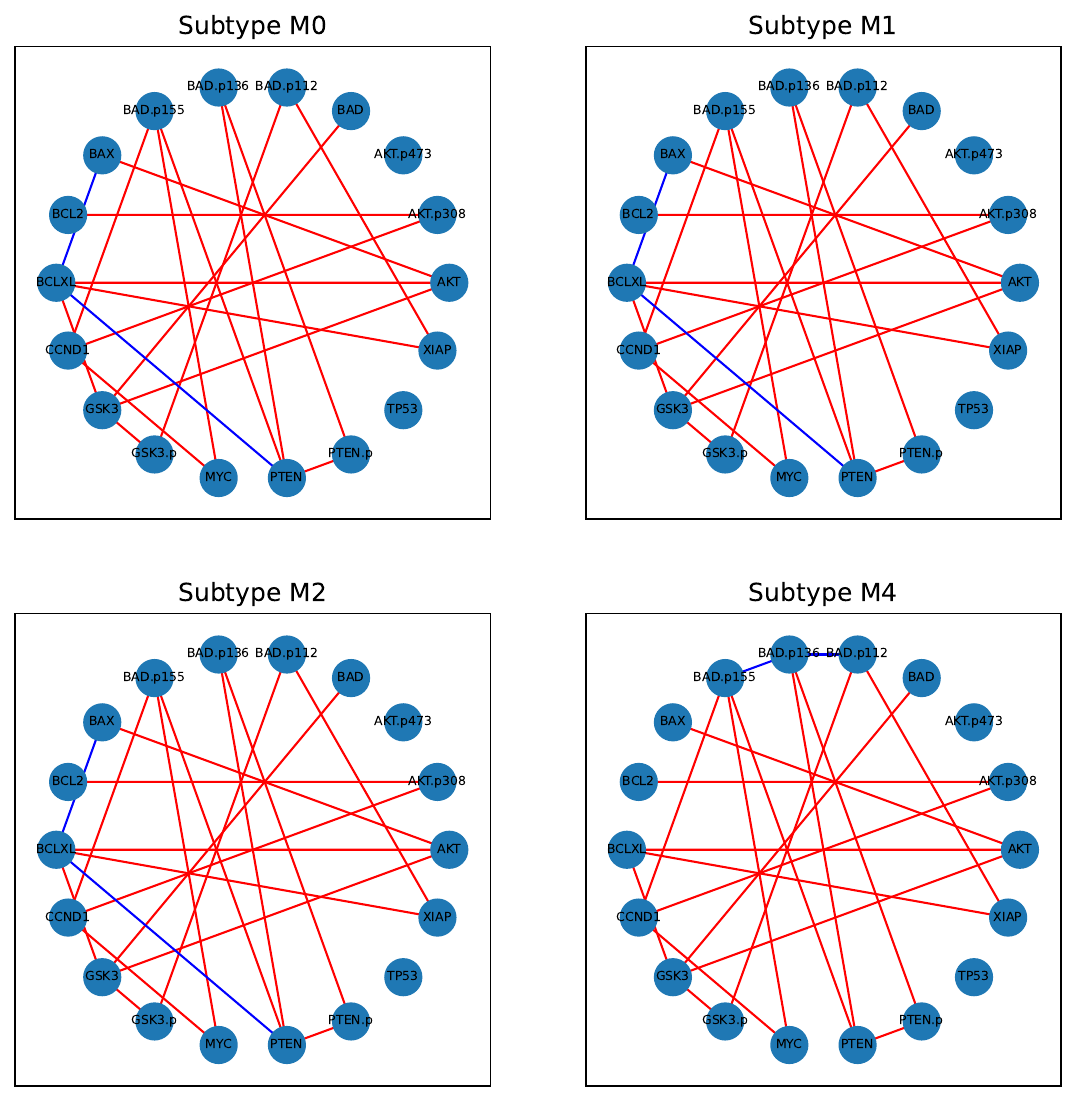}
\caption{Protein networks for four AML subtypes. Red edges are shared by all, blue edges are not.}
\label{fig:AML}
\end{figure}

\subsection{Protein Networks Data for AML}
\label{sec:real-world-protein}

We apply our method to a protein network dataset for Acute Myeloid Leukemia (AML) subtypes. Understanding protein relationships is crucial in cancer studies, and graphical models aid in constructing these networks. Following~\cite{peterson2015bayesian}, we analyze protein levels from 178 newly diagnosed AML patients,\footnote{Dataset provided as a supplement to~\cite{kornblau2009functional}.} classified by the FAB system. While protein interactions may vary across subtypes, shared AML-related processes suggest common connections, making transfer learning beneficial. We focus on 18 proteins involved in apoptosis and cell cycle regulation across four subtypes: M0 (17 subjects), M1 (34), M2 (68), and M4 (59)~\citep{peterson2015bayesian, kanehisa2012kegg}.

For each subtype, we apply \myalg and extract the final estimated graph by selecting 20 edges with the largest absolute values in the estimated precision matrix, ensuring comparability with~\cite{peterson2015bayesian}. The results in Figure~\ref{fig:AML} reveal substantial overlap with~\cite{peterson2015bayesian}, though our graphs exhibit greater structural similarity across subtypes. M0, M1, and M2 share identical structures, while M4 differs by two edges and shows stronger BAD family protein connections. This aligns with~\cite{tzifi2012role}, which reported higher BAD protein expression in AML subtypes M4, M5, and M6~\citep[Table 1]{tzifi2012role}, suggesting more active interactions.

\section{Conclusion}
\label{sec:conclusion}

We introduce \myalg, a novel transfer learning approach for precision matrix estimation that mitigates small target sample limitations by leveraging related source data. \myalg follows a two-step process: initial estimation via multi-task learning, followed by refinement through differential network estimation. This method achieves minimax optimality across various parameter regimes. Extensive simulations show that \myalg consistently outperforms baselines, demonstrating robustness and adaptability, especially in high-dimensional settings with limited target samples.

Future directions include extending \myalg to other graphical models, such as Gaussian copula~\citep{liu2009nonparanormal}, transelliptical~\citep{liu2012transelliptical}, functional~\citep{qiao2019functional, tsai2023latent, zhao2024high}, and Ising models~\citep{kuang2017screening}, as well as developing inferential methods within the transfer learning framework.

\section*{Funding Acknowledgments}

C.M.\ was partially supported by the National Science Foundation under grant DMS-2311127 and the CAREER Award DMS-2443867. 

\clearpage

\newpage

\putbib[boxinz-papers]
\end{bibunit}


\newpage

\appendix

\pagebreak
\begin{bibunit}[plainnat]





{\newText

\section{Comparison of Similarity Assumptions}
\label{sec:comparison-li2023}

In this section, we provide a more detailed comparison between our Assumption~\ref{assump:model-structure} and the similarity assumption used in~\cite{li2023transfer}. Let 
\[
\Upsilon^{(k)} = (\Omega^{(0)} - \Omega^{(k)}) \Sigma^{(k)} = -\diff \Sigma^{(k)},
\]
or equivalently, $\diff = -\Upsilon^{(k)} \Omega^{(k)}$. \citet{li2023transfer} assume that $\Upsilon^{(k)}$ is column-wise sparse in the $L_p$-norm. While our Assumption~\ref{assump:model-structure} and the assumption in~\cite{li2023transfer} do not imply each other and are generally not directly comparable, our assumption may be more suitable in certain applications.

First, whereas $\Upsilon^{(k)}$ is motivated by controlling the KL divergence between Gaussian distributions, our Assumption~\ref{assump:model-structure} is purely structural and applicable to general distributions. Second, our assumption offers a natural interpretation in Gaussian graphical models, unlike the divergence-based similarity assumption in~\cite{li2023transfer}. Finally, a technical advantage of our framework lies in its flexibility when establishing sparsity relationships. While sparsity in $\diff$ does not necessarily imply sparsity in $\Upsilon^{(k)}$, and vice versa, additional structural assumptions on $\Sigma^{(k)}$ or $\Omega^{(k)}$ can relate the two. Specifically, we have
\begin{equation*}
\left\vert \diff \right\vert_0 \leq \sum_{i,j} \mathds{1}\left\{ \sum^d_{l=1} \mathds{1}\left\{ \Upsilon^{(k)}_{il} \neq 0 \right\} \mathds{1}\left\{ \Omega^{(k)}_{lj} \neq 0 \right\} \geq 1  \right\}.
\end{equation*}
If both $\Upsilon^{(k)}$ and $\Omega^{(k)}$ are sparse, then $\diff$ is also sparse. Conversely, if both $\diff$ and $\Sigma^{(k)}$ are sparse, then $\Upsilon^{(k)}$ must be sparse. A key distinction is that sparsity in $\Upsilon^{(k)}$ implies sparsity in $\diff$ only when $\Omega^{(k)}$ is sparse, whereas sparsity in $\diff$ implies sparsity in $\Upsilon^{(k)}$ only when $\Sigma^{(k)}$ is sparse. In graphical models, sparsity assumptions on the precision matrix $\Omega^{(k)}$ are far more common than on the covariance matrix $\Sigma^{(k)}$, making sparsity in $\diff$ generally a weaker and more practical assumption than sparsity in $\Upsilon^{(k)}$.

Although a comprehensive theoretical comparison between our Assumption~\ref{assump:model-structure} and the similarity assumption in~\cite{li2023transfer} is challenging, we provide some empirical insights in the following discussion. We begin by introducing the divergence distance used in~\cite{li2023transfer} to quantify the difference between $\Omega^{(0)}$ and $\{ \Omega^{(k)} \}^K_{k=1}$, defined as
\begin{equation*}
h_{\mathcal{D}} \coloneqq \max_{1 \leq k \leq K} \left\{ \max_{1 \leq j \leq d} \left\Vert \Upsilon^{(k)}_{j,\cdot} \right\Vert_p + \max_{1 \leq j \leq d} \left\Vert \Upsilon^{(k)}_{\cdot,j} \right\Vert_p \right\},
\end{equation*}
where $\Upsilon^{(k)}_{j,\cdot}$ denotes the $j$-th row of $\Upsilon^{(k)}$ and $\Upsilon^{(k)}_{\cdot,j}$ denotes its $j$-th column. Here, $\Vert \cdot \Vert_p$ is the $L_p$ vector norm, for which \citet{li2023transfer} establish theoretical guarantees when $p \in [0,1]$.

\begin{figure}[p]
    \centering
    \begin{subfigure}{0.75\textwidth}
        \centering
        \includegraphics[width=\textwidth]{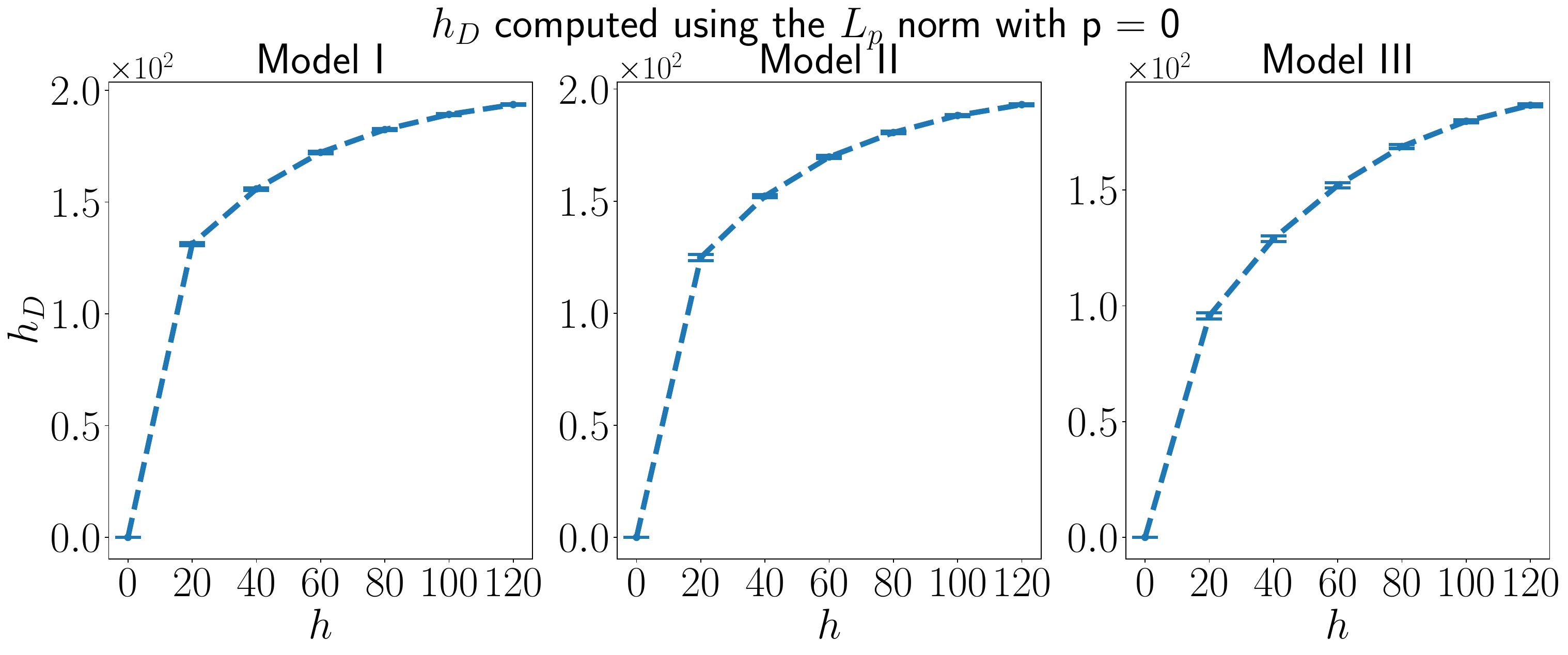}
        \label{fig:assump_compare_subfig1}
    \end{subfigure}

    \vspace{-0.5cm}
    
    \begin{subfigure}{0.75\textwidth}
        \centering
        \includegraphics[width=\textwidth]{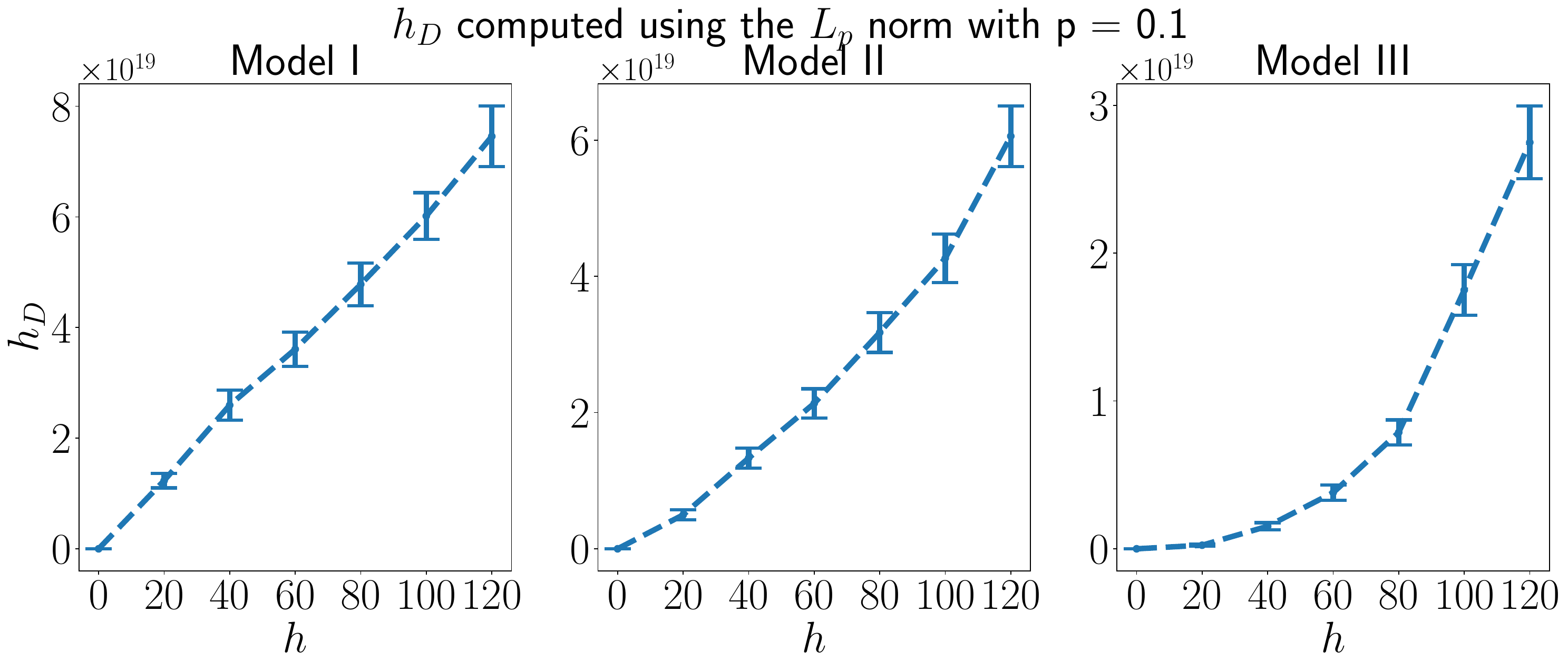}
        \label{fig:assump_compare_subfig2}
    \end{subfigure}

    \vspace{-0.5cm}
    
    \begin{subfigure}{0.75\textwidth}
        \centering
        \includegraphics[width=\textwidth]{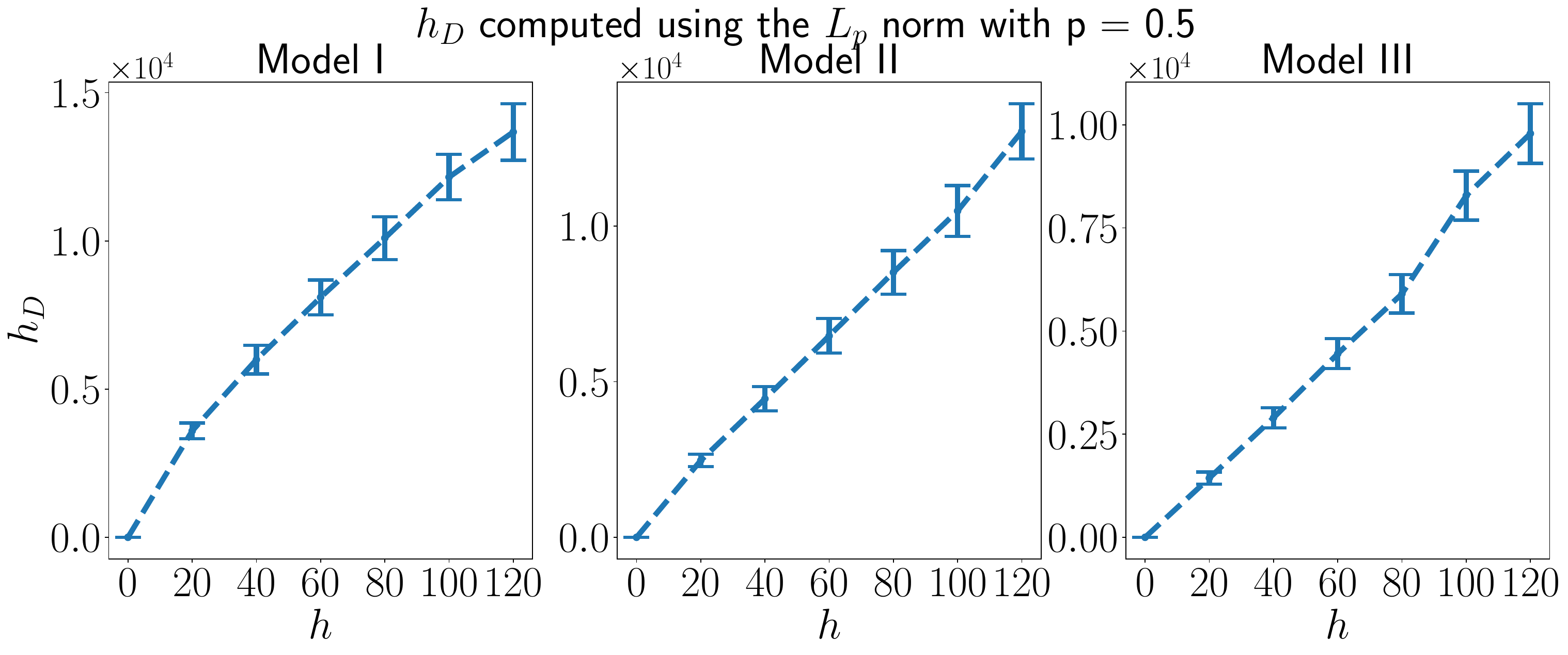}
        \label{fig:assump_compare_subfig3}
    \end{subfigure}

    \vspace{-0.5cm}
    
    \begin{subfigure}{0.75\textwidth}
        \centering
        \includegraphics[width=\textwidth]{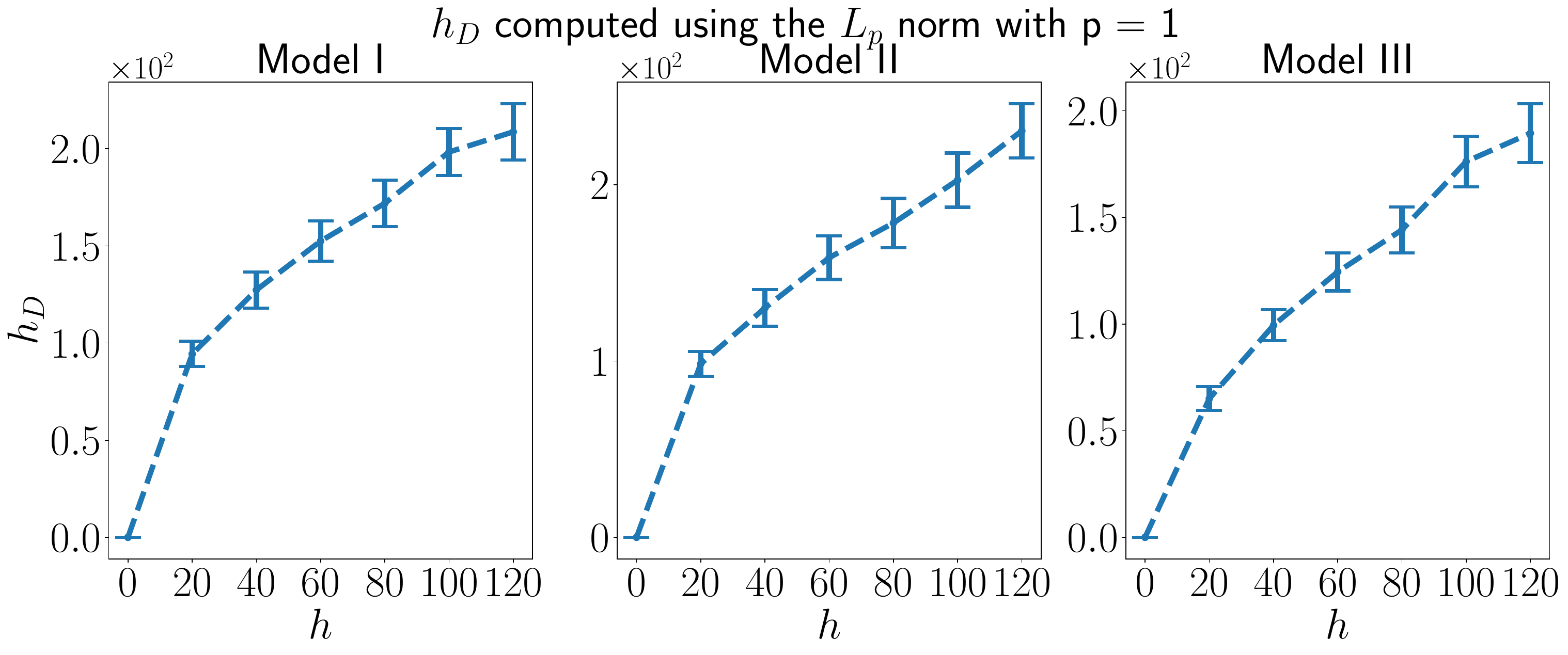}
        \label{fig:assump_compare_subfig4}
    \end{subfigure}

    \vspace{-0.5cm}

    \caption{Empirical analysis of how $h_{\mathcal{D}}$ increases with $h$ under the three models described in Section~\ref{sec:data-generate-models}, for different choices of $p \in [0,1]$.}
    \label{fig:assump_compare}
\end{figure}

 We empirically examine how $h_{\mathcal{D}}$ grows as $h$ increases under the three models discussed in Section~\ref{sec:data-generate-models}, considering different choices of $p \in [0,1]$. Specifically, we set $d = 100$ and $K = 5$. For each value of $h$, we randomly generate 100 sets of target and source precision matrices $\{\Omega^{(k)}_t\}_{k=0}^{K}$ for $t = 1, \dots, 100$. For each $t$, we compute $\{ \Upsilon^{(k)}_t \}^K_{k=1}$ and the corresponding divergence distance $h_{\mathcal{D},t}$, then estimate the empirical mean and standard error of $\{h_{\mathcal{D},1}, \dots, h_{\mathcal{D},100}\}$, denoted as $\widebar{h}_{\mathcal{D}}$ and $\text{SE}_{h_{\mathcal{D}}}$, respectively. We visualize the error bars as $\widebar{h}_{\mathcal{D}} \pm \frac{2}{\sqrt{100}} \, \text{SE}_{h_{\mathcal{D}}}$. Additionally, we report the results for various values of $p$ to examine the effect of different $L_p$ norms.

The results are presented in Figure~\ref{fig:assump_compare}. We observe that the pattern of how $h_{\mathcal{D}}$ increases with $h$ depends on the choice of $p$, which defines the $L_p$ norm. Notably, when $p = 0$, $h_{\mathcal{D}}$ directly quantifies the sparsity level of the divergence matrices. In this case, even for small values of $h$, the sparsity of the divergence matrices can be relatively high, aligning with our earlier discussion.

These observations suggest that there are scenarios where $h$ is small while $h_{\mathcal{D}}$ remains large. However, this does not imply that \myalg consistently outperforms Trans-CLIME. In fact, both algorithms achieve minimax optimality within certain parameter spaces. Our primary point is that \myalg has the potential to perform better than Trans-CLIME under specific conditions. Intuitively, when $h$ is small but $h_{\mathcal{D}}$ is large, \myalg is more likely to perform well; conversely, when $h_{\mathcal{D}}$ is small but $h$ is large, Trans-CLIME may be favored. In practice, the choice between these similarity assumptions should be guided by domain-specific knowledge. A more comprehensive theoretical and empirical comparison between \myalg and Trans-CLIME is an important direction for future research.

}


\begin{figure}[t]
\centering
\includegraphics[width=0.9\textwidth]{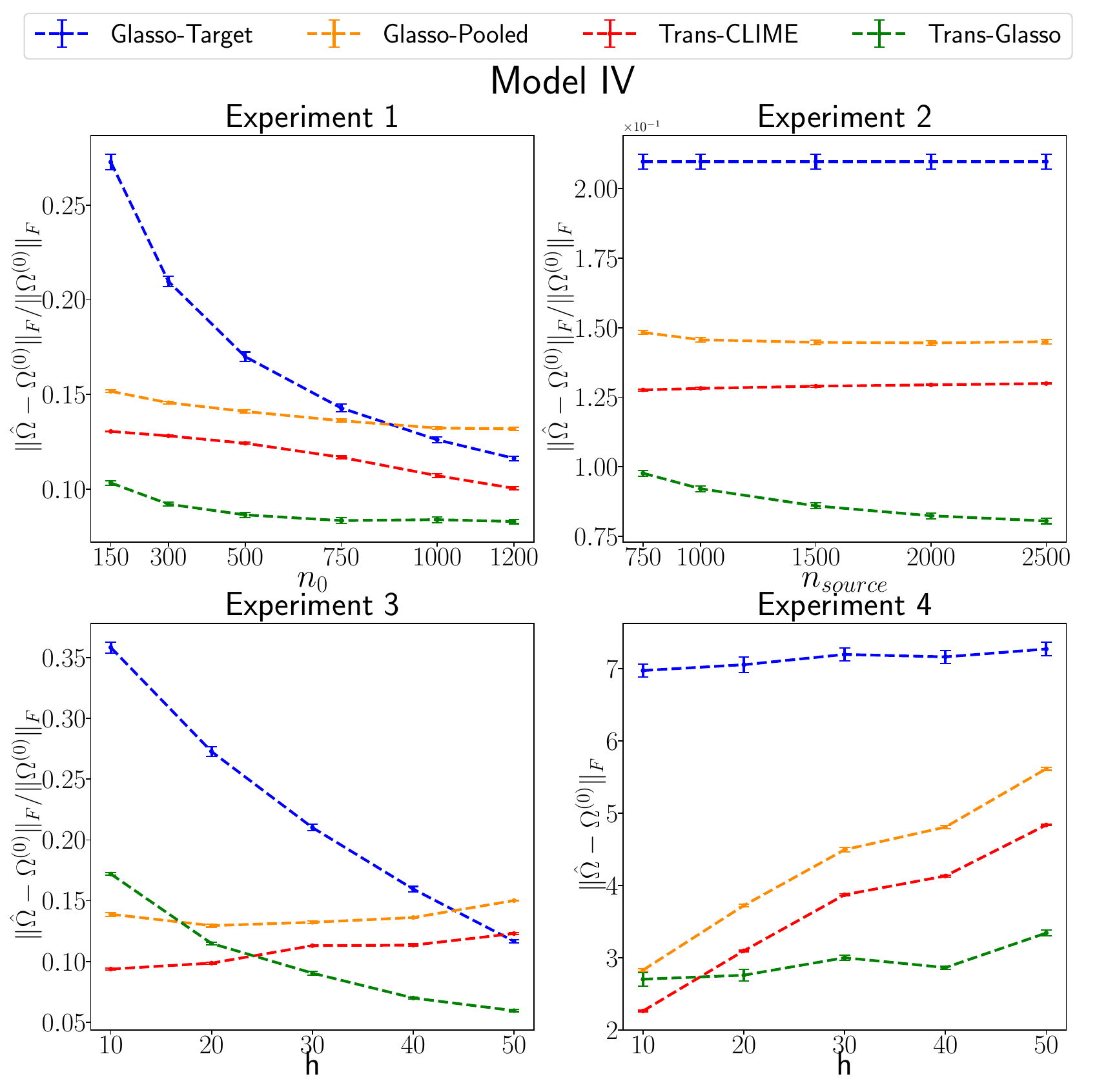}
\caption{Simulation results for Model IV.}
\label{fig:simu-exp-model4}
\end{figure}

\section{More Simulation Details}

In this section, we provide more simulation details.

For the simulation for Model I, the default setting is $n_0=300$, $n_{\text{source}}=1000$ and $h=40$. In the first experiment, we increase $n_0$ while fixing $n_{\text{source}}$ and $h$. In the second experiment, we increase $n_{\text{source}}$ while fixing $n_0$ and $h$. In the third experiment, we increase both $n_0$ and $n_{\text{source}}$ while increasing $h$. More specifically, we let $n_{\text{source}}=3 n_0$, and $n_0=70$ when $h=10$, $n_0=150$ when $h=20$, $n_0=300$ when $h=30$, $n_0=600$ when $h=40$ and $n_0=1200$ when $h=50$. In the fourth experiment, we fix both $n_0$ and $n_{\text{source}}$ while increasing $h$. 

For the simulation for Model II, the default setting is $n_0=750$, $n_{\text{source}}=2000$ and $h=40$. In the first experiment, we increase $n_0$ while fixing $n_{\text{source}}$ and $h$. In the second experiment, we increase $n_{\text{source}}$ while fixing $n_0$ and $h$. In the third experiment, we increase both $n_0$ and $n_{\text{source}}$ while increasing $h$. More specifically, we let $n_{\text{source}}=3 n_0$, and $n_0=100$ when $h=20$, $n_0=200$ when $h=30$, $n_0=300$ when $h=40$, $n_0=500$ when $h=50$, $n_0=800$ when $h=60$, $n_0=1000$ when $h=70$ and $n_0=1200$ when $h=80$. In the fourth experiment, we fix both $n_0$ and $n_{\text{source}}$ while increasing $h$.

For the simulation for Model III, the default setting is $n_0=150$, $n_{\text{source}}=1000$ and $h=40$. In the first experiment, we increase $n_0$ while fixing $n_{\text{source}}$ and $h$. In the second experiment, we increase $n_{\text{source}}$ while fixing $n_0$ and $h$. In the third experiment, we increase both $n_0$ and $n_{\text{source}}$ while increasing $h$. 
More specifically, we let $n_{\text{source}}=4 n_0$, and $n_0=15$ when $h=20$, $n_0=30$ when $h=30$, $n_0=80$ when $h=40$, $n_0=300$ when $h=50$, and $n_0=1000$ when $h=60$. In the fourth experiment, we fix both $n_0$ and $n_{\text{source}}$ while increasing $h$.

For the simulation results when the informative set $\mathcal{A}$ is unknown. We set $n_0=500$ and $n_{\text{source}}=1500$ for Model I; $n_0=750$ and $n_{\text{source}}=2500$ for Model II; and $n_0=100$ and $n_{\text{source}}=500$ for Model III.

In Figure~\ref{fig:simu-exp-model1}--\ref{fig:simu-exp-unknownA}, each dot represents the empirical mean across $30$ repetitions, and the vertical bar represents $\text{Mean} \pm \frac{2}{\sqrt{30}} \times \text{Standard Error}$.

\section{Simulation with Smaller Scales}

To investigate the impact of the precision matrix scale on simulation results, we conduct additional experiments in this section. The data are generated using the same setup as Model I in Section~\ref{sec:data-generate-models}, with a single modification: we set $\tilde{\Omega}_{ij} = 2 \times 0.6^{\vert i - j \vert} \mathds{1}( \vert i - j \vert \leq 1)$ for $1 \leq i,j \leq d$, and draw $u_{ij}$ values from $\text{Unif}[-1,1]$. This adjustment reduces the overall scale of the precision matrix, making it more comparable to the setting examined in the simulation experiments of~\cite{li2023transfer}. We refer to this data generation setup as Model IV. The results are presented in Figure~\ref{fig:simu-exp-model4}. As shown, the outcomes closely resemble those of Model I in Figure~\ref{fig:simu-exp-model1}, suggesting that the performance of the various algorithms remains generally robust to changes in the scale of the precision matrix.

{\newText

\section{Ablation Study of the Refinement Step}

We empirically assess the contribution of the refinement step by evaluating the performance of the initial estimator $\widecheck{\Omega}^{(0)}$, which is the direct output of \myalgMT. By comparing the performance of $\widecheck{\Omega}^{(0)}$ with that of the final estimator $\widehat{\Omega}^{(0)}$, we analyze the impact of the refinement step via differential network estimation. We adopt the same simulation settings as in Section~\ref{sec:simulation} and evaluate empirical performance across all three data generation models.

The results are presented in Figures~\ref{fig:simu-exp-model1-transmt}--\ref{fig:simu-exp-model3-transmt}. As shown, \myalg consistently outperforms \myalgMT, underscoring the crucial role of the refinement step in enhancing estimation accuracy.

\begin{figure}[t]
\centering
\includegraphics[width=.9\textwidth]{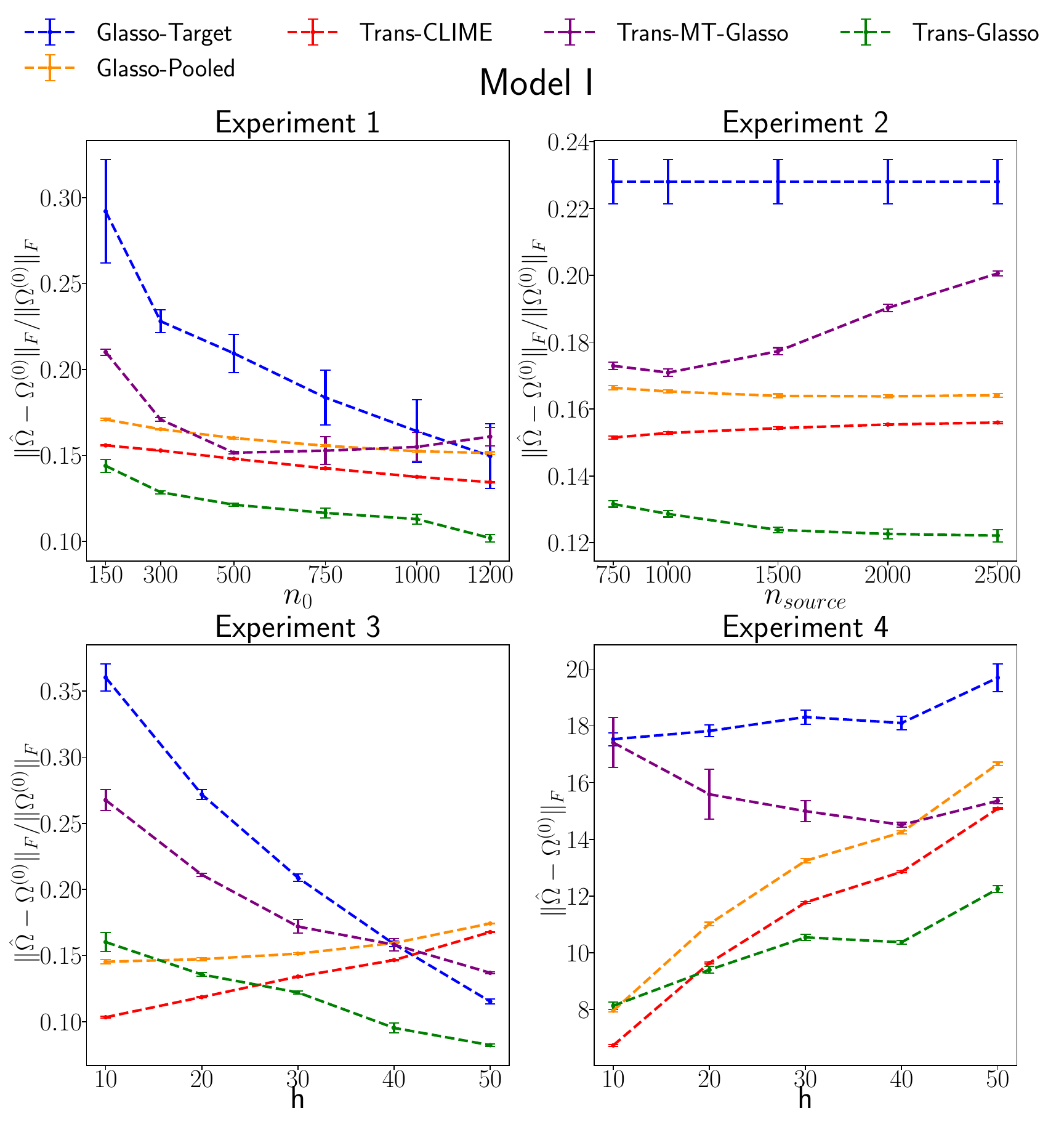}
\caption{Simulation results of \myalgMT for Model I.}
\label{fig:simu-exp-model1-transmt}
\end{figure}

\begin{figure}[t]
\centering
\includegraphics[width=.9\textwidth]{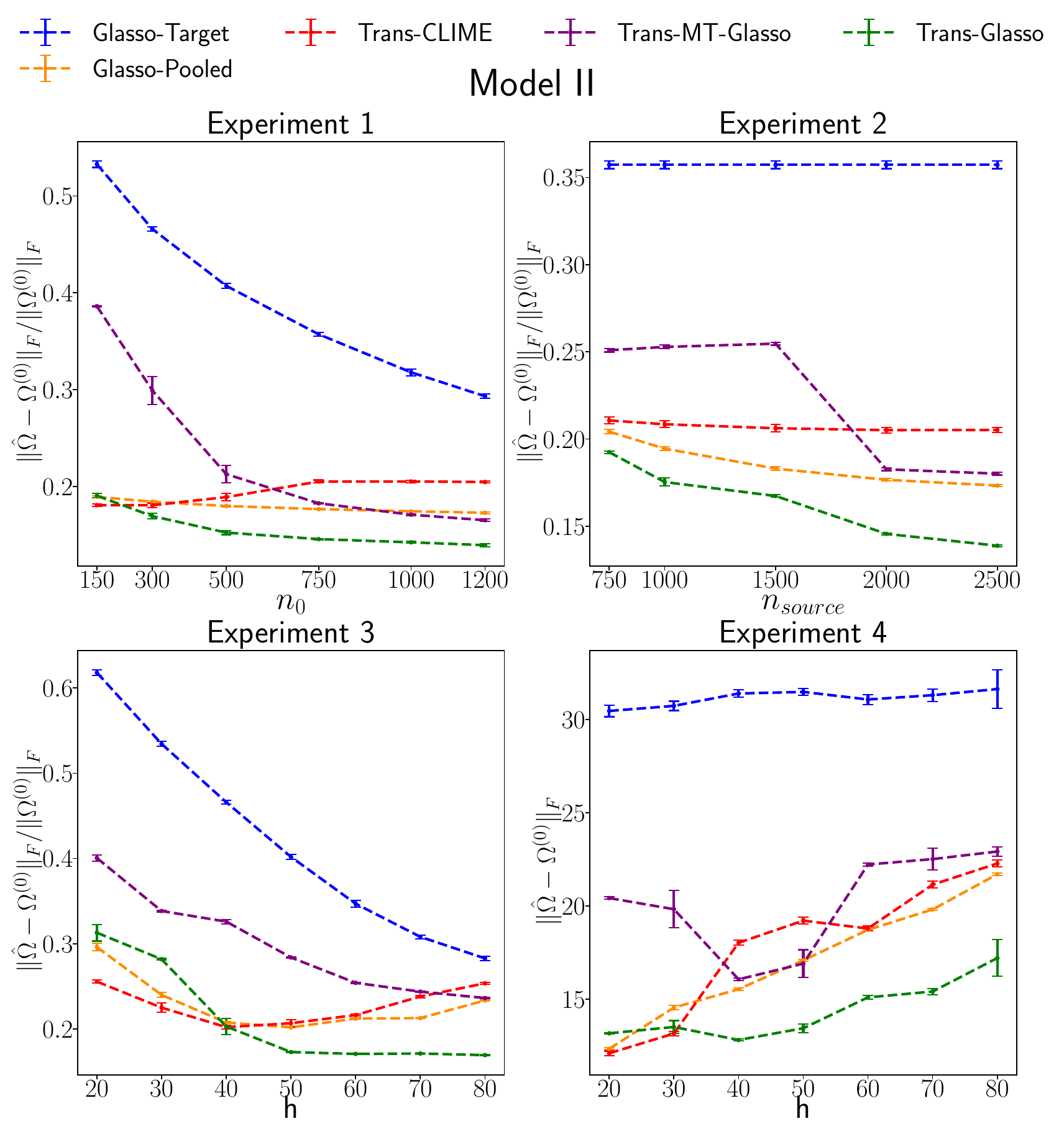}
\caption{Simulation results of \myalgMT for Model II.}
\label{fig:simu-exp-model2-transmt}
\end{figure}

\begin{figure}[t]
\centering
\includegraphics[width=.9\textwidth]{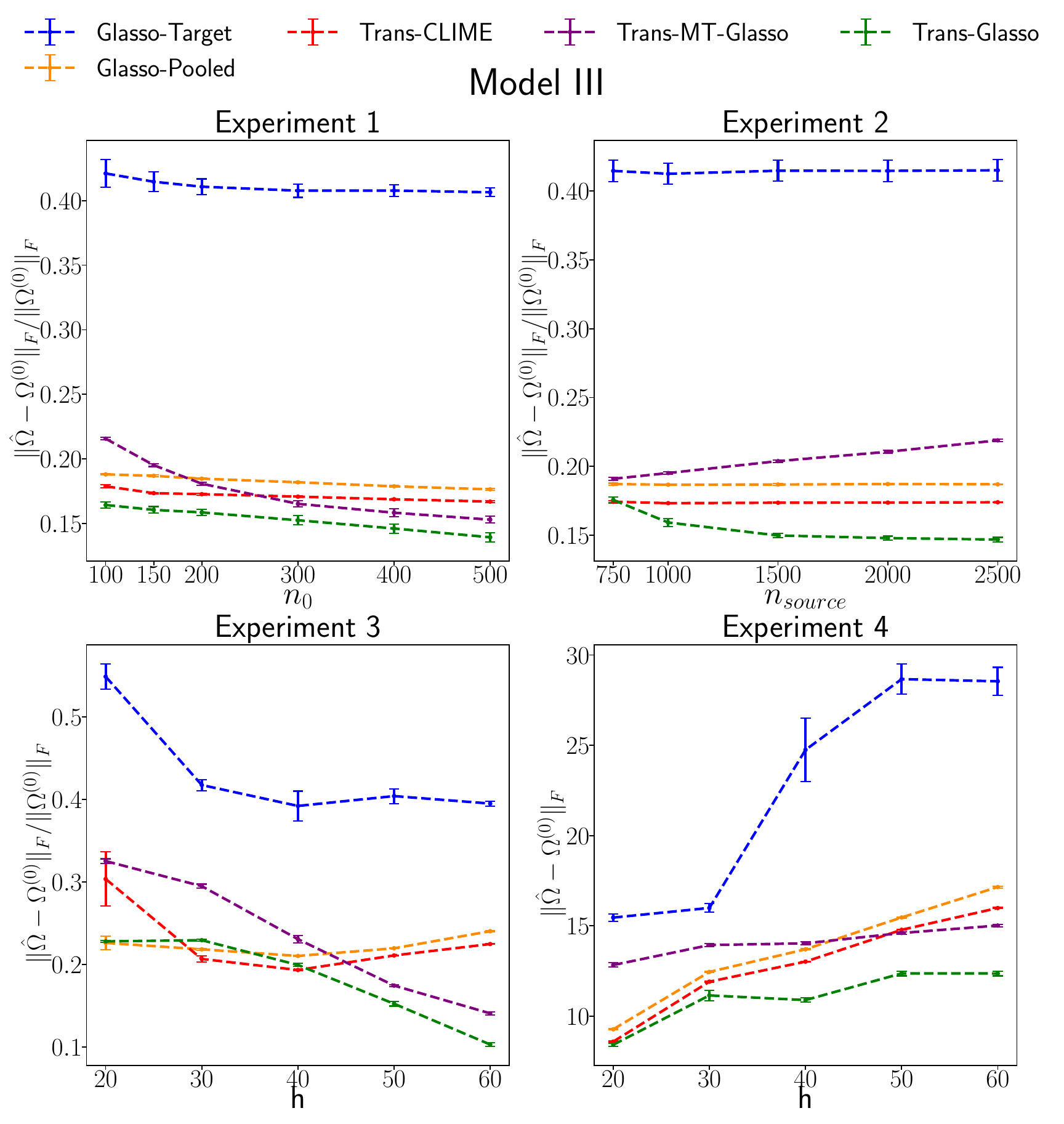}
\caption{Simulation results of \myalgMT for Model III.}
\label{fig:simu-exp-model3-transmt}
\end{figure}

}

\clearpage

{\newText

\section{Comparative Analysis of Protein Networks Data Using Trans-CLIME}

We apply the Trans-CLIME algorithm proposed by~\cite{li2023transfer} to the protein network data analyzed in Section~\ref{sec:real-world-protein}. Consistent with the methodology described in Section~\ref{sec:real-world-protein}, we select the top 20 edges exhibiting the largest absolute values in the estimated precision matrices. The resulting graphs are presented in Figure~\ref{fig:AML_TransCLIME}.

Comparing these results to those in Figure~\ref{fig:AML}, we observe that both \myalg and Trans-CLIME produce identical graphs for subtypes M0 and M2, suggesting a high degree of similarity between these two subtypes in terms of their protein interaction networks. However, the graphs obtained via Trans-CLIME exhibit fewer shared edges overall, indicating that this method tends to highlight greater differences among subtypes. In contrast, the graphs estimated by \myalg display greater similarity across different subtypes.

\begin{figure}[t]
\centering
\includegraphics[width=.7\textwidth]{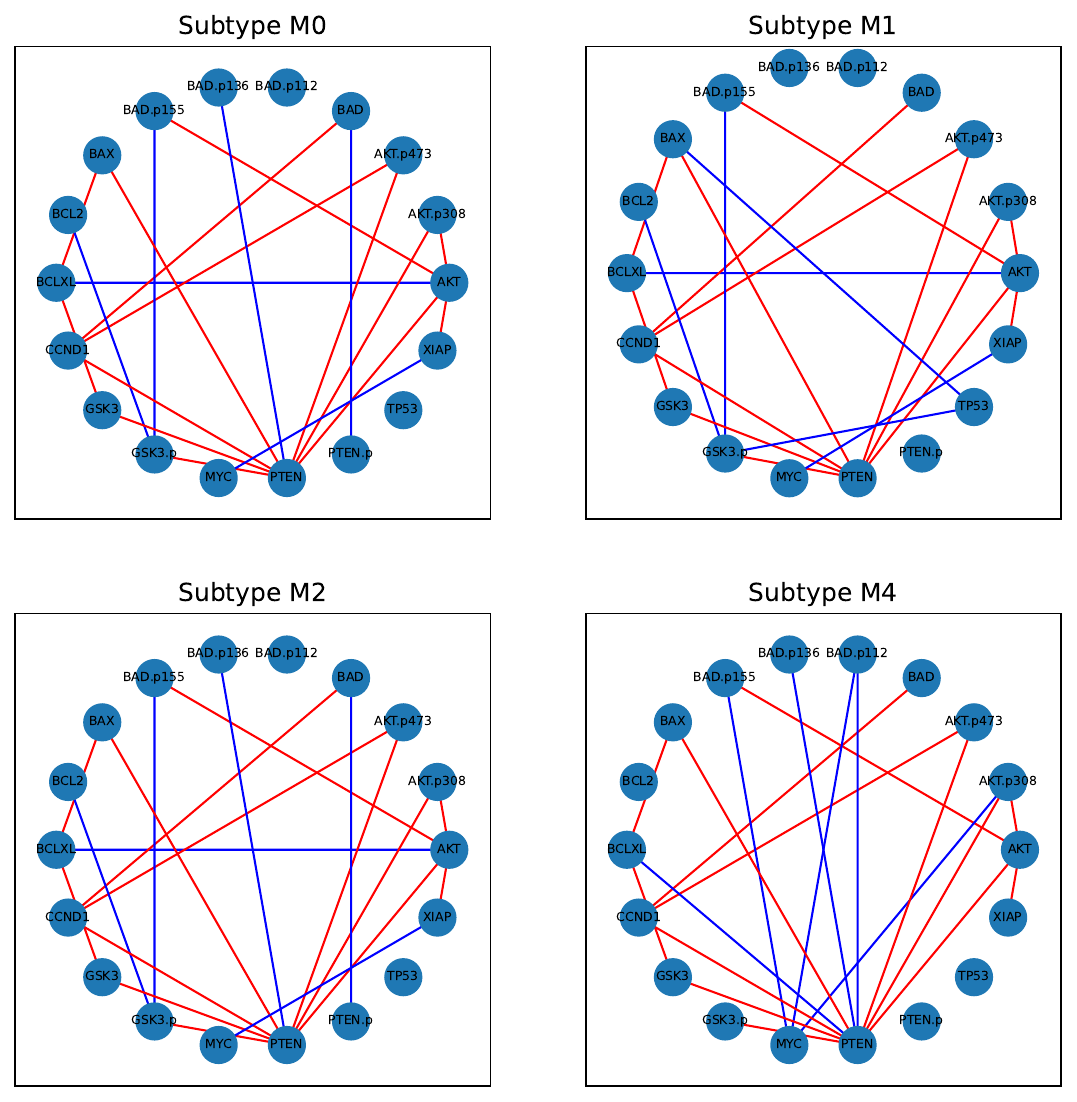}
\caption{Protein networks for four AML subtypes obtained by Trans-CLIME. Red edges are shared by all subtypes, blue edges are not.}
\label{fig:AML_TransCLIME}
\end{figure}

This difference in outcomes can be attributed to the distinct assumptions underlying each algorithm. Specifically, \myalg assumes that most entries in the precision matrices are shared across subtypes, leading to more similar graph structures. In contrast, Trans-CLIME relies on the assumption that the $L_p$-norm of the divergence matrix is small, which does not necessarily imply structural similarity. Determining which assumption is more biologically relevant to protein interactions requires deeper biological investigation, a question we leave for future work.

}

\clearpage

\section{Preliminary Lemmas}
We first collect several inequalities related to matrix norms. 
\begin{lemma}
\label{lemma:infty-err-prop}
For any two matrices $A,B \in \mathbb{R}^{d \times d}$, we have
\begin{equation}
\label{eq:infty-err-prop-1}
\vert A B \vert_{\infty} \leq \min \left\{ \Vert A \Vert_{\infty} \vert B \vert_{\infty}, \Vert B \Vert_1 \vert A \vert_{\infty}  \right\} .
\end{equation}
\end{lemma}
\begin{proof}
For $1\leq j \leq d$, denote by $B_{:j}$ the $j$-th column of $B$. 
We then have
\begin{align*}
\vert A B \vert_{\infty} & = \max_{1 \leq j \leq d} \Vert A B_{:j} \Vert_{\infty} \\
& \leq \max_{1 \leq j \leq d} \Vert A \Vert_{\infty} \Vert B_{:j} \Vert_{\infty} \\
& = \Vert A \Vert_{\infty} \max_{1 \leq j \leq d}  \Vert B_{:j} \Vert_{\infty} \\
& = \Vert A \Vert_{\infty} \vert B \vert_{\infty}.
\end{align*}
The other claim follows from the facts that $\vert A B \vert_{\infty} = \vert B^\top A^\top \vert_{\infty}$, and $\Vert B^\top \Vert_{\infty} = \Vert B \Vert_1$. 
\end{proof}

\begin{lemma}
\label{lemma:mat-vec-infty-bd}
For any matrix $A \in \mathbb{R}^{m \times n}$ and any vector $v \in \mathbb{R}^n$, we have $\Vert Av \Vert_{\infty} \leq \vert A \vert_{\infty} \Vert v \Vert_1$.
\end{lemma}
\begin{proof}
Let $A_{i\cdot}$ denote the $i$-th row of $A$, we then have
\begin{align*}
\Vert Av \Vert_{\infty} & = \max_{1 \leq i \leq m} \left\vert \langle A_{i\cdot}, v \rangle \right\vert \leq \max_{1 \leq i \leq m} \Vert A_{i\cdot} \Vert_{\infty} \Vert v \Vert_1 = \vert A \vert_{\infty} \Vert v \Vert_1,
\end{align*}
which completes the proof.
\end{proof}

The next lemma requires the following definition of the sub-Gaussian random variable.
\begin{definition}
[Sub-Gaussian Random Variable, Definition~2 of~\cite{ravikumar2011high}]
\label{def:sub-Gaussian-rv}
We say that a random variable $X \in \mathbb{R}$ is sub-Gaussian with parameter $\sigma$ if
\begin{equation*}
\mathbb{E}\left[ \exp \left( \lambda X \right) \right] \leq \exp \left( \sigma^2 \lambda / 2 \right) \quad \text{for all } \lambda \in \mathbb{R}.
\end{equation*}
\end{definition}

\begin{lemma}
\label{lemma:avg-emp-cov-infty}
Assume that we obtain samples from $K$ distributions, each with a mean of zero and a covariance matrix of $\Sigma^{(k)}$. 
Let $\{X^{(k)}_i=(X^{(k)}_{i1},\ldots,X^{(k)}_{id})^{\top}\}_{i=1}^{n_k}$ represent $n_k$ independently distributed samples from the $k$-th distribution. In addition, we assume that $X^{(k)}_{ij}/\sqrt{\Sigma^{(k)}_{jj}}$ is sub-Gaussian with parameter $\sigma$ as defined in Definition~\ref{def:sub-Gaussian-rv}.
Let
\begin{equation*}
\widehat{\Sigma}^{(k)}=\frac{1}{n_k}\sum^{n_k}_{i=1} X^{(k)}_{i} X^{(k)\top}_i,
\quad
\widehat{\Sigma}=\sum^K_{k=1} \alpha_k \widehat{\Sigma}^{(k)}, 
\quad 
\text{and}
\quad
\Sigma^{\star} = \sum^K_{k=1} \alpha_k \Sigma^{(k)},
\end{equation*}
where $\alpha_k = {n_k}/{N}$ and $N = \sum^K_{k=1} n_k$.
For a fixed $\delta \in (0,1]$, we have
\begin{equation*}
\mathbb{P} \left\{ \left\vert  \widehat{\Sigma} - \Sigma^{\star} \right\vert_{\infty} > 16(1+4\sigma^2) \left\{ \max_{1 \leq k \leq K} \left\vert \Sigma^{(k)} \right\vert_{\infty} \right\} \max \left\{  \sqrt{ \frac{ \log (2 d^2 / \delta) }{ 2 N } } , \, \frac{ \log (2 d^2 / \delta) }{ N } \right\} \right\} \leq \delta.
\end{equation*}
\end{lemma}

\begin{proof}
Let $\widebar{X}^{(k)}_{ij} = {X^{(k)}_{ij}}/{ \sqrt{\Sigma^{(k)}_{jj}}  }$,
$i=1,\ldots,n_k$, $k=1, \ldots, K$, $j=1,\ldots,d$. Define
$U^{(k)}_{i,jl} = \widebar{X}^{(k)}_{ij} + \widebar{X}^{(k)}_{il}$,
$V^{(k)}_{i,jl} = \widebar{X}^{(k)}_{ij} - \widebar{X}^{(k)}_{il}$,
and 
$\rho^{(k)}_{jl}= {\Sigma^{(k)}_{jl}}/{ \sqrt{ \Sigma^{(k)}_{jj} \cdot \Sigma^{(k)}_{ll} } }$.
Then it follows from Proposition 2.9 and (2.18) in \citet{wainwright2019high}, that 
\begin{align*}
& \mathbb{E} \left[ \exp \left\{  \lambda \left[ \left( U^{(k)}_{i,jl} \right)^2 - 2 \left( 1 + \rho^{(k)}_{jl} \right)  \right] \right\}  \right] \leq \exp \left( \frac{\lambda^2 \nu^2}{2} \right), \quad \text{and}\\
& \mathbb{E} \left[ \exp \left\{  \lambda \left[ \left( V^{(k)}_{i,jl} \right)^2 - 2 \left( 1 - \rho^{(k)}_{jl} \right)  \right] \right\}  \right] \leq \exp \left( \frac{\lambda^2 \nu^2}{2} \right),
\end{align*}
for all $\vert \lambda \vert < {1}/{\phi}$ where $\nu=\phi=16(1+4\sigma^2)$. Since
\begin{equation*}
\sum^{n_k}_{i=1} \left( \widebar{X}^{(k)}_{ij} \widebar{X}^{(k)}_{il} - \rho^{(k)}_{jl} \right) = \frac{1}{4} \sum^{n_k}_{i=1} \left( \left( U^{(k)}_{i,jl} \right)^2 - 2 \left( 1 + \rho^{(k)}_{jl} \right) \right) - \frac{1}{4} \sum^{n_k}_{i=1} \left( \left( V^{(k)}_{i,jl} \right)^2 - 2 \left( 1 - \rho^{(k)}_{jl} \right) \right),
\end{equation*}
we have
\begin{align*}
\mathbb{E} &\left[ \exp \left\{  \lambda \left[  \sum^{n_k}_{i=1} \left( \widebar{X}^{(k)}_{ij} \widebar{X}^{(k)}_{il} - \rho^{(k)}_{jl} \right) \right] \right\}  \right] \\
& = \mathbb{E} \left[ \exp \left\{  \frac{\lambda}{4} \left[  \sum^{n_k}_{i=1} \left( \left( U^{(k)}_{i,jl} \right)^2 - 2 \left( 1 + \rho^{(k)}_{jl} \right) \right) \right] \right\} \cdot \exp \left\{  - \frac{\lambda}{4} \left[  \sum^{n_k}_{i=1} \left( \left( V^{(k)}_{i,jl} \right)^2 - 2 \left( 1 - \rho^{(k)}_{jl} \right) \right) \right] \right\}  \right] \\
& \leq \left\{ \mathbb{E} \left[ \exp \left\{  \frac{\lambda}{2} \left[  \sum^{n_k}_{i=1} \left( \left( U^{(k)}_{i,jl} \right)^2 - 2 \left( 1 + \rho^{(k)}_{jl} \right) \right) \right] \right\} \right] \right\}^{\frac{1}{2}} \times \\
& \quad \quad \left\{ \mathbb{E} \left[ \exp \left\{  -\frac{\lambda}{2} \left[  \sum^{n_k}_{i=1} \left( \left( V^{(k)}_{i,jl} \right)^2 - 2 \left( 1 - \rho^{(k)}_{jl} \right) \right) \right] \right\} \right] \right\}^{\frac{1}{2}} \\
& \leq \exp \left( \frac{n_k \nu^2 \lambda^2 }{8} \right) \quad \text{for all } \vert \lambda \vert < \frac{2}{\phi}.
\end{align*}
Let $\tau_{\infty}=\max_{1 \leq k \leq K}\vert \Sigma^{(k)} \vert_{\infty}$.
It then follows that
\begin{align*}
\mathbb{E} & \left[ \exp \left\{  \lambda \left[  \sum^{n_k}_{i=1} \left( X^{(k)}_{ij} X^{(k)}_{il} - \Sigma^{(k)}_{jl} \right) \right] \right\}  \right] \\
& = \mathbb{E} \left[ \exp \left\{  \lambda \sqrt{ \Sigma^{(k)}_{jj} } \sqrt{ \Sigma^{(k)}_{ll} }  \left[  \sum^{n_k}_{i=1} \left( \widebar{X}^{(k)}_{ij} \widebar{X}^{(k)}_{il} - \rho^{(k)}_{jl} \right) \right] \right\}  \right] \\
& \leq \mathbb{E} \left[ \exp \left\{  \lambda \vert \Sigma^{(k)} \vert_{\infty} \left[  \sum^{n_k}_{i=1} \left( \widebar{X}^{(k)}_{ij} \widebar{X}^{(k)}_{il} - \rho^{(k)}_{jl} \right) \right] \right\}  \right] \\
& \leq \mathbb{E} \left[ \exp \left\{  \lambda \tau_{\infty} \left[  \sum^{n_k}_{i=1} \left( \widebar{X}^{(k)}_{ij} \widebar{X}^{(k)}_{il} - \rho^{(k)}_{jl} \right) \right] \right\}  \right] \\
& \leq \exp \left( \frac{n_k \nu^2 \tau^2_{\infty} \lambda^2 }{8  } \right) \quad \text{for all } \vert \lambda \vert < \frac{2}{\phi \tau_{\infty} },
\end{align*}
and
\begin{align*}
\mathbb{E} \left[ \exp \left\{ \lambda N \left( \widehat{\Sigma}_{jl} - \Sigma^{\star}_{jl} \right) \right\} \right] & = \mathbb{E} \left[ \exp \left\{ \lambda \left( \sum^K_{k=1} \sum^{n_k}_{i=1} \left( X^{(k)}_{ij} X^{(k)}_{il} - \Sigma^{(k)}_{jl} \right)  \right) \right\} \right] \\
& = \prod^K_{k=1} \mathbb{E} \left[ \exp \left\{ \lambda \left( \sum^{n_k}_{i=1} \left( X^{(k)}_{ij} X^{(k)}_{il} - \Sigma^{(k)}_{jl} \right)  \right) \right\} \right] \\
& \leq \prod^K_{k=1} \exp \left( \frac{n_k \nu^2 \tau^2_{\infty} \lambda^2 }{8  } \right) \\
& = \exp \left( \frac{N \nu^2 \tau^2_{\infty}  \lambda^2 }{8 } \right)
\end{align*}
when $\vert \lambda \vert < 2  / (\tau_{\infty} \phi)$.
Using Proposition 2.9 in \citet{wainwright2019high}, for $ t \geq 0$, we have that
\begin{equation*}
\mathbb{P} \left\{ \left\vert  \widehat{\Sigma}_{jl} - \Sigma^{\star}_{jl} \right\vert > t \right\} = \mathbb{P} \left\{ \left\vert N \left( \widehat{\Sigma}_{jl} - \Sigma^{\star}_{jl} \right) \right\vert > N t \right\} \leq 2 \exp \left\{ - \frac{1}{2} \min \left( \frac{4 N t^2}{ \nu^2 \tau^2_{\infty}}, \, \frac{2 N  t }{\phi \tau_{\infty}} \right)  \right\}.
\end{equation*}
A union bound then gives us
\begin{equation*}
\mathbb{P} \left\{ \left\vert  \widehat{\Sigma} - \Sigma^{\star} \right\vert_{\infty} > t \right\} \leq 2 d^2 \exp \left\{ - \frac{1}{2} \min \left( \frac{4 N t^2}{ \nu^2 \tau^2_{\infty}}, \, \frac{2 N  t }{\phi \tau_{\infty}} \right)  \right\}.
\end{equation*}
Rewriting the above equation, we complete the proof. 
\end{proof}

\begin{lemma}
\label{lemma:infty-cov-err}
Consider a zero mean random vector $X=(X_1,\ldots,X_d)^{\top}$ with covariance $\Sigma^{\star}$ such that each $X_j/\sqrt{\Sigma^{\star}_{jj}}$ is sub-Gaussian with parameter $\sigma$ as defined in Definition~\ref{def:sub-Gaussian-rv}.
Let $\{ X(i) \}^n_{i=1}$ be $n$ i.i.d.~copies of $X$ and let $\widehat{\Sigma}_n=\frac{1}{n}\sum^n_{i=1} X(i) X(i)^{\top}$.
For $\delta \in (0,1]$, we have
\begin{equation*}
\mathbb{P} \left\{ \left\vert  \widehat{\Sigma}_n - \Sigma^{\star} \right\vert_{\infty} > 16(1+4\sigma^2) \left\vert \Sigma^{\star} \right\vert_{\infty}  \max \left\{  \sqrt{ \frac{ \log (2 d^2 / \delta) }{ 2 n } } , \, \frac{ \log (2 d^2 / \delta) }{ n } \right\} \right\} \leq \delta.
\end{equation*}
\end{lemma}
\begin{proof}
The proof follows directly from Lemma~\ref{lemma:avg-emp-cov-infty}.
\end{proof}

In the next lemma, we use $\mathbb{S}^d_{+}$ to denote the set of positive definite matrices with dimension $d$.

\begin{lemma}
\label{lemma:likelihood-taylor-series}
Let $A \in \mathbb{S}^d_{+}$.
Define $f \left( \Omega \right) \coloneqq \langle \Omega, A \rangle - \log \det \Omega$ for $\Omega \in \mathbb{S}^d_{+}$.
Given $\Omega_0 \in \mathbb{S}^d_{+}$ and $\Delta_\Omega \in \mathbb{R}^{d \times d}$ such that $\Omega_0 + \Delta_\Omega \in \mathbb{S}^d_{+}$, there exists $t \in (0,1)$ such that
\begin{equation*}
f \left( \Omega_0 + \Delta_\Omega \right) - f \left(  \Omega_0 \right) - \left\langle \nabla f ( \Omega_0 ) , \Delta_\Omega \right\rangle = \frac{1}{2} {\rm vec} \left( \Delta_\Omega \right)^{\top} \left\{ \left( \Omega_0 + t \Delta_\Omega \right)^{-1} \otimes \left( \Omega_0 + t \Delta_\Omega \right)^{-1} \right\} {\rm vec} \left( \Delta_\Omega \right)
\end{equation*}
\end{lemma}
\begin{proof}
One can view $f$ as a function of $\text{vec} ( \Omega )$. Note that $\nabla^2 f( \text{vec} ( \Omega ) )=\Omega^{-1}\otimes \Omega^{-1}$. The rest of the proof then follows Taylor's Theorem~\citep[Theorem 2.1]{nocedal2006numerical}.
\end{proof}


\begin{lemma}
\label{lemma:cov-number}
Let $\mathbb{S}^p (s) \coloneqq \left\{ \theta \in \mathbb{R}^p: \, \Vert \theta \Vert_0 \leq s , \, \Vert \theta \Vert_2 \leq 1 \right\}$. There exists $\{ \theta^0, \theta^1,\ldots,\theta^M \} \subseteq \mathbb{S}^p (s)$ such that
\begin{enumerate}[label=(\roman*)]
\item $\theta^0=0$;
\item $\left\Vert \theta^j - \theta^k \right\Vert_2 \geq \frac{1}{4}$ for all $0 \leq j \neq k \leq M$;
\item $\log (M+1) \geq \frac{s}{2} \log \left( \frac{d-s}{s} \right)$.
\end{enumerate}
\end{lemma}
\begin{proof}
It follows from Example 15.16 in \citet{wainwright2019high} that there exists a ${1}/{2}$-packing of $\mathbb{S}^p (s)$, which we denote as $\{ \tilde{\theta}^0, \tilde{\theta}^1,\ldots,\tilde{\theta}^M \}$, such that
\begin{equation}
\label{eq:lemma-cov-number-1}
\left\Vert \tilde{\theta}^j - \tilde{\theta}^k \right\Vert_2 \geq \frac{1}{2} \qquad \text{for all } 0 \leq j \neq k \leq M 
\qquad\text{ and }\qquad
\log (M+1) \geq \frac{s}{2} \log \left( \frac{d-s}{s} \right).
\end{equation}
Without loss of generality, assume that $\Vert \tilde{\theta}^0 \Vert_2 \leq \Vert \tilde{\theta}^j \Vert_2$ for $1 \leq j \leq M$. 
Let $\theta^0=0$ and let $\theta^j=\tilde{\theta}^j$ for $1 \leq j \leq M$.
Then the set $\{ \theta^0, \theta^1,\ldots,\theta^M \}$ satisfies (i)--(iii). 
To prove the claim, by our construction of the set and~\eqref{eq:lemma-cov-number-1}, we only need to verify that $\Vert \theta^j - \theta^0  \Vert_2 = \Vert \theta^j  \Vert_2 \geq \frac{1}{4}$ for all $1 \leq j \leq M$. We prove the result in two cases.

\textit{Case 1: $\Vert \tilde{\theta}^0 \Vert_2 \geq \frac{1}{4}$.} Since $\Vert \tilde{\theta}^0 \Vert_2 \leq \Vert \tilde{\theta}^j \Vert_2$, we have that $\Vert \theta^j  \Vert_2 = \Vert \tilde{\theta}^j  \Vert_2 \geq \Vert \tilde{\theta}^0  \Vert_2 \geq \frac{1}{4}$ for $1 \leq j \leq M$.

\textit{Case 2: $\Vert \tilde{\theta}^0 \Vert_2 < \frac{1}{4}$.} We have
\begin{equation*}
\left\Vert \theta^j  \right\Vert_2 = \left\Vert \tilde{\theta}^j  \right\Vert_2 \geq \left\Vert \tilde{\theta}^j - \tilde{\theta}^0 \right\Vert_2 - \left\Vert \tilde{\theta}^0  \right\Vert_2 > \frac{1}{2} - \frac{1}{4} = \frac{1}{4} \quad \text{for } 1 \leq j \leq M,
\end{equation*}
where the last inequality follows from~\eqref{eq:lemma-cov-number-1}.

Combining Case 1 and Case 2, we have proved that $\Vert \theta^j  \Vert_2 \geq \frac{1}{4}$ for $1 \leq j \leq M$, which completes the proof.
\end{proof}

\begin{lemma}
\label{lemma:KL-divergence}
Let $m \geq 2$ be an even integer and set $d = 2m$. Let $B \in \mathbb{R}^{m \times m}$  be such that $\Vert B \Vert_{\mathrm{F}} \leq 1/4$ and define
\begin{equation*}
\Omega = 
\begin{bmatrix}
I_m & B \\
B^{\top} & I_m
\end{bmatrix}.
\end{equation*}
Then $\Omega \succ 0$ and
$
D_{\text{KL}} \left( N \left( 0, \Omega^{-1} \right) \Vert  N \left( 0, I_d \right) \right) \leq \frac{16}{15} \left\Vert B \right\Vert^2_{\mathrm{F}}.
$
\end{lemma}
\begin{proof}
We begin with proving $\Omega \succ 0$.
By Weyl's inequality, we know that  
\begin{align*}
\gamma_{\min} \left( \Omega \right) & \geq 1 - \left\Vert 
\begin{bmatrix}
0 & B \\
B^{\top} & 0
\end{bmatrix}
\right\Vert \geq  1 -  \left\Vert 
\begin{bmatrix}
0 & B \\
B^{\top} & 0
\end{bmatrix}
\right\Vert_{\mathrm{F}} \\
& = 1 - \sqrt{2} \left\Vert B \right\Vert_{\mathrm{F}} \geq \frac{1}{2}.
\end{align*}
As a result, we have $\Omega \succ 0$.

Now we move on to the second claim regarding the KL divergence. Recall that
\begin{equation}
\label{eq:lemma-KL-divergence-1}
D_{\text{KL}} \left( N \left( 0, \Omega^{-1} \right) \Vert  N \left( 0, I_d \right) \right) = \frac{1}{2} \left[ \log \det \left( \Omega \right) - d + \tr \left( \Omega^{-1} \right) \right],
\end{equation}
which motivates us to compute $\log \det \left( \Omega \right)$ and $\tr \left( \Omega^{-1} \right)$.
Let $B = U D V^{\top}$ be the singular value decomposition of $B$ with $U,D,V \in \mathbb{R}^{m \times m}$, $U^{\top} U = U U^{\top} = I_m$, $V^{\top} V = V V^{\top} = I_m$, and $D=\text{diag}(\lambda_{1},\ldots,\lambda_m)$. We have the following identities. \begin{subequations}
\begin{align}
    \log \det \left( \Omega \right) &= \sum^m_{i=1} \log \left( 1- \lambda^2_i  \right). \label{eq:lemma-KL-divergence-2} \\ 
\label{eq:lemma-KL-divergence-6}
\tr \left( \Omega^{-1} \right) &= \sum^m_{i=1} \frac{2}{1 - \lambda^2_i}.
\end{align}
\end{subequations}
Combining~\eqref{eq:lemma-KL-divergence-1}, \eqref{eq:lemma-KL-divergence-2}, and~\eqref{eq:lemma-KL-divergence-6}, we obtain
\begin{align*}
D_{\text{KL}} \left( N \left( 0, \Omega^{-1} \right) \Vert  N \left( 0, I_d \right) \right) & = \frac{1}{2} \left[ \sum^m_{i=1} \log \left( 1- \lambda^2_i  \right) + \sum^m_{i=1} \frac{2}{1 - \lambda^2_i} - d \right] \\
& = \sum^m_{i=1} \left[ \frac{\log \left( 1- \lambda^2_i  \right)}{2} + \frac{1}{1 - \lambda^2_i} - 1 \right].
\end{align*}
Since $\lambda^2_j \leq \sum^d_{i=1} \lambda^2_i = \Vert B \Vert^2_{\mathrm{F}} \leq {1}/{16}$, we have $-\lambda^2_j \geq -{1}/{16}$. Also note that $\log(1+x) \leq x$ for all $x \geq -1$. We thus have
\begin{align*}
D_{\text{KL}} \left( N \left( 0, \Omega^{-1} \right) \Vert  N \left( 0, I_d \right) \right) & \leq \sum^m_{i=1} \left[ \frac{ - \lambda^2_i }{2} + \frac{1}{1 - \lambda^2_i} - 1 \right] \\
& = \sum^m_{i=1} \frac{ \lambda^2_i + \lambda^4_i  }{ 2 \left( 1 - \lambda^2_i \right)  } \leq \sum^m_{i=1} \frac{ \lambda^2_i }{ 1 - \lambda^2_i } \leq \frac{16}{15} \sum^m_{i=1} \lambda^2_i = \frac{16}{15} \left\Vert B \right\Vert^2_{\mathrm{F}} .
\end{align*}
This completes the proof.

\paragraph{Proof of Equation~\eqref{eq:lemma-KL-divergence-2}.}
Using Section 9.1.2 of~\citet{petersen2008matrix}, we have
$\det \left( \Omega \right) = \det \left( I_m - B^{\top} B \right)$.
Since 
\begin{equation}
\label{eq:KL-lemma-eigen-decomp}    
I_m - B^{\top} B = I_m - V D^2 V^{\top} = V \left( I_m - D^2 \right) V^{\top}, 
\end{equation}
we have $\det \left( \Omega \right) = \det \left( I_m - D^2 \right) = \prod^m_{i=1} \left( 1- \lambda^2_i  \right)$ and hence 
$
\log \det \left( \Omega \right) = \sum^m_{i=1} \log \left( 1- \lambda^2_i  \right).
$

\paragraph{Proof of Equation~\eqref{eq:lemma-KL-divergence-6}.}

Using Section 9.1.3 of~\citet{petersen2008matrix}, we have
\begin{equation*}
\Omega^{-1} = 
\begin{bmatrix}
I_m & B \\
B^{\top} & I_m
\end{bmatrix}^{-1} =
\begin{bmatrix}
\left( I_m - B B^{\top} \right)^{-1} & - B \left( I_m - B^{\top} B \right)^{-1} \\
- \left( I_m - B^{\top} B \right)^{-1} B^{\top} & \left( I_m - B^{\top} B \right)^{-1}
\end{bmatrix},
\end{equation*}
which implies
\begin{equation}
\label{eq:lemma-KL-divergence-tr}
\tr \left( \Omega^{-1} \right) = \tr \left\{ \left( I_m - B B^{\top} \right)^{-1} \right\} + \tr \left\{ \left( I_m - B^{\top}  B \right)^{-1} \right\}.
\end{equation}
It follows from~\eqref{eq:KL-lemma-eigen-decomp} that $\left( I_m - B^{\top}  B \right)^{-1} = V \left( I_m - D^2 \right)^{-1} V^{\top}$, and therefore,
\begin{equation}
\label{eq:lemma-KL-divergence-3}
\tr \left\{ \left( I_m - B B^{\top} \right)^{-1} \right\} = \tr \left\{ V \left( I_m - D^2 \right)^{-1} V^{\top} \right\} = \tr \left\{ \left( I_m - D^2 \right)^{-1} \right\} = \sum^m_{i=1} \frac{1}{1 - \lambda^2_i}.
\end{equation}
Similarly, we have 
$
\tr \{ \left( I_m - B^{\top} B  \right)^{-1} \}  = \sum^m_{i=1} \frac{1}{1 - \lambda^2_i}.
$
Take the previous results collectively to yield the claim~\eqref{eq:lemma-KL-divergence-6}.
\end{proof}

\section{Proof of Theorem~\ref{thm:multi-task-frob}}
\label{sec:proof-thm-multi-task-frob}

We adopt the proof strategy laid out in~\citet[Chapter 9]{wainwright2019high}.

\subsection{Additional Notation}

We first introduce some additional notation that will be helpful in the proof. 
We define the Hilbert space for the parameters
\begin{equation*}
\mathbb{H} \coloneqq \left\{ \Theta = \left( \Omega, \left\{ \uniqueparameter \right\}^K_{k=0} \right) : \, \Omega, \uniqueparameter \in \mathbb{R}^{d \times d} \text{ for all } 0 \leq k \leq K \right\},
\end{equation*}
with the associated inner product
\begin{equation*}
\left\langle \Theta, \Theta^{\prime} \right\rangle_{\mathbb{H}} = \left\langle \Theta, \Theta^{\prime} \right\rangle \coloneqq \left\langle \Omega, \Omega^{\prime} \right\rangle + \sum^K_{k=0} \left\langle \uniqueparameter, \Gamma^{(k)\prime} \right\rangle.
\end{equation*}
The space $\mathbb{H}$ endowed with the inner product $\langle \cdot,\cdot \rangle$ is indeed a Hilbert space. 
Correspondingly, we have
\begin{equation*}
\left\Vert \Theta \right\Vert^2_{\mathbb{H}} = \left\Vert \Omega \right\Vert^2_{\mathrm{F}} + \sum^K_{k=0} \left\Vert \uniqueparameter \right\Vert^2_{\mathrm{F}}.
\end{equation*}
We also need the dual norm of $\Phi(\cdot)$, which is defined as
\begin{equation}
\label{eq:dual-norm}
\Phi^{\star} \left( \Theta \right) \coloneqq \sup_{\Theta^{\prime}:\Phi(\Theta^{\prime}) \leq 1} \left\langle \Theta, \Theta^{\prime} \right\rangle.
\end{equation}

Recall that $\Theta^{\star}=(\shared, \{ \unique \}^K_{k=0})$ is the true parameter.
Let $\supportshared$ and $\supportunique$ be the supports of $\shared$ and $\unique$ for all $0 \leq k \leq K$, respectively.
Under Assumption~\ref{assump:model-structure}, the true parameter $\Theta^{\star}$ lies in the following subspace of $\mathbb{H}$:
\begin{equation}
\label{eq:subspace-M}
\mathbb{M} \coloneqq \left\{ \Theta = \left( \Omega, \left\{ \uniqueparameter \right\}^K_{k=0} \right) : \, \supp(\Omega) \subseteq \supportshared , \, \supp\left[ \uniqueparameter \right] \subseteq \supportunique \text{ for all } 0 \leq k \leq K \right\}.
\end{equation}
The orthogonal complement of $\mathbb{M}$ is given by
\begin{equation}
\label{eq:subspace-M-bot}
\mathbb{M}^{\bot} \coloneqq \left\{ \Theta = \left( \Omega, \left\{ \uniqueparameter \right\}^K_{k=0} \right) : \, \supp(\Omega) \subseteq \supportshared^c , \, \supp\left[ \uniqueparameter \right] \subseteq \supportunique^{c} \text{ for all } 0 \leq k \leq K \right\}.
\end{equation}
Clearly, for any $\Theta \in \mathbb{M}$ and $\Theta^{\prime} \in \mathbb{M}^{\bot}$, we have that $\langle \Theta, \Theta^{\prime} \rangle=0$ .

We also need to define the projection of a parameter onto a subspace. For any matrix $B \in \mathbb{R}^{d \times d}$ and any subspace $\mathcal{F} \subseteq \mathbb{R}^{d \times d}$, we define
\begin{equation*}
\left[ B \right]_{\mathcal{F}} = \arg \min_{\tilde{B} \in \mathcal{F}} \left\Vert \tilde{B} - B \right\Vert_{\mathrm{F}}.
\end{equation*}
Similarly, for any $\Theta \in \mathbb{H}$ and any subspace $\mathbb{F} \subseteq \mathbb{H}$, we define
\begin{equation*}
\left[ \Theta \right]_{\mathbb{F}} = \arg \min_{\tilde{B} \in \mathbb{F}} \left\Vert \tilde{B} - B \right\Vert_{\mathbb{H}}.
\end{equation*}
In addition, for $S \subseteq [d] \times [d]$, we define
\begin{equation*}
\mathcal{M}(S) \coloneqq \left\{ B \in \mathbb{R}^{d \times d}: \supp( B) \subseteq  S \right\}.
\end{equation*}
For any $B \in \mathbb{R}^{d \times d}$, we define $\left[ B \right]_S \coloneqq \left[ B \right]_{\mathcal{M}(S)}$.
This way, for any $\Theta = ( \Omega, \{ \uniqueparameter \}^K_{k=0} ) \in \mathbb{H}$, we have
\begin{equation*}
\left[ \Theta \right]_{\mathbb{M}} = \left( \left[ \Omega \right]_{ \supportshared }, \left\{ \left[ \uniqueparameter \right]_{\supportunique} \right\}^K_{k=0} \right) \qquad \text{and} \qquad \left[ \Theta \right]_{\mathbb{M}^{\bot}} = \left( \left[ \Omega \right]_{ \supportshared^c }, \left\{ \left[ \uniqueparameter \right]_{ \supportunique^c } \right\}^K_{k=0} \right).
\end{equation*}


Finally, we define our metric on the estimation error. For $\Theta = ( \Omega, \{ \uniqueparameter \}^K_{k=0} ) \in \mathbb{H}$, let
\begin{equation}
\label{eq:err-measure}
H (\Theta) \coloneqq \sum^K_{k=0} \alpha_k \left\Vert \Omega + \uniqueparameter\right\Vert^2_{\mathrm{F}}.
\end{equation}
We define
\begin{equation*}
\widehat \Delta \coloneqq \widehat{\Theta} - \Theta^{\star} = \left( \widehat{\Omega}-\shared, \left\{ \uniqueestimate - \unique \right\}^K_{k=0} \right),
\end{equation*}
where $\widehat{\Theta}$ is the \myalg~estimator~\eqref{eq:Glasso-mutli-task}. Then our goal is to control 
\begin{equation*}
H ( \widehat{\Delta} ) = \sum^K_{k=0} \alpha_k \left\Vert \widehat{\Omega} - \shared +  \uniqueestimate - \unique \right\Vert^2_{\mathrm{F}} = \sum^K_{k=0} \alpha_k \left\Vert \steponeoutput - \Omega^{(k)} \right\Vert^2_{\mathrm{F}}.
\end{equation*}

\subsection{Useful lemmas}
Now we collect several useful lemmas, whose proofs are deferred to the end of the section.  

We begin by demonstrating that $\Phi(\cdot)$, defined in~\eqref{eq:regularizer-def}, is decomposable with respect to $(\mathbb{M},\mathbb{M}^{\bot})$.

\begin{lemma}
\label{lemma:decomposable}
We have
\begin{equation*}
\Phi(\Theta + \Theta^{\prime}) = \Phi(\Theta) + \Phi(\Theta^{\prime}) \quad \text{for any } \Theta \in \mathbb{M}, \, \Theta^{\prime} \in \mathbb{M}^{\bot}.
\end{equation*}
\end{lemma}

The following lemma relates $\Phi(\Theta)$ with $H(\Theta)$ defined in~\eqref{eq:err-measure} for any $\Theta \in \mathbb{M}$.
\begin{lemma}
\label{lemma:bound-Phi-H}
For any $\Theta \in \mathbb{M}$, we have
\begin{equation*}
\Phi(\Theta) \leq \sqrt{2} \sqrt{ s + ( K+1 ) h } \sqrt{ H (\Theta) }.
\end{equation*}
\end{lemma}

For $\Delta_\Theta = ( \Delta_\Omega, \{ \Delta_{\uniqueparameter} \}^K_{k=0} ) \in \mathbb{H}$ such that $ \Delta_\Omega + \Delta_{\uniqueparameter} + \Omega^{(k)} \succ 0$ for all $0 \leq k \leq K$,
we define
\begin{equation*}
\mathcal{R}\left( \Delta_\Theta \right) \coloneqq \mathcal{L} \left( \Theta^{\star} + \Delta_\Theta \right) - \mathcal{L} \left( \Theta^{\star} \right) - \langle \nabla \mathcal{L} \left( \Theta^{\star} \right), \Delta_\Theta \rangle
\end{equation*}
to be the residual of $\mathcal{L}(\cdot)$ around $\Theta^{\star}$, where $\mathcal{L}(\cdot)$ is defined in~\eqref{eq:loss-fun-def}.
The following lemma claims that $\mathcal{R}\left( \cdot \right)$ is locally strongly convex with respect to the geometry defined by $H(\cdot)$.


{
\begin{lemma}
\label{lemma:rsc-H-new}
Let $\Delta_\Theta = ( \Delta_\Omega, \{ \Delta_{\uniqueparameter} \}^K_{k=0} ) \in \mathbb{H}$ be such that $ \Delta_\Omega + \Delta_{\uniqueparameter} + \Omega^{(k)} \succ 0$ and assume that $\left\Vert \Delta_\Omega + \Delta_{\uniqueparameter} \right\Vert_2 \leq M_{\mathrm{op}} + M_{\Omega}$ for all $0 \leq k \leq K$. Then, we have
\begin{equation*}
\mathcal{R}\left( \Delta_\Theta \right) \geq \frac{\kappa}{2} H \left( \Delta_\Theta \right),
\end{equation*}
where $\kappa = \left( 2 M_{\Omega} + M_{\mathrm{op}}  \right)^{-2}$.
\end{lemma}
}

Last but not least, define the event
\begin{equation}
\label{eq:good-event}
\mathbb{G}(\lambda_{\text{M}}) \coloneqq \left\{ \frac{\lambda_{\text{M}}}{2} \geq \Phi^{\star} \left( \nabla \mathcal{L} \left( \Theta^{\star} \right) \right) \right\}.
\end{equation}
The following lemma asserts that this event happens with high probability.
\begin{lemma}
\label{lemma:good-event-hold}
When~\eqref{eq:min-nk-large} holds and $\lambda_{\text{M}}$ satisfies~\eqref{eq:mlut-Frob-lambda-large-enough}, we have that $\mathbb{P} \{ \mathbb{G}(\lambda_{\text{M}}) \} \geq 1 - \delta$.
\end{lemma}

\subsection{Remaining proof}
\label{sec:remaining-proof}

Now we are ready to prove Theorem~\ref{thm:multi-task-frob}. As Lemma~\ref{lemma:good-event-hold} asserts that $\mathbb{G}(\lambda_{\text{M}})$ holds with high probability, thus we only need to prove the conclusion under the assumption that $\mathbb{G}(\lambda_{\text{M}})$ holds.
Throughout the proof, we assume the event $\mathbb{G}(\lambda_{\text{M}})$ holds. 

For any $\Delta_\Theta = ( \Delta_\Omega, \{ \Delta_{\uniqueparameter}  \}^K_{k=0} ) \in \mathbb{H}$, we define the objective difference
\begin{equation*}
F \left( \Delta_\Theta \right) \coloneqq \mathcal{L} \left( \Theta^{\star} + \Delta_\Theta \right) - \mathcal{L} \left( \Theta^{\star}  \right) + \lambda_{\text{M}} \left\{ \Phi \left( \Theta^{\star} + \Delta_\Theta \right) - \Phi \left( \Theta^{\star} \right) \right\}.
\end{equation*}
Note that by definition, we have $F \left( \widehat{\Delta}_\Theta \right) \leq 0$.
It suffices to show that for all $\Delta_\Theta$ such that $\Delta_\Theta + \Theta^\star \in \mathcal{C}(M_{\mathrm{op}})$ and $H(\Delta_\Theta)> \frac{18 \left( s +  ( K+1 ) h \right)\lambda^2_{\text{M}}}{ \kappa^2}$, we have $F(\Delta_\Theta) >0$. 

It is easy to see that any such $\Delta_\Theta$ obeys
\begin{equation*}
\Delta {\Omega} + \Delta {\Gamma}_k + \Omega^{(k)} = {\Omega} - \shared + {\Gamma}_k - \unique + \Omega^{(k)} = {\Omega} + {\Gamma}_k \succ 0,
\end{equation*}
and for all $0 \leq k \leq K$
\begin{equation*}
\left\Vert \Delta {\Omega} + \Delta {\Gamma}_k \right\Vert_2 = \left\Vert {\Omega} - \shared + {\Gamma}_k - \unique \right\Vert_2 = \left\Vert {\Omega} + {\Gamma}_k - \Omega^{(k)} \right\Vert_2 \leq \left\Vert {\Omega} + {\Gamma}_k  \right\Vert_2 + \left\Vert \Omega^{(k)} \right\Vert_2 \leq M_{\mathrm{op}} + M_{\Omega},
\end{equation*}
where the last inequality follows the definition of $M_{\mathrm{op}}$ and $M_{\Omega}$.
These two taken together allows us to invoke Lemma~\ref{lemma:rsc-H-new} to obtain
\begin{equation*}
F \left( \Delta_\Theta \right) \geq \langle \nabla \mathcal{L} \left( \Theta^{\star} \right), \Delta_\Theta \rangle + \frac{\kappa}{2} H \left( \Delta_\Theta \right) + \lambda_{\text{M}} \left\{ \Phi \left( \Theta^{\star} + \Delta_\Theta \right) - \Phi \left( \Theta^{\star} \right) \right\},
\end{equation*}
where $\kappa = \left( 2 M_{\Omega} + M_{\mathrm{op}}  \right)^{-2}$.

In addition, combining Lemma~\ref{lemma:decomposable} and the fact that $\left[ \Theta^{\star} \right]_{\mathbb{M}^{\bot}}=0$ with \citet[Lemma 9.14]{wainwright2019high}, we have
\begin{equation}
\label{eq:equ9.32}
\Phi \left( \Theta^{\star} + \Delta_\Theta \right) - \Phi \left( \Theta^{\star} \right) \geq \Phi \left( \left[ \Delta_\Theta \right]_{\mathbb{M}^{\bot}} \right) - \Phi \left( \left[ \Delta_\Theta \right]_{\mathbb{M}} \right)
\end{equation}
for any $\Delta_\Theta \in \mathbb{H}$. This way, we have obtained a lower bound for $\Phi ( \Theta^{\star} + \Delta_\Theta ) - \Phi ( \Theta^{\star} )$.

Take the previous two displays together to reach 
\begin{equation*}
F \left( \Delta_\Theta \right) \geq \langle \nabla \mathcal{L} \left( \Theta^{\star} \right), \Delta_\Theta \rangle + \frac{\kappa}{2} H \left( \Delta_\Theta \right) + \lambda_{\text{M}} \left\{ \Phi \left( \left[ \Delta_\Theta \right]_{\mathbb{M}^{\bot}} \right) - \Phi \left( \left[ \Delta_\Theta \right]_{\mathbb{M}} \right) \right\}.
\end{equation*}
Recall the definition of $\Phi^{\star}$ in~\eqref{eq:dual-norm}, under the assumption that $\mathbb{G}(\lambda_{\text{M}})$ is true, we have
\begin{align*}
F \left( \Delta_\Theta \right) & \geq \frac{\kappa}{2} H \left( \Delta_\Theta \right) - \left\vert \langle \nabla \mathcal{L} \left( \Theta^{\star} \right), \Delta_\Theta \rangle \right\vert + \lambda_{\text{M}} \left\{ \Phi \left( \left[ \Delta_\Theta \right]_{\mathbb{M}^{\bot}} \right) - \Phi \left( \left[ \Delta_\Theta \right]_{\mathbb{M}} \right) \right\} \\
& \geq \frac{\kappa}{2} H \left( \Delta_\Theta \right) - \Phi \left( \Delta_\Theta \right) \Phi^{\star} \left( \nabla \mathcal{L} \left( \Theta^{\star} \right) \right) + \lambda_{\text{M}} \left\{ \Phi \left( \left[ \Delta_\Theta \right]_{\mathbb{M}^{\bot}} \right) - \Phi \left( \left[ \Delta_\Theta \right]_{\mathbb{M}} \right) \right\} \\
& \geq \frac{\kappa}{2} H \left( \Delta_\Theta \right) - \frac{\lambda_{\text{M}}}{2} \Phi \left( \Delta_\Theta \right) +\lambda_{\text{M}} \left\{ \Phi \left( \left[ \Delta_\Theta \right]_{\mathbb{M}^{\bot}} \right) - \Phi \left( \left[ \Delta_\Theta \right]_{\mathbb{M}} \right) \right\} \\
& = \frac{\kappa}{2} H \left( \Delta_\Theta \right) - \frac{\lambda_{\text{M}}}{2} \left\{ \Phi \left( \left[ \Delta_\Theta \right]_{\mathbb{M}} \right) + \Phi \left( \left[ \Delta_\Theta \right]_{\mathbb{M}^{\bot}} \right) \right\} +\lambda_{\text{M}} \left\{ \Phi \left( \left[ \Delta_\Theta \right]_{\mathbb{M}^{\bot}} \right) - \Phi \left( \left[ \Delta_\Theta \right]_{\mathbb{M}} \right) \right\} \\
& = \frac{\kappa}{2} H \left( \Delta_\Theta \right) - \frac{\lambda_{\text{M}}}{2} \left\{ 3 \Phi \left( \left[ \Delta_\Theta \right]_{\mathbb{M}} \right) - \Phi \left( \left[ \Delta_\Theta \right]_{\mathbb{M}^{\bot}} \right) \right\} \\
& \geq \frac{\kappa}{2} H \left( \Delta_\Theta \right) - \frac{3\lambda_{\text{M}}}{2} \Phi \left( \left[ \Delta_\Theta \right]_{\mathbb{M}} \right). \numberthis \label{eq:proof-thm-frob-1}
\end{align*}
By Lemma~\ref{lemma:bound-Phi-H}, we have
\begin{equation*}
\Phi \left( \left[ \Delta_\Theta \right]_{\mathbb{M}} \right) \leq \sqrt{ s +  ( K+1 ) h } \sqrt{ H \left( \left[ \Delta_\Theta \right]_{\mathbb{M}} \right) } \leq \sqrt{2} \sqrt{ s + ( K+1 ) h } \sqrt{ H \left( \Delta_\Theta \right) },
\end{equation*}
where the last inequality follows from $H(\Delta_\Theta)=H \left( \left[ \Delta_\Theta \right]_{\mathbb{M}} \right) + H \left( \left[ \Delta_\Theta \right]_{\mathbb{M}^{\bot}} \right) \geq H \left( \left[ \Delta_\Theta \right]_{\mathbb{M}} \right)$.
Plugging the above inequality into~\eqref{eq:proof-thm-frob-1}, we arrive at the conclusion that 
\begin{align*}
F \left( \Delta_\Theta \right) &\geq \frac{\kappa}{2} H \left( \Delta_\Theta \right) - \frac{3\lambda_{\text{M}}}{\sqrt{2}} \sqrt{ s +  ( K+1 ) h } \sqrt{ H \left( \Delta_\Theta \right) } \\
&= \frac{1}{2} \sqrt{ H \left( \Delta_\Theta \right) } \left\{ \kappa \sqrt{H \left( \Delta_\Theta \right)} -  3 \sqrt{2} \lambda_{\text{M}} \sqrt{ s +  ( K+1 ) h }\right\} >0,
\end{align*}
where the last relation arises from the assumption 
\begin{equation}
\label{eq:proof-thm-frob-4}
H \left( \Delta_\Theta \right) > \frac{18 \left( s +  ( K+1 ) h \right)\lambda^2_{\text{M}}}{\kappa^2}. 
\end{equation}
This finishes the proof. 

\subsection{Proof of Useful Lemmas}
In this section, we collect the proof of useful lemmas. 

\subsubsection{Proof of Lemma~\ref{lemma:decomposable}}
\label{sec:proof-lemma-decomposable}

By the definition of $\Phi(\cdot)$ and the fact that $\Theta \in \mathbb{M}$ and $\Theta^{\prime} \in \mathbb{M}^{\bot}$, we have
\begin{align*}
\Phi(\Theta + \Theta^{\prime}) & = \Phi \left( \left[ \Theta \right]_{\mathbb{M}} + \left[ \Theta^{\prime} \right]_{\mathbb{M}^{\bot}} \right) \\
& = \left\vert \left[ \Omega \right]_{\supportshared} + \left[ \Omega^{\prime} \right]_{\supportshared^c} \right\vert_1 +  \sum^K_{k=0} \sqrt{\alpha_k} \left\vert \left[ \uniqueparameter \right]_{\supportunique} + \left[ \Gamma^{(k)\prime} \right]_{ \supportunique^c  } \right\vert_1 \\
& = \left\vert \left[ \Omega \right]_{\supportshared} \right\vert_1 + \left\vert \left[ \Omega^{\prime} \right]_{\supportshared^c} \right\vert_1 +  \sum^K_{k=0} \sqrt{\alpha_k} \left\vert \left[ \uniqueparameter \right]_{\supportunique} \right\vert_1 +  \sum^K_{k=0} \sqrt{\alpha_k} \left\vert \left[ \Gamma^{(k)\prime} \right]_{ \supportunique^c  } \right\vert_1 \\
& = \Phi \left(  \left[ \Theta \right]_{\mathbb{M}} \right) + \Phi \left( \left[ \Theta^{\prime} \right]_{\mathbb{M}^{\bot}} \right) \\
& = \Phi(\Theta ) + \Phi(\Theta^{\prime}).
\end{align*}

\subsubsection{Proof of Lemma~\ref{lemma:bound-Phi-H}}
\label{sec:proof-lemma-bound-Phi-H}

For any $\Theta \in \mathbb{M}$, we have $\Omega=\left[ \Omega \right]_{\supportshared}$ and $\uniqueparameter=\left[ \uniqueparameter \right]_{ \supportunique }$.
Thus
\begin{align*}
\Phi ( \Theta ) & = \left\vert \left[ \Omega \right]_{\supportshared}  \right\vert_1 +  \sum^K_{k=0} \sqrt{\alpha_k}  \left\vert \left[ \uniqueparameter \right]_{ \supportunique }  \right\vert_1 \\
& \leq \vert \supportshared \vert^{\frac{1}{2}} \left\Vert \left[ \Omega \right]_{\supportshared}  \right\Vert_{\mathrm{F}} +  \sum^K_{k=0} \sqrt{\alpha_k} \left\vert \supportunique \right\vert^{\frac{1}{2}} \left\Vert \left[ \uniqueparameter \right]_{ \supportunique }  \right\Vert_{\mathrm{F}} 
&& (\text{Jensen's inequality})
\\
& \overset{(ii)}{\leq} \sqrt{s} \left\Vert \left[ \Omega \right]_{\supportshared}  \right\Vert_{\mathrm{F}} +  \sqrt{h} \sum^K_{k=0} \sqrt{\alpha_k}  \left\Vert \left[ \uniqueparameter \right]_{ \supportunique }  \right\Vert_{\mathrm{F}} .
&& (\text{by Assumption~\ref{assump:model-structure}})
\end{align*}
Furthermore, by Jensen's inequality, we have
\begin{align*}
\frac{1}{2}\Phi^2 ( \Theta ) & \leq  s \left\Vert \left[ \Omega \right]_{\supportshared}  \right\Vert^2_{\mathrm{F}} +  h \left( \sum^K_{k=0} \sqrt{\alpha_k}  \left\Vert \left[ \uniqueparameter \right]_{ \supportunique }  \right\Vert_{\mathrm{F}} \right)^2  \\
& \leq  s \left\Vert \left[ \Omega \right]_{\supportshared}  \right\Vert^2_{\mathrm{F}} +  h ( K+1 ) \sum^K_{k=0} \alpha_k  \left\Vert \left[ \uniqueparameter \right]_{ \supportunique }  \right\Vert^2_{\mathrm{F}}   \\
& \leq \left( s +  ( K+1 ) h \right) \left\{ \left\Vert \left[ \Omega \right]_{\supportshared}  \right\Vert^2_{\mathrm{F}} + \sum^K_{k=0} \alpha_k  \left\Vert \left[ \uniqueparameter \right]_{ \supportunique }  \right\Vert^2_{\mathrm{F}} \right\} \\
& = \left( s + ( K+1 ) h \right) \sum^K_{k=0} \alpha_k \left( \left\Vert \left[ \Omega \right]_{\supportshared}  \right\Vert^2_{\mathrm{F}} + \left\Vert \left[ \uniqueparameter \right]_{ \supportunique }  \right\Vert^2_{\mathrm{F}} \right).
\end{align*}
By the assumption that the supports of $\shared$ and $\unique$ are disjoint and the fact that $\Omega=\left[ \Omega \right]_{\supportshared}$ and $\uniqueparameter=\left[ \uniqueparameter \right]_{ \supportunique }$, we then have
\begin{align*}
\Phi^2 ( \Theta ) & \leq 2 \left( s +  ( K+1 ) h \right) \sum^K_{k=0} \alpha_k \left\Vert  \left[ \Omega \right]_{\supportshared} + \left[ \uniqueparameter \right]_{ \supportunique }  \right\Vert^2_{\mathrm{F}} \\
& \leq 2 \left( s +  ( K+1 ) h \right) \sum^K_{k=0} \alpha_k \left\Vert  \Omega  + \uniqueparameter  \right\Vert^2_{\mathrm{F}} \\
& = 2 \left( s +  ( K+1 ) h \right) H \left( \Theta \right).
\end{align*}

\subsubsection{Proof of Lemma~\ref{lemma:rsc-H-new}}
\label{sec:proof-lemma-rsc-H-new}

By Lemma~\ref{lemma:likelihood-taylor-series}, we have
\begin{multline}
\label{eq:residual-lw-bd-step1-new}
\mathcal{R}\left( \Delta_\Theta \right) 
= \sum^K_{k=0} \frac{\alpha_k}{2} \text{vec} \left( \Delta_\Omega + \Delta_{\uniqueparameter} \right)  \\ 
\times \left\{ \left( \Omega^{(k)} + t_k \left( \Delta_\Omega + \Delta_{\uniqueparameter} \right) \right)^{-1} \otimes \left( \Omega^{(k)} + t_k \left( \Delta_\Omega + \Delta_{\uniqueparameter} \right) \right)^{-1}  \right\} \\
\times \text{vec} \left( \Delta_\Omega + \Delta_{\uniqueparameter} \right),
\end{multline}
where $t_k \in (0,1)$, $0 \leq k \leq K$.
Since $\gamma_{\min} ( A^{-1} \otimes A^{-1} )= \Vert A \Vert^{-2}_2$ for any $A \succ 0$, we have
\begin{align*}
\gamma_{\min} & \left( \left( \Omega^{(k)} + t_k \left( \Delta_\Omega + \Delta_{\uniqueparameter} \right) \right)^{-1} \otimes \left( \Omega^{(k)} + t_k \left( \Delta_\Omega + \Delta_{\uniqueparameter} \right) \right)^{-1}  \right) \\
& = \left\{ \left\Vert \Omega^{(k)} + t_k \left( \Delta_\Omega + \Delta_{\uniqueparameter} \right) \right\Vert_2  \right\}^{-2} \\
& \geq \left\{ \left\Vert \Omega^{(k)} \right\Vert_2 + t_k \left\Vert  \Delta_\Omega + \Delta_{\uniqueparameter} \right\Vert_2  \right\}^{-2} \\
& \geq \left\{ \left\Vert \Omega^{(k)} \right\Vert_2 +  \left\Vert  \Delta_\Omega + \Delta_{\uniqueparameter} \right\Vert_2  \right\}^{-2} \\
& \geq \left( 2 M_{\Omega} + M_{\mathrm{op}}   \right)^{-2},
\end{align*}
where the last line follows the definition of $ M_{\Omega}$ in~\eqref{eq:def-Msigma-Momega} and the assumption that $\left\Vert  \Delta_\Omega + \Delta_{\uniqueparameter} \right\Vert_2 \leq  M_{\Omega} +  M_{\mathrm{op}}$ for all $0 \leq k \leq K$. 
Let $\kappa = \left( 2 M_{\Omega} + M_{\mathrm{op}}  \right)^{-2}$.
Then
\begin{align*}
\mathcal{R}\left( \Delta_\Theta \right) & \geq \frac{\kappa}{2} \sum^K_{k=0} \alpha_k \left\Vert \text{vec} \left( \Delta_\Omega + \Delta_{\uniqueparameter} \right) \right\Vert^2_2.
\end{align*}
The final result follows by noting that $\Vert \text{vec} ( \Delta_\Omega + \Delta_{\uniqueparameter} ) \Vert_2=\Vert \Delta_\Omega + \Delta_{\uniqueparameter} \Vert_{\mathrm{F}}$.

\subsubsection{Proof of Lemma~\ref{lemma:good-event-hold}}

We first state and prove the following lemma that gives the closed-form expression of $\Phi^{\star}(\cdot)$. 
\begin{lemma}
\label{lemma:dual-norm}
For the dual norm defined in~\eqref{eq:dual-norm}, we have
\begin{equation*}
\Phi^{\star} \left( \Theta \right) = \max \left\{ \left\vert \Omega \right\vert_{\infty} ,  \max_{ 0 \leq k \leq K } \frac{\left\vert \uniqueparameter \right\vert_{\infty}}{ \sqrt{\alpha_k} } \right\}.
\end{equation*}
\end{lemma}
\begin{proof}
For any $\Theta,\Theta^{\prime} \in \mathbb{H}$, we have
\begin{align*}
\left\langle \Theta, \Theta^{\prime} \right\rangle &= \left\langle \Omega, \Omega^{\prime} \right\rangle + \sum^K_{k=0} \left\langle \uniqueparameter, \Gamma^{(k)\prime} \right\rangle \\
& \leq \left\vert \Omega \right\vert_{\infty} \left\vert \Omega^{\prime} \right\vert_1 + \sum^K_{k=0} \left\vert \uniqueparameter \right\vert_{\infty} \left\vert \Gamma^{\prime} \right\vert_1 \\
& = \left\vert \Omega \right\vert_{\infty} \left\vert \Omega^{\prime} \right\vert_1 + \sum^K_{k=0} \sqrt{\alpha_k} \, \left\vert \Gamma^{\prime} \right\vert_1 \frac{ \vert \uniqueparameter \vert_{\infty} }{  \sqrt{\alpha_k} } \\
& \leq \left\vert \Omega \right\vert_{\infty} \left\vert \Omega^{\prime} \right\vert_1 + \left\{  \max_{ 0 \leq k \leq K} \frac{ \left\vert \uniqueparameter \right\vert_{\infty} }{  \sqrt{\alpha_k} } \right\}  \sum^K_{k=0} \sqrt{\alpha_k} \, \left\vert \Gamma^{\prime} \right\vert_1 \\
& \leq \max \left\{ \left\vert \Omega \right\vert_{\infty},   \max_{0 \leq k \leq K} \frac{ \left\vert \uniqueparameter \right\vert_{\infty} }{  \sqrt{\alpha_k} } \right\} \left( \left\vert \Omega^{\prime} \right\vert_1 +  \sum^K_{k=0} \sqrt{\alpha_k} \, \left\vert \Gamma^{\prime} \right\vert_1 \right) \\
& = \max \left\{ \left\vert \Omega \right\vert_{\infty},  \max_{0 \leq k \leq K} \frac{ \left\vert \uniqueparameter \right\vert_{\infty} }{  \sqrt{\alpha_k} } \right\} \Phi \left( \Theta^{\prime} \right).
\end{align*}
Thus, we have
\begin{equation*}
\Phi^{\star} \left( \Theta \right) \coloneqq \sup_{\Theta^{\prime}:\Phi(\Theta^{\prime}) \leq 1} \left\langle \Theta, \Theta^{\prime} \right\rangle \leq \max \left\{ \left\vert \Omega \right\vert_{\infty}, \,  \max_{0 \leq k \leq K} \frac{ \left\vert \uniqueparameter \right\vert_{\infty} }{  \sqrt{\alpha_k} } \right\}.
\end{equation*}
Finally, it is easy to see that the equality is achievable.
\end{proof}

Now we are ready to prove Lemma~\ref{lemma:good-event-hold}. Note that
\begin{align*}
\nabla_{\Omega} \mathcal{L} \left( \Theta \right) &= \sum^K_{k=0} \alpha_k \left( \widehat{\Sigma}^{(k)} - \left( \Omega + \uniqueparameter \right)^{-1} \right), \\
\nabla_{\uniqueparameter} \mathcal{L} \left( \Theta \right) &= \alpha_k \left( \widehat{\Sigma}^{(k)} - \left( \Omega + \uniqueparameter \right)^{-1} \right) \quad \text{for all } 0 \leq k \leq K.
\end{align*}
Then by Lemma~\ref{lemma:dual-norm}, we have
\begin{equation}
\label{eq:good-event-closed-form}
\mathbb{G}(\lambda_{\text{M}}) = \left\{ \frac{\lambda_{\text{M}}}{2} \geq \max \left\{ \left\vert \sum^K_{k=0} \alpha_k \left( \widehat{\Sigma}^{(k)} - \Sigma^{(k)} \right) \right\vert_{\infty}, \,  \max_{0 \leq k \leq K}  \sqrt{\alpha_k} \left\vert   \widehat{\Sigma}^{(k)} - \Sigma^{(k)}  \right\vert_{\infty} \right\}  \right\}.
\end{equation}

By Lemma~\ref{lemma:avg-emp-cov-infty} and the union bound, we have
\begin{align*}
\left\vert \sum^K_{k=0} \alpha_k \left( \widehat{\Sigma}^{(k)} - \Sigma^{(k)} \right) \right\vert_{\infty} \leq 80 M_\Sigma \max \left\{  \sqrt{ \frac{ \log (2 \left( K + 2 \right) d^2 / \delta) }{ 2 N } } , \, \frac{ \log (2 \left( K + 2 \right) d^2 / \delta) }{ N } \right\}
\end{align*}
and
\begin{equation*}
\left\vert  \widehat{\Sigma}^{(k)} - \Sigma^{(k)} \right\vert_{\infty} \leq 80 M_\Sigma \max \left\{  \sqrt{ \frac{ \log (2 \left( K + 2 \right) d^2 / \delta) }{ 2 n_k } } , \, \frac{ \log (2 \left( K + 2 \right) d^2 / \delta) }{ n_k } \right\}
\end{equation*}
for all $0 \leq k \leq K$ hold simultaneously with probability at least $1-\delta$. 
When $\min_{0 \leq k \leq K} n_k$ is large enough such that
\begin{equation*}
\frac{ \log (2 \left( K + 2 \right) d^2 / \delta) }{ \min_{0 \leq k \leq K} n_k } \leq \frac{1}{2},
\end{equation*}
we then have that
\begin{align*}
& \left\vert \sum^K_{k=0} \alpha_k \left( \widehat{\Sigma}^{(k)} - \Sigma^{(k)} \right) \right\vert_{\infty} \leq 80 M_\Sigma \sqrt{ \frac{ \log (2 \left( K + 2 \right) d^2 / \delta) }{ 2 N } }, \\
& \left\vert  \widehat{\Sigma}^{(k)} - \Sigma^{(k)} \right\vert_{\infty} \leq 80 M_\Sigma  \sqrt{ \frac{ \log (2 \left( K + 2 \right) d^2 / \delta) }{ 2 n_k } } \quad \text{for all } 0 \leq k \leq K . \numberthis \label{eq:proof-thm-frob-8}
\end{align*}
hold simultaneously with probability at least $1-\delta$.
Note that~\eqref{eq:proof-thm-frob-8} implies that
\begin{align*}
\sqrt{\alpha_k} \left\vert  \widehat{\Sigma}^{(k)} - \Sigma^{(k)} \right\vert_{\infty} \leq 80 M_\Sigma \sqrt{ \frac{ \log (2 \left( K + 2 \right) d^2 / \delta) }{ 2 N } } \quad \text{for all } 0 \leq k \leq K .
\end{align*}
Thus, when
\begin{equation}
\label{eq:proof-thm-frob-9}
\lambda_{\text{M}} \geq 160 M_\Sigma \sqrt{ \frac{ \log (2 \left( K + 2 \right) d^2 / \delta) }{ 2 N } },
\end{equation}
by~\eqref{eq:good-event-closed-form}, we have $\mathbb{G}(\lambda_{\text{M}})$ hold with probability $1-\delta$.


\section{Corollary~\ref{corollary:analysis-Dtrace} and Its Proof}
\label{sec:proof-corollary:analysis-Dtrace}

The following corollary of Theorem~10 in~\citet{zhao2022fudge} provides a high-probability error bound for the D-Trace Loss estimator.
\begin{corollary}
\label{corollary:analysis-Dtrace}
Suppose that Assumption~\ref{assump:model-structure} and Assumption~\ref{assump:eig-element-upp-bd} hold.
Let $C_{\gamma}=M^{-1}_{\Omega}$.
Furthermore, suppose that for a given $\delta \in (0,1]$ and all $ k \in [K]$ we have 
\begin{equation}
\label{eq:corollary:analysis-Dtrace-requir-1}
\frac{ \log \left( 2 (K+1) d^2 / \delta \right)  }{  \min \{ n_k, n_0 \} } \leq \min \left\{ 1, \frac{ \gamma^2_{\min} (\Sigma^{(k)}) \gamma^2_{\min} (\Sigma^{(0)})  }{ C^{\prime} h^2 
M_\Sigma^4
} \right\}.
\end{equation}
If 
\begin{equation*}
\lambda^{(k)}_{\Gamma} \geq 2 C_{\gamma} \left[ (C_{\gamma}+2)  
M_\Sigma
M_{\Gamma} + 2 \right] 
M_\Sigma
\sqrt{ \frac{ \log \left( 2 (K+1)  d^2 / \delta \right)  }{ 2 \min \{ n_k, n_0 \} } }
\end{equation*}
for all $k \in [K]$, then 
\begin{align*}
\left\Vert \diffestimate - \diff \right\Vert_{\mathrm{F}} \leq \frac{C^{\prime\prime} \sqrt{h} \lambda^{(k)}_{\Gamma} }{ \gamma_{\min} (\Sigma^{(k)}) \gamma_{\min} (\Sigma^{(0)}) }  \quad \text{for all } 0 \leq k \leq K
\end{align*}
hold simultaneously with probability at least $1-\delta$, where $C^{\prime},C^{\prime\prime}$ are universal constants that depend on $C_{\gamma}$.
\end{corollary}

\begin{proof}
The D-Trace Loss estimator is a special case of the FuDGE estimator proposed in~\citet{zhao2019direct,zhao2022fudge}. We adapt the proof of Theorem~10 in~\citet{zhao2022fudge} here.

Given $\delta \in (0,1]$, we define the event
\begin{multline*}
\mathcal{E}^1 (\delta) \coloneqq \Bigg\{  \left\vert \widehat{\Sigma}^{(k)} - \Sigma^{(k)} \right\vert_{\infty} 
\leq C_{\gamma} M_\Sigma \cdot
\max \Bigg\{  \sqrt{ \frac{ \log \left( 2 ( K+1 )  d^2 / \delta \right)  }{ 2 n_k } } , \, \frac{ \log \left( 2 ( K+1 ) d^2 / \delta \right)  }{  n_k }  \Big\} \\
\text{for all } 0 \leq k \leq K \Bigg\},
\end{multline*}
where $C_{\gamma}$ is a universal constant.
By Lemma~\ref{lemma:infty-cov-err} and the union bound, we have $\mathbb{P}(\mathcal{E}^1(\delta)) \geq 1 - \delta$. 
In the following, we work on the event $\mathcal{E}^1(\delta)$.

Let
\begin{equation*}
\psi_k = C_{\gamma} M_{\Sigma}  \sqrt{ \frac{ \log \left( 2 ( K+1 )  d^2 / \delta \right)  }{ 2 \min \{ n_k, n_0 \} } } \,.
\end{equation*}
On the event $\mathcal{E}^1$ and when the condition~\eqref{eq:corollary:analysis-Dtrace-requir-1} holds with $C^{\prime}=4096 C^2_{\gamma} (C_{\gamma} + 2)^2$, we have
\begin{equation*}
\max \left\{ \left\vert \widehat{\Sigma}^{(k)} - \Sigma^{(k)} \right\vert_{\infty}, \left\vert \widehat{\Sigma}^{(0)} - \Sigma^{(0)} \right\vert_{\infty}  \right\} \leq \psi_k.
\end{equation*}
When~\eqref{eq:corollary:analysis-Dtrace-requir-1} holds, we have 
\begin{equation*}
\psi_k \leq C_{\gamma} \cdot \max_{0 \leq k \leq K} \left\vert \Sigma^{(k)} \right\vert_{\infty}
\quad\text{and}\quad
\psi_k \leq \frac{ \gamma_{\min} (\Sigma^{(k)}) \gamma_{\min} (\Sigma^{(0)})  }{  64 (C_{\gamma}+2) h M_\Sigma }.
\end{equation*}
Thus, we have 
\begin{equation*}
\kappa^{(k)}_{\mathcal{L}} \coloneqq \frac{1}{2} \gamma_{\min} (\Sigma^{(k)}) \gamma_{\min} (\Sigma^{(0)}) - 16 h \left( \psi^2_k + 2 \psi_k M_\Sigma \right) \geq \frac{1}{4} \gamma_{\min} (\Sigma^{(k)}) \gamma_{\min} (\Sigma^{(0)}),
\end{equation*}
where $\kappa^{(k)}_{\mathcal{L}}$ is the restricted convexity parameter defined in the proof of Theorem~10 of~\cite{zhao2022fudge}.
Set 
\begin{equation*}
\lambda^{(k)}_{\Gamma} \geq 4 \left[ (C_{\gamma}+2)   M_\Sigma  M_{\Gamma} + 1 \right] \psi_k\,.
\end{equation*}
Following the proof of Theorem~10 of~\cite{zhao2022fudge}, we have
\begin{align*}
\left\Vert \diffestimate - \diff \right\Vert^2_{\mathrm{F}} \leq \frac{18 h}{ \frac{1}{16} \gamma^2_{\min} (\Sigma^{(k)}) \gamma^2_{\min} (\Sigma^{(0)}) } \left( \lambda^{(k)}_{\Gamma} \right)^2,
\end{align*}
which completes the proof.
\end{proof}

\section{Proof of Lemma~\ref{lemma:low-bd-shared}}
\label{sec:proof-thm-low-bd-shared}

Let $R=\lfloor s / d \rfloor$ and let $r=s - R d$.
We follow the construction of the hard-case collection of precision matrices in~Step~2 of Proof of Theorem~4.1 in \citet{cai2016estimating} to get $\mathcal{F}_{\star}$. 
We set $M_{n,p}$ and $c_{n,p}$ therein as $M_{n,p}=\text{some constant}$ and $c_{n,p}=R$.
By (4.13) and (4.15) in~\cite{cai2016estimating}, for any $\Omega \in \mathcal{F}_{\star}$, we have
\begin{equation*}
0 < c_1 \leq  \gamma_{\min} ( \Omega ) \leq \gamma_{\max} ( \Omega ) \leq c_2 < \infty \, .
\end{equation*}
Furthermore, we have $\vert \Omega \vert_0 \leq c_{n,p} d \leq s$. 
Thus, $\mathcal{F}_{\star} \subseteq \mathcal{G}_1$.
Following Theorem 6.1 of~\cite{cai2016estimating}, we have
\begin{align*}
\inf_{\widehat{\Omega}} \sup_{\Omega \in \mathcal{G}_1} \mathbb{E} \left[ \left\Vert \widehat{\Omega} - \Omega \right\Vert^2_{\mathrm{F}} \right] & \geq \inf_{\widehat{\Omega}} \sup_{\Omega \in \mathcal{F}_{\star} } \mathbb{E} \left[ \left\Vert \widehat{\Omega} - \Omega \right\Vert^2_{\mathrm{F}} \right] \gtrsim d c_{n,p} \frac{\log d}{n} = \frac{ d R }{ s } \cdot \frac{s \log d}{n} \geq \frac{1}{2} \frac{s \log d}{n} \, ,
\end{align*}
which completes the proof.

\section{Proof of Lemma~\ref{lemma:lower-bd-diff-net}}
\label{sec:proof-thm-lower-bd-diff-net}

Without loss of generality, we assume that $d$ and $h$ are even positive numbers.
Let $m={d}/{2}$.
By Lemma~\ref{lemma:cov-number}, there exist $\{ b^0, b^1, \ldots, b^M \} \subseteq \mathbb{R}^{m^2}$ such that
\begin{enumerate}[label=(\roman*)]
\item $\Vert b^j \Vert_0 \leq \frac{h}{2}$ and $\Vert b^j \Vert_2 \leq 1$ for all $0 \leq j \leq M$;
\item $b^0=0$;
\item $\Vert b^j - b^k \Vert_2 \geq \frac{1}{4}$ for all $0 \leq j \neq k \leq M$;
\item $\log (M+1) \geq \frac{h}{4} \log \left( \frac{d^2/4-h/2}{h/2 } \right)$.
\end{enumerate}

For $j=0,\ldots,M$, let $B^j \in \mathbb{R}^{m \times m}$ be such that $\text{vec}(B^j)=\delta \cdot b^j$, where $\delta$ is a positive number that depends on $n,d,h$ and will be specified later. 
By~\eqref{eq:lower-bd-diff-net-cond}, $d \geq 4h$ and 
\begin{equation*}
\log (M+1) \geq \frac{h}{4} \log \left( \frac{d^2/4-h/2}{h/2 } \right) = \frac{h}{4} \log \left( \frac{d^2}{2 h} - 1 \right) \geq \frac{h}{4} \log \left( \frac{d^2}{4 h} \right) \geq \frac{h}{4} \log d.
\end{equation*}
Furthermore, we have
\begin{gather*}
\label{eq:proof-thm-lower-bd-diff-net-1}
\left\vert B^j \right\vert_0 \leq \frac{h}{2}, 
\quad
\left\Vert B^j \right\Vert_{\mathrm{F}} \leq \delta,
\quad
\text{for } 0 \leq j \leq M, \\
\quad
B^0 = 0, 
\quad
\left\Vert B^j - B^k \right\Vert_{\mathrm{F}} \geq \frac{\delta}{4},
\quad 
\text{for } 0 \leq j \neq k \leq M.
\end{gather*}

For $j = 0, \ldots, M$, let
\begin{equation*}
\Omega^j = 
\begin{bmatrix}
I_m & B^j \\
( B^j )^{\top} & I_m
\end{bmatrix}.
\end{equation*}
We next verify that $\Omega^j \in \mathcal{G}_2$ when
\begin{equation}
\label{eq:proof-thm-lower-bd-diff-net-6}
\delta \leq \min \left\{ \frac{C_{\Gamma}}{\sqrt{2 h}}, \frac{1 - c_1}{2}, \frac{c_2 - 1}{2} \right\},
\end{equation}
where $\mathcal{G}_2$ is defined in~\eqref{eq:low-bd-param-space-2}. 
First note that $\text{diag}(\Omega^j)=I_d$, $\left\vert \Omega^j - I_d \right\vert_0 = 2 \left\vert B^j \right\vert_0 \leq h$, and
\begin{equation*}
\left\vert \Omega^j - I_d \right\vert_1 = 2 \left\vert B^j \right\vert_1 \leq 2 \sqrt{ \frac{h}{2} } \left\Vert B^j \right\Vert_{\mathrm{F}} \leq \sqrt{2 h} \delta \leq C_{\Gamma},
\end{equation*}
using the choice of $\delta$ in~\eqref{eq:proof-thm-lower-bd-diff-net-6}.
Furthermore, we have
\begin{align*}
\gamma_{\min} \left( \Omega^j \right) & = \gamma_{\min} \left(  
\begin{bmatrix}
I_m & 0 \\
0 & I_m
\end{bmatrix}
+
\begin{bmatrix}
0 & B^j \\
(B^j)^{\top} & 0
\end{bmatrix}
\right) \\
& \geq 1 -  \left\Vert 
\begin{bmatrix}
0 & B^j \\
(B^j)^{\top} & 0
\end{bmatrix}
\right\Vert_2 \\
& \geq  1 -  \left\Vert 
\begin{bmatrix}
0 & B^j \\
(B^j)^{\top} & 0
\end{bmatrix}
\right\Vert_{\mathrm{F}} \\
& = 1 - 2 \left\Vert B^j \right\Vert_{\mathrm{F}} \\
& \geq c_1,
\end{align*}
again using the choice of $\delta$ in~\eqref{eq:proof-thm-lower-bd-diff-net-6}. Similarly, we have $\gamma_{\max} \left( \Omega^j \right) \leq c_2$.
Therefore, when \eqref{eq:proof-thm-lower-bd-diff-net-6} holds, we have that $\Omega^j \in \mathcal{G}_2$.

Let $\mathbb{P}_j$ denote the probability measure of $N \left( 0, \{ \Omega^j \}^{-1}  \right)$ and let $\mathbb{P}^n_j$ denote the product probability measure of $(X_1,\ldots,X_n)$ where $X_1,\ldots,X_n \overset{i.i.d.}{\sim} \mathbb{P}_j$. Note that $\mathbb{P}_0=N(0,I_d)$.
When $\delta \leq {1}/{4}$, Lemma~\ref{lemma:KL-divergence} gives us
\begin{equation*}
D_{\text{KL}} \left( \mathbb{P}^n_j \, \Vert \, \mathbb{P}^n_0  \right) = n D_{\text{KL}} \left( \mathbb{P}_j \, \Vert \, \mathbb{P}_0  \right) \leq \frac{16}{15} n \left\Vert B^j \right\Vert^2_{\mathrm{F}} \leq \frac{16}{15} n \delta^2,
\end{equation*}
which implies that
\begin{equation}
\label{eq:proof-thm-lower-bd-diff-net-2}
\frac{1}{M} \sum^M_{j=1} D_{\text{KL}} \left( \mathbb{P}^n_j \, \Vert \, \mathbb{P}^n_0  \right) \leq \frac{16}{15} n \delta^2.
\end{equation}
From condition~\eqref{eq:lower-bd-diff-net-cond}, we have $h \log d \geq 8 \log 3$ and thus $\log (M+1) \geq \log 9$, which implies that $M \geq 8$. Furthermore, we have 
\begin{equation*}
\log M = \frac{ \log M}{ \log (M+1) } \log (M+1) \geq \frac{ \log 8 }{ \log 9 } \log (M+1) \geq \frac{4}{5} \log (M+1) \geq \frac{h}{5} \log d.
\end{equation*}

We set
\begin{equation*}
\delta = \frac{1}{4 \sqrt{2}} \sqrt{  \frac{h \log d}{n} }.
\end{equation*}
By~\eqref{eq:lower-bd-diff-net-cond}, we have that~\eqref{eq:proof-thm-lower-bd-diff-net-6} holds and that $\delta \leq {1}/{4}$. Thus~\eqref{eq:proof-thm-lower-bd-diff-net-2} holds, implying 
\begin{equation}
\label{eq:proof-thm-lower-bd-diff-net-3}
\frac{1}{M} \sum^M_{j=1} D_{\text{KL}} \left( \mathbb{P}^n_j \, \Vert \, \mathbb{P}^n_0  \right) \leq \frac{1}{30} h \log d \leq \frac{1}{6} \log M.
\end{equation}
In addition, we have
\begin{equation}
\label{eq:proof-thm-lower-bd-diff-net-4}
\left\Vert B^j - B^k \right\Vert_{\mathrm{F}} \geq 2 \cdot \frac{1}{32\sqrt{2}} \sqrt{  \frac{h \log d}{n} } \quad \text{for all } 0 \leq j \neq k \leq M.
\end{equation}
Note that
\begin{equation}
\label{eq:proof-thm-lower-bd-diff-net-5}
\frac{\log (M+1) - \log 2}{\log M } - \frac{1}{6} \geq \frac{\log M - \log 2}{\log M } - \frac{1}{6} \geq \frac{\log 8 - \log 2}{ \log 8} - \frac{1}{6} = \frac{1}{2}.
\end{equation}
Following Section 2.2 and Corollary 2.6 in \citet{tsybakov2008introduction}, combined with~\eqref{eq:proof-thm-lower-bd-diff-net-3}--\eqref{eq:proof-thm-lower-bd-diff-net-5}, we have
\begin{equation*}
\min_{\widehat{\Omega}} \sup_{\Omega \in \mathcal{G}_2} \mathbb{E} \left[ \left\Vert  \widehat{\Omega} - \Omega \right\Vert^2_{\mathrm{F}} \right] \geq \frac{1}{2} \left( \frac{1}{32\sqrt{2}} \sqrt{  \frac{h \log d}{n} } \right)^2 = \frac{1}{4096} \cdot \frac{h \log d}{n},
\end{equation*}
which completes the proof.

\section{Proof of Theorem~\ref{thm:minimax-optimal-rate}}
\label{sec:proof-thm-minimax-optimal-rate}

The upper bound follows directly from Corollary~\ref{corollary:convg-Trans-GLasso-Frob-Dtrace}. To prove the lower bound, we assume that $\mathcal{P}_k=N \left( 0, \left( \Omega^{(k)} \right)^{-1} \right)$.

First, when $h=0$, we have $N$ i.i.d.~samples from $\Omega^{(0)}$ where
\begin{equation*}
\Omega^{(0)} \in \left\{ \Omega \in \mathbb{R}^{d \times d}: \Omega \succ 0, \, \left\vert \Omega \right\vert_0 \leq s , \, \left\Vert \Omega \right\Vert_2 \leq M_{\Omega}, \, \left\vert \Omega^{-1} \right\vert_{\infty} \leq M_{\Sigma} \right\} \coloneqq \tilde{\mathcal{G}}^{\prime}
\end{equation*}
for some positive universal constants $M_{\Omega}$ and $M_{\Sigma}$. 
Let
\begin{equation*}
\mathcal{G}^{\prime} = \left\{ \Omega \in \mathbb{R}^{d \times d}: \Omega \succ 0, \, \left\vert \Omega \right\vert_0 \leq s , \, 0 < \frac{1}{M_{\Sigma}} \leq \gamma_{\min} \left( \Omega \right) \leq \gamma_{\max} \left( \Omega \right) \leq M_{\Omega} < \infty  \right\}.
\end{equation*}
Since $\left\vert \Omega^{-1} \right\vert_{\infty} \leq \Vert \Omega^{-1} \Vert_2 = \{ \gamma_{\min}(\Omega) \}^{-1} \leq M_{\Sigma}$ for any $\Omega \in \mathcal{G}^{\prime}$, we have $\mathcal{G}^{\prime} \subseteq \tilde{\mathcal{G}}^{\prime}$.
Thus, when $s \geq d \geq c^{\prime} N^{\beta}$, for some universal constants $c^{\prime}>0$ and $\beta>1$ and
\begin{equation*}
\lceil s / d \rceil = o \left( \frac{ N }{ \left( \log d \right)^{\frac{3}{2}} } \right)
\end{equation*}
it follows by Lemma~\ref{lemma:low-bd-shared} that
\begin{equation}
\label{eq:lw-bd-proof-part-1}
\inf_{\widehat{\Omega}} \sup_{\Omega^{(0)} \in \tilde{\mathcal{G}}^{\prime} } \mathbb{E} \left[ \left\Vert \widehat{\Omega} - \Omega^{(0)} \right\Vert^2_{\mathrm{F}} \right] \geq \inf_{\widehat{\Omega}} \sup_{\Omega^{(0)} \in \mathcal{G}^{\prime} } \mathbb{E} \left[ \left\Vert \widehat{\Omega} - \Omega^{(0)} \right\Vert^2_{\mathrm{F}} \right] \gtrsim \frac{ s \log d }{ N } \, .
\end{equation}

If $\Omega^{(k)}=\shared=I_d$ and $\unique=0$ for every $k \in [K]$, then samples from source distributions cannot be used to estimate $\Omega^{(0)}$. 
Therefore, we must depend solely on samples from the target distribution to estimate $\Omega^{(0)}$. 
Note that now we have
\begin{equation*}
\begin{aligned}
\Omega^{(0)} \in & \left\{ \Omega \in \mathbb{S}^{d \times d} : \, \Omega \succ 0 , \, \Omega = I_d + \Delta, \, \Delta_{jj} = 0 \text{ for all } 1 \leq j \leq d, \right. \\
& \quad\quad\quad\quad \left. \left\vert \Delta \right\vert_0 \leq h, \,  \left\vert \Delta \right\vert_1 \leq M_{\Gamma}, \, \left\Vert \Omega \right\Vert_2 \leq M_{\Omega}, \, \left\vert \Omega^{-1} \right\vert_{\infty} \leq M_{\Sigma} \right\} \coloneqq \tilde{\mathcal{G}}^{\prime\prime} \, .
\end{aligned}
\end{equation*}
Let
\begin{equation*}
\begin{aligned}
\mathcal{G}^{\prime\prime} &= \left\{ \Omega \in \mathbb{S}^{d \times d} : \, \Omega \succ 0 , \, \Omega = I_d + \Delta, \, \Delta_{jj} = 0 \text{ for all } 1 \leq j \leq d, \right. \\
& \quad\quad\quad\quad \left. \left\vert \Delta \right\vert_0 \leq h, \,  \left\vert \Delta \right\vert_1 \leq M_{\Gamma}, \, 0 < \frac{1}{M_{\Sigma}} \leq \gamma_{\min} (\Omega) \leq \gamma_{\max} (\Omega) \leq M_{\Omega} < \infty \right\} \, .
\end{aligned}
\end{equation*}
Given that for any $\Omega \in \mathcal{G}^{\prime\prime}$, it holds that $\left\vert \Omega^{-1} \right\vert_{\infty} \leq \Vert \Omega^{-1} \Vert_2 = \{ \gamma_{\min}(\Omega) \}^{-1} \leq M_{\Sigma}$, we can deduce that $\mathcal{G}^{\prime\prime} \subseteq \tilde{\mathcal{G}}^{\prime\prime}$.
By Lemma~\ref{lemma:lower-bd-diff-net}, when $M_{\Omega},M_{\Sigma}>1$ and
\begin{equation*}
d \geq 4h, \, h \log d \geq 8 \log 3, \, \frac{h \log d}{n_0} \leq \min \left\{ 2, 8 \left( 1 - \frac{1}{M_{\Sigma}} \right)^2, 8 \left( 1 - M_{\Omega} \right)^2 \right\}, \, h \sqrt{\frac{\log d}{n_0}} \leq 4 M_{\Gamma},
\end{equation*}
we have
\begin{equation}
\label{eq:lw-bd-proof-part-2}
\inf_{\widehat{\Omega}} \sup_{\Omega^{(0)} \in \tilde{\mathcal{G}}^{\prime\prime} } \mathbb{E} \left[ \left\Vert \widehat{\Omega} - \Omega^{(0)} \right\Vert^2_{\mathrm{F}} \right] \geq \inf_{\widehat{\Omega}} \sup_{\Omega^{(0)} \in \mathcal{G}^{\prime\prime} } \mathbb{E} \left[ \left\Vert \widehat{\Omega} - \Omega^{(0)} \right\Vert^2_{\mathrm{F}} \right] \gtrsim \frac{ h \log d }{ n_0 } \, .
\end{equation}

The final result follows from~\eqref{eq:lw-bd-proof-part-1} and~\eqref{eq:lw-bd-proof-part-2}.

\section{Additional Optimization Details}
\label{sec:opt-addition}

In this section, we provide more details about the numerical algorithms introduced in Section~\ref{sec:Optimization}.

We first discuss how to compute the updating steps~\eqref{eq:ADMM-update-1} and~~\eqref{eq:ADMM-update-2}.
Note that~\eqref{eq:ADMM-update-1} is equivalent to
\begin{align*}
\Omega^{(t)}_k &= \arg \min_{\Omega_k \succ 0} \left\{  f_k \left( \Omega_k \right) + \rho \left\langle Z^{(t-1)}_k , \,  \Omega_k - \Omega^{(t-1)} - \Gamma^{(t-1)}_k\right\rangle + \frac{\rho}{2} \left\Vert \Omega_k - \Omega^{(t-1)} - \Gamma^{(t-1)}_k \right\Vert^2_{\mathrm{F}} \right\} \\
&= \arg \min_{\Omega_k \succ 0} \left\{  - \alpha_k \log \det \left( \Omega_k \right) +  \frac{\rho}{2} \left\Vert \Omega_k - \tilde{C}^{(t-1)}_k  \right\Vert^2_{\mathrm{F}}  \right\}, \numberthis \label{eq:ADMM-update-1-1}
\end{align*}
where $
\tilde{C}^{(t-1)}_k = - Z^{(t-1)}_k + \left( \Omega^{(t-1)} + \Gamma^{(t-1)}_k \right) - ( \alpha_k / \rho) \widehat{\Sigma}^{(k)}$ for $0 \leq k \leq K$. Taking the gradient with respect to $\Omega_k$ in~\eqref{eq:ADMM-update-1-1} and setting it to zero gives
\begin{equation}
\label{eq:ADMM-update-1-2}
- \alpha_k \Omega^{-1}_k + \rho \Omega_k - \rho \tilde{C}^{(t-1)}_k = 0.
\end{equation}
The matrix $\Omega^{(t)}_k$ is obtained by finding $\Omega_k \succ 0$ that satisfies~\eqref{eq:ADMM-update-1-2}.
Let $\rho \tilde{C}^{(t-1)}_k = U \Lambda U^{\top}$, $\Lambda = \text{diag} \left( \left\{ \lambda_i \right\}^d_{i=1} \right)$, be the eigenvalue decomposition of $\rho \tilde{C}^{(t-1)}_k$. Following~\cite[Section 6.5]{boyd2011distributed}, we have
\begin{equation*}
\Omega^{(t)}_k = U \text{diag} \left( \left\{ \frac{ \lambda_i + \sqrt{  \lambda^2_i + 4 \rho \alpha_k }  }{ 2 \rho } \right\}^d_{i=1}  \right) U^{\top}.
\end{equation*}

On ther other hand, computing~\eqref{eq:ADMM-update-2} is equivalent to solving
\begin{align*}
\Omega^{(t)}, \left\{ \Gamma^{(t)}_k  \right\} & \in \arg \min_{\Omega , \{ \uniqueparameter \} } \left\{ \rho \sum^K_{k=0} \left\langle Z^{(t-1)}_k, \Omega^{(t)}_k - \Omega - \uniqueparameter \right\rangle   + \frac{\rho}{2} \sum^K_{k=0} \left\Vert \Omega^{(t)}_k - \Omega - \uniqueparameter \right\Vert^2_{\mathrm{F}} \right. \\
& = \arg \min_{\Omega , \{ \uniqueparameter \} } \left\{ \frac{\rho}{2} \sum^K_{k=0}  \left\Vert \Omega + \uniqueparameter - \check{C}^{(t)}_k \right\Vert^2_{\mathrm{F}} + \lambda_{\text{M}} \left\vert \Omega \right\vert_1 + \lambda_{\text{M}} \sum^K_{k=0} \sqrt{\alpha_k} \left\vert \uniqueparameter \right\vert_1 \right\}
\end{align*}
where $\check{C}^{(t)}_k = \Omega^{(t)}_k +  Z^{(t-1)}_k$. Given $c=(c_1,\ldots,c_K)$, let $\mathcal{S}(c)$ be defined as $\mathcal{S}(c) = ( x^{\star}, y^{\star} )$ where 
\begin{equation}
\label{eq:ADMM-update-2-sub}
\left( x^{\star}, y^{\star} \right) \in \arg \min_{(x, y)} \left\{ \frac{\rho}{2} \sum^K_{k=0} \left( x + y_k - c_k \right)^2 + \lambda_{\text{M}} \vert x \vert + \lambda_{\text{M}} \sum^K_{k=0} \sqrt{\alpha_k} \vert y_k \vert \right\}
\end{equation} 
with $y=(y_1,\ldots,y_K)$. With $\Omega^{(t)}=( \Omega^{(t)}_{jl} )_{1 \leq j,l \leq d}$, $\Gamma^{(t)}_k=( \Gamma^{(t)}_{k,jl} )_{1 \leq j,l \leq d}$, and $\check{C}^{(t)}_k=( \check{C}^{(t)}_{k,jl} )_{1 \leq j,l \leq d}$, we have
\begin{equation*}
\Omega^{(t)}_{jl}, \left\{ \Gamma^{(t)}_{k,jl}  \right\}^K_{k=1} = \mathcal{S} \left( \left\{ \check{C}^{(t)}_{k,jl} \right\}^K_{k=1} \right), \quad \text{for } 1 \leq j,l \leq d.
\end{equation*}
To solve~\eqref{eq:ADMM-update-2-sub}, we iteratively update $x$ or $y$ while fixing the other until convergence. For $c \in \mathbb{R}$ and $\lambda \geq 0$, let
\begin{equation}
\label{eq:soft-threshod-function}
\text{ST}_{\lambda}(c) =
\begin{cases}
c - \lambda & \text{if } c > \lambda, \\
0 & \text{if } |c| \leq \lambda, \\
c + \lambda & \text{if } c < -\lambda
\end{cases}
\end{equation}
be the soft-thresholding function. After initializing $x^{(0)},y^{(0)}$, we repeat the following process until convergence:
\begin{align*}
x^{(r)} &= \arg \min_x \left\{ \frac{\rho}{2} \sum^K_{k=0} \left( x + y^{(r-1)}_k - c_k \right)^2 + \lambda_{\text{M}} \vert x \vert + \lambda_{\text{M}} \sum^K_{k=0} \sqrt{\alpha_k} \left\vert y^{(r-1)}_k \right\vert \right\} \\
& = \text{ST}_{\lambda_{\text{M}} / (\rho (K+1)) } \left( \frac{1}{K+1} \sum^K_{k=0} \left(  c_k - y^{(r-1)}_k \right)  \right),
\end{align*}
and
\begin{align*}
y^{(r)} &= \arg \min_y \left\{ \frac{\rho}{2} \sum^K_{k=0} \left( x^{(r)} + y_k - c_k \right)^2 + \lambda_{\text{M}} \left\vert x^{(r)} \right\vert + \lambda_{\text{M}} \sum^K_{k=0} \sqrt{\alpha_k} \left\vert y_k \right\vert \right\} \\
\Leftrightarrow y^{(r)}_k & = \text{ST}_{ \lambda_{\text{M}} \sqrt{\alpha_k} / \rho} \left(  c_k - x^{(r)}  \right) \quad \text{for } 0 \leq k \leq K.
\end{align*}

We then discuss the stopping criterion for the ADMM algorithm to solve \myalgMT objective~\eqref{eq:Glasso-mutli-task}.

\paragraph{Stopping criterion for ADMM.} Following~\cite[Section 3.3.1]{boyd2011distributed}, let $\epsilon^{\text{abs}}>0$ be an absolute tolerance and $\epsilon^{\text{rel}}>0$ be a relative tolerance, we then define the feasibility tolerance for primal feasibility condition $\epsilon^{\text{pri}}>0$ and the feasibility tolerance for dual feasibility condition $\epsilon^{\text{dual}}>0$ at iteration $t$ as
\begin{align*}
\epsilon^{\text{pri}} &=  \epsilon^{\text{abs}} d \sqrt{K+1} +  \epsilon^{\text{rel}} \max \left\{ \left( \sum^K_{k=0} \left\Vert \Omega^{(t)}_k \right\Vert^2_{\mathrm{F}} \right)^{\frac{1}{2}} \, , \, \left( \sum^K_{k=0} \left\Vert \Omega^{(t)} + \Gamma^{(t)}_k
 \right\Vert^2_{\mathrm{F}} \right)^{\frac{1}{2}} \right\}, \\
\epsilon^{\text{dual}} &=  \epsilon^{\text{abs}} d \sqrt{K+1} + \epsilon^{\text{rel}} \left( \sum^K_{k=0} \left\Vert Z^{(t)}_k \right\Vert^2_{\mathrm{F}} \right)^{\frac{1}{2}} .
\end{align*}
Besides, let
\begin{align*}
r^{\text{pri}} &= \left( \sum^K_{k=0} \left\Vert
\Omega^{(t)}_k - \left( \Omega^{(t)} + \Gamma^{(t)}_k \right) \right\Vert^2_{\mathrm{F}}  \right)^{\frac{1}{2}} \quad \text{and} \\
r^{\text{dual}} &= \rho \left( \sum^K_{k=0} \left\Vert \left( \Omega^{(t)} + \Gamma^{(t)}_k \right) - \left( \Omega^{(t-1)} + \Gamma^{(t-1)}_k \right) \right\Vert^2_{\mathrm{F}}  \right)^{\frac{1}{2}}
\end{align*}
be the primal and dual residuals at iteration $t$. We then stop the iteration if
\begin{equation*}
r^{\text{pri}} \leq \epsilon^{\text{pri}} \quad \text{and} \quad r^{\text{dual}} \leq \epsilon^{\text{dual}}.
\end{equation*}

\paragraph{Stopping criterion for sub problem~\eqref{eq:ADMM-update-2-sub}.} The optimality conditions of problem~\eqref{eq:ADMM-update-2-sub} are
\begin{align}
& 0 \in \rho \sum^K_{k=0} \left( x^{\star} + y^{\star}_k - c_k \right) + \lambda_{\text{M}} \partial \vert x^{\star} \vert, \label{eq:opt-cond-1-ADMM-update-2-sub} \\
& 0 \in \rho \left( x^{\star} + y^{\star} - c_k \right) + \lambda_{\text{M}} \sqrt{\alpha_k} \partial \vert y^{\star}_k \vert \quad \text{for all } k=0,1,\ldots,K. \label{eq:opt-cond-2-ADMM-update-2-sub}
\end{align}
By the definition of $y^{(r)}$, we always have $x^{(r)},y^{(r)}$ satisfy~\eqref{eq:opt-cond-2-ADMM-update-2-sub}. Besides, by definition of $x^{(r)}$, we have
\begin{equation}
\label{eq:opt-cond-3-ADMM-update-2-sub}
0 \in \rho \sum^K_{k=0} \left( x^{(r)} + y^{(r-1)}_k - c_k \right) + \lambda_{\text{M}} \partial \vert x^{(r)} \vert.
\end{equation}
Thus, when
\begin{equation*}
\rho \sum^K_{k=0} \left( y^{(r)}_k - y^{(r-1)}_k \right) = 0,
\end{equation*}
we will have $x^{(r)},y^{(r)}$ satisfy~\eqref{eq:opt-cond-3-ADMM-update-2-sub}. 
Let $\epsilon^{\text{abs}}>0$ be an absolute tolerance and $\epsilon^{\text{rel}}>0$ be a relative tolerance, we then stop at iteration $r$ if
\begin{equation*}
r^{\text{sub}} = \rho \left\vert \sum^K_{k=0} \left( y^{(r)}_k - y^{(r-1)}_k \right) \right\vert \leq (K+1) \epsilon^{\text{abs}} + \epsilon^{\text{rel}} \max \left\{ \sum^K_{k=0} \left\vert y^{(r)}_k \right\vert \, , \, \sum^K_{k=0} \left\vert y^{(r-1)}_k \right\vert \right\} \coloneqq \epsilon^{\text{sub}} .
\end{equation*}

\begin{algorithm}[t]
\caption{ADMM for Trans-MT-GLasso}
\label{alg:admm-mt-glasso}
\begin{algorithmic}[1]
\STATE{\bfseries Input:} $\{\widehat{\Sigma}^{(k)}\}^K_{k=0}$, $\lambda_{\text{M}}$, $\rho$, $\epsilon^{\text{abs}}$, $\epsilon^{\text{rel}}$ and $\{\alpha_k\}^K_{k=0}$.
\STATE {\bfseries Initialize:} Let $\Omega^{(0)}=I_d$, $\Gamma^{(0)}_k=0$ and $Z^{(0)}_k=I_d$ for all $0 \leq k \leq K$.
Let $r^{\text{pri}}=r^{\text{dual}}=\infty$ and $\epsilon^{\text{pri}}=\epsilon^{\text{dual}}=0$. Let $t = 0$.
\WHILE{$r^{\text{pri}} > \epsilon^{\text{pri}}$ or $r^{\text{dual}} > \epsilon^{\text{dual}}$}{
\STATE $t \leftarrow t+1$.
\FOR{$k=0,1,\ldots,K$}{
\STATE Let
\begin{equation*}
\tilde{C}^{(t-1)}_k = - Z^{(t-1)}_k + \left( \Omega^{(t-1)} + \Gamma^{(t-1)}_k \right) - ( \alpha_k / \rho) \widehat{\Sigma}^{(k)},
\end{equation*}
and compute the eigenvalue decomposition of $\rho \tilde{C}^{(t-1)}_k$ as
\begin{equation*}
\rho \tilde{C}^{(t-1)}_k = U \Lambda U^{\top}, \quad \Lambda = \text{diag} \left( \left\{ \lambda_i \right\}^d_{i=1} \right).
\end{equation*}
\STATE Let
\begin{equation*}
\Omega^{(t)}_k = U \text{diag} \left( \left\{ \frac{ \lambda_i + \sqrt{  \lambda^2_i + 4 \rho \alpha_k }  }{ 2 \rho } \right\}^d_{i=1}  \right) U^{\top}.
\end{equation*}
\STATE Let
\begin{equation*}
\check{C}^{(t)}_k = \Omega^{(t)}_k +  Z^{(t-1)}_k.
\end{equation*}
}
\ENDFOR
\STATE Solve 
\begin{equation*}
\Omega^{(t)}_{jl}, \left\{ \Gamma^{(t)}_{k,jl}  \right\}^K_{k=1} = \mathcal{S} \left( \left\{ \check{C}^{(t)}_{k,jl} \right\}^K_{k=1} \right) \quad \text{for all } 1 \leq j,l \leq d.
\end{equation*}
by Algorithm~\ref{alg:admm-mt-glasso}.
\STATE Let
\begin{equation*}
Z^{(t)}_k = Z^{(t-1)}_k + \rho \left( \Omega^{(t)}_k - \Omega^{(t)} - \Gamma^{(t)}_k \right) \text{for all } 0 \leq k \leq K.
\end{equation*}
\STATE Let
\begin{align*}
r^{\text{pri}} &= \left( \sum^K_{k=0} \left\Vert
\Omega^{(t)}_k - \left( \Omega^{(t)} + \Gamma^{(t)}_k \right) \right\Vert^2_{\mathrm{F}}  \right)^{\frac{1}{2}}, \\
r^{\text{dual}} &= \rho \left( \sum^K_{k=0} \left\Vert \left( \Omega^{(t)} + \Gamma^{(t)}_k \right) - \left( \Omega^{(t-1)} + \Gamma^{(t-1)}_k \right) \right\Vert^2_{\mathrm{F}}  \right)^{\frac{1}{2}},
\end{align*}
and
\begin{align*}
\epsilon^{\text{pri}} &=  \epsilon^{\text{abs}} d \sqrt{K+1} +  \epsilon^{\text{rel}} \max \left\{ \left( \sum^K_{k=0} \left\Vert \Omega^{(t)}_k \right\Vert^2_{\mathrm{F}} \right)^{\frac{1}{2}} \, , \, \left( \sum^K_{k=0} \left\Vert \Omega^{(t)} + \Gamma^{(t)}_k
 \right\Vert^2_{\mathrm{F}} \right)^{\frac{1}{2}} \right\}, \\
\epsilon^{\text{dual}} &=  \epsilon^{\text{abs}} d \sqrt{K+1} + \epsilon^{\text{rel}} \left( \sum^K_{k=0} \left\Vert Z^{(t)}_k \right\Vert^2_{\mathrm{F}} \right)^{\frac{1}{2}} .
\end{align*}
}
\ENDWHILE
\STATE{\bfseries Output:} $\steponeoutput=\Omega^{(t)}+\Gamma^{(t)}_k$ for all $k=0,1,\ldots,K$.
\end{algorithmic}
\end{algorithm}

The final ADMM algorithm and the algorithm to solve the subproblem are summarized in Algorithm~\ref{alg:admm-mt-glasso} and Algorithm~\ref{alg:admm-mt-glasso-sub-problem}. 

For the proximal gradient descent algorithm to solve D-Trace loss objective~\eqref{eq:diff-network-est-Dtrace}, the stopping criterion is introduced as below, and the detailed algorithm is described in Algorithm~\ref{alg:prox-grad-D-trace}.

\begin{algorithm}[t]
\caption{Solver of sub problem~\eqref{eq:ADMM-update-2-sub}}
\label{alg:admm-mt-glasso-sub-problem}
\begin{algorithmic}[1]
\STATE {\bfseries Input:} $c=(c_0,c_1,\ldots,c_K)$; Initial $x^{(0)}$ and $\{ y^{(0)}_k \}^K_{k=0}$; $\lambda_{\text{M}}$, $\rho$ and $\{\alpha_k\}^K_{k=0}$.
\STATE {\bfseries Initialize:} $r^{\text{sub}}=\infty$ and $\epsilon^{\text{sub}}=0$. Define $\text{ST}_{\lambda}(\cdot)$ as in~\eqref{eq:soft-threshod-function}.
Let $r = 0$.
\WHILE{$r^{\text{sub}} > \epsilon^{\text{sub}}$}{
\STATE $r \leftarrow r+1$.
\STATE Let
\begin{equation*}
x^{(r)} = \text{ST}_{\lambda_{\text{M}} / (\rho (K+1)) } \left( \frac{1}{K+1} \sum^K_{k=0} \left(  c_k - y^{(r-1)}_k \right)  \right).
\end{equation*}
\STATE Let
\begin{equation*}
y^{(r)}_k = \text{ST}_{ \lambda_{\text{M}} \sqrt{\alpha_k} / \rho} \left(  c_k - x^{(r)}  \right) \quad \text{for all } 0 \leq k \leq K.
\end{equation*}
\STATE Let
\begin{align*}
r^{\text{sub}} &= \rho \left\vert \sum^K_{k=0} \left( y^{(r)}_k - y^{(r-1)}_k \right) \right\vert, \\
\epsilon^{\text{sub}} &= (K+1) \epsilon^{\text{abs}} + \epsilon^{\text{rel}} \rho \max \left\{ \sum^K_{k=0} \left\vert y^{(r)}_k \right\vert \, , \, \sum^K_{k=0} \left\vert y^{(r-1)}_k \right\vert \right\}.
\end{align*}
}
\ENDWHILE
\STATE {\bfseries Output:} $x^{(r)}$ and $\{ y^{(r)}_k \}^K_{k=0}$.
\end{algorithmic}
\end{algorithm}

\paragraph{Stopping criterion for proximal gradient.} 
If $\Gamma^{(t)}$ is the solution for~\eqref{eq:diff-network-est-Dtrace}, the optimization criterion requires that
\begin{equation*}
0 \in \nabla L_D (\Gamma^{(t)}) + \lambda^{(k)}_{\Gamma} \cdot 
 \partial \left\vert \Gamma^{(t)} \right\vert_1.
\end{equation*}
Note that by the definition of $\Gamma^{(t)}$ in~\eqref{eq:iterorign}, we have
\begin{equation}
\label{eq:opt-cond-proxiaml-grad}
0 \in \Gamma^{(t)} - \left( \Gamma^{(t-1)}-\eta\nabla L_D \left(\Gamma^{(t-1)}\right) \right) + \eta \lambda^{(k)}_{\Gamma} \cdot 
 \partial \left\vert \Gamma^{(t)} \right\vert_1,
\end{equation}
which implies that
\begin{equation*}
0 \in \nabla L_D (\Gamma^{(t)}) + \lambda^{(k)}_{\Gamma} \cdot 
 \partial \left\vert \Gamma^{(t)} \right\vert_1 + \frac{1}{\eta} \left( \Gamma^{(t)} - \Gamma^{(t-1)} \right) - \left[ \nabla L_D \left(\Gamma^{(t)}\right) - \nabla L_D \left(\Gamma^{(t-1)}\right) \right].
\end{equation*}
Thus, when
\begin{align*}
r^D &= \left\Vert \frac{1}{\eta} \left( \Gamma^{(t)} - \Gamma^{(t-1)} \right) - \left[ \nabla L_D \left(\Gamma^{(t)}\right) - \nabla L_D \left(\Gamma^{(t-1)}\right) \right] \right\Vert_{\mathrm{F}} \\
&= \left\Vert \frac{1}{\eta} \left( \Gamma^{(t)} - \Gamma^{(t-1)} \right) - \frac{1}{2} \widehat{\Sigma}^{(0)} \left( \Gamma^{(t)} - \Gamma^{(t-1)} \right) \widehat{\Sigma}^{(k)} - \frac{1}{2} \widehat{\Sigma}^{(k)} \left( \Gamma^{(t)} - \Gamma^{(t-1)} \right) \widehat{\Sigma}^{(0)} \right\Vert_{\mathrm{F}}
\end{align*}
is close to zero, we then have~\eqref{eq:opt-cond-proxiaml-grad} hold approximately. 
Therefore, given an absolute tolerance $\epsilon^{\text{abs}}>0$ and a relative tolerance $\epsilon^{\text{rel}}>0$, we then stop at iteration $t$ if
\begin{equation*}
r^D \leq \epsilon^{\text{abs}} d + \epsilon^{\text{rel}} \times  \max \left\{ \frac{1}{\eta} \left\Vert  \Gamma^{(t)} - \Gamma^{(t-1)} \right\Vert_{\mathrm{F}}, \left\Vert \frac{1}{2} \widehat{\Sigma}^{(0)} \left( \Gamma^{(t)} - \Gamma^{(t-1)} \right) \widehat{\Sigma}^{(k)} \right\Vert_{\mathrm{F}} \right\} .
\end{equation*}

\begin{algorithm}[t]
\caption{Proximal Gradient Method for D-Trace Loss}
\label{alg:prox-grad-D-trace}
\begin{algorithmic}[1]
\STATE {\bfseries Input:} $\widehat{\Sigma}^{(0)}$, $\widehat{\Sigma}^{(k)}$, $\lambda^{(k)}_{\Gamma}$, $\eta$, $\epsilon^{\text{abs}}$ and $\epsilon^{\text{rel}}$.
\STATE {\bfseries Initialize:} $\Gamma^{(0)}=0$, $r^D=\infty$, $\epsilon^D =0$ and $t = 0$.
\WHILE{$r^D > \epsilon^D$}{
\STATE $t \leftarrow t + 1$
\STATE Let
\begin{align*}
A^{(t-1)} &= \Gamma^{(t-1)}-\eta\nabla L_D \left(\Gamma^{(t-1)}\right) \\
&= \Gamma^{(t-1)} - \eta \left\{ \frac{1}{2} \left( \widehat{\Sigma}^{(k)} \Gamma^{(t-1)} \widehat{\Sigma}^{(0)} + \widehat{\Sigma}^{(0)} \Gamma^{(t-1)} \widehat{\Sigma}^{(k)} \right) - \left( \widehat{\Sigma}^{(0)} - \widehat{\Sigma}^{(k)} \right) \right\}.
\end{align*}
\STATE Let
\begin{equation*}
\Gamma^{(t)}_{jl} = \left[ \vert A^{(t-1)}_{jl} \vert - \lambda^{(k)}_{\Gamma} \eta \right]_{+} \cdot A^{(t-1)}_{jl} / \vert A^{(t-1)}_{jl} \vert, \qquad 1 \leq j,l \leq d.
\end{equation*}
\STATE Let
\begin{equation*}
r^D = \left\Vert \frac{1}{\eta} \left( \Gamma^{(t)} - \Gamma^{(t-1)} \right) - \frac{1}{2} \widehat{\Sigma}^{(0)} \left( \Gamma^{(t)} - \Gamma^{(t-1)} \right) \widehat{\Sigma}^{(k)} - \frac{1}{2} \widehat{\Sigma}^{(k)} \left( \Gamma^{(t)} - \Gamma^{(t-1)} \right) \widehat{\Sigma}^{(0)} \right\Vert_{\mathrm{F}},
\end{equation*}
and
\begin{equation*}
\epsilon^D = \epsilon^{\text{abs}} d + \epsilon^{\text{rel}} \times  \max \left\{ \frac{1}{\eta} \left\Vert  \Gamma^{(t)} - \Gamma^{(t-1)} \right\Vert_{\mathrm{F}}, \left\Vert  \frac{1}{2}\widehat{\Sigma}^{(0)} \left( \Gamma^{(t)} - \Gamma^{(t-1)} \right) \widehat{\Sigma}^{(k)} \right\Vert_{\mathrm{F}} \right\} .
\end{equation*}
}
\ENDWHILE
\STATE {\bfseries Output:} $\diffestimate = \Gamma^{(t)}$.
\end{algorithmic}
\end{algorithm}

\clearpage

\section{Minimax Optimal Rate for Differential Network Estimation}
\label{sec:minimax-opt-diff-network}

Differential network estimation aims to estimate the difference between two precision matrices without first estimating the individual precision matrices. Although existing studies focus on providing upper bounds for this problem~\citep{zhao2014direct,yuan2017differential,zhao2022fudge}, there is no known matching lower bound, making the minimax optimal rate an open question.
As a byproduct of our analysis, in this section, we provide a minimax optimal rate for differential network estimation problem under certain conditions. To the best of our knowledge, this is the first minimax optimal guarantee towards this direction.

We start by formulating the problem setup.
Note that we reintroduce some notation used in this section to make it self-contained, and one should not confuse it with the notation used in the other parts of the paper.
Suppose that we have $n_X$ i.i.d. samples $X_1,\ldots,X_{n_X} \sim N(0,\Omega^{-1}_X)$ and $n_Y$ i.i.d. samples $Y_1,\ldots,Y_{n_Y} \sim N(0,\Omega^{-1}_Y)$. Let $\Delta \coloneqq \Omega_Y - \Omega_X$ be the differential network between $\Omega_X$ and $\Omega_Y$. Our goal is to use samples from two populations to estimate $\Delta$. In addition, we assume that $\Omega_X$ and $\Omega_Y$ belong to the following parameter space.
\begin{multline}
\label{eq:param-space-diff-network}
\left( \Omega_X, \Omega_Y \right) \in \mathcal{M} \coloneqq \left\{ \Omega_X,\Omega_Y \in \mathbb{S}^{d \times d} \, : \,  c_1 \leq \gamma_{\min}(\Omega_X) \leq \gamma_{\max}(\Omega_X) \leq c_2 , \right. \\
\left. c_1 \leq \gamma_{\min}(\Omega_Y) \leq \gamma_{\max}(\Omega_Y) \leq c_2, \, \left\vert \Delta \right\vert_0 \leq h, \, \left\vert \Delta \right\vert_1 \leq C_{\Gamma} \right\},
\end{multline}
where $c_1,c_2,C_{\Gamma}>0$ are positive universal constants. The parameter space defined in~\eqref{eq:param-space-diff-network} requires that the smallest eigenvalues and the largest eigenvalues of $\Omega_X$ and $\Omega_Y$ are lower bounded and upper bounded, respectively.
Besides, we also require that the sparsity level of the differential network is bounded by $h$, and the $L_1$ norm of the differential network is bounded by universal constant.

We first present the minimax lower bound for the above estimation problem. We state the result in the following theorem.

\begin{theorem}
\label{thm:diff-network-low-bd}
Assume that
\begin{equation}
d \geq 4h, \, h \log d \geq 8 \log 3, \, \frac{h \log d}{n} \leq \min \left\{ 2, 8 (1-c_1)^2, 8(1-c_2)^2 \right\}, \, h \sqrt{\frac{\log d}{n}} \leq 4 C_{\Gamma}.
\end{equation}
We then have
\begin{equation*}
\min_{\widehat{\Delta}} \sup_{ (\Omega_X, \Omega_Y) \in \mathcal{M} } \mathbb{E} \left[ \left\Vert \widehat{\Delta} - \Delta \right\Vert^2_{\mathrm{F}} \right] \gtrsim \frac{h \log d}{ \min \left\{ n_X, n_Y \right\} }
\end{equation*}
\end{theorem}
\begin{proof}
Let $\Omega_X=I_d$, note that
\begin{equation*}
\left( I_d, \Omega_Y \right) \in \mathcal{M} \Leftarrow \Omega_Y \in \mathcal{G}_2,
\end{equation*}
where $\mathcal{G}_2$ is defined in~\eqref{eq:low-bd-param-space-2}. Now that samples from population $X$ are useless for estimating $\Omega_Y$, and we can only rely on samples from population $Y$ to estimate $\Omega_Y$. By Lemma~\ref{lemma:lower-bd-diff-net}, we then have
\begin{equation*}
\min_{\widehat{\Delta}} \sup_{ (\Omega_X, \Omega_Y) \in \mathcal{M} } \mathbb{E} \left[ \left\Vert \widehat{\Delta} - \Delta \right\Vert^2_{\mathrm{F}} \right] \geq \min_{\widehat{\Omega}_Y} \sup_{\Omega_Y \in \mathcal{G}_2} \mathbb{E} \left[ \left\Vert \widehat{\Omega}_Y - \Omega_Y \right\Vert^2_{\mathrm{F}} \right] \gtrsim \frac{h \log d}{ n_Y }.
\end{equation*}
Similarly, we can show that
\begin{equation*}
\min_{\widehat{\Delta}} \sup_{ (\Omega_X, \Omega_Y) \in \mathcal{M} } \mathbb{E} \left[ \left\Vert \widehat{\Delta} - \Delta \right\Vert^2_{\mathrm{F}} \right] \gtrsim \frac{h \log d}{ n_X }.
\end{equation*}
Combine the above two inequalities, we have the final result.
\end{proof}

Next, we derive the matching upper bound. Let $\widehat{\Delta}$ be the D-Trace loss estimator defined in~\eqref{eq:diff-network-est-Dtrace}, and $\check{\Delta}(\tau)$ be the truncated version of $\widehat{\Delta}$ as defined in~\eqref{eq:proj-Gamma}.
Following directly from Theorem~\ref{thm:expected-frob-err-dtrace}, we then have the following theorem.
\begin{theorem}
\label{thm:diff-network-upp-bd}
Assume that $C_{\Gamma} \leq d^{\tau_3}$, where $C_{\Gamma}$ is the same universal constant used in~\eqref{eq:param-space-diff-network}. Let
\begin{equation*}
\lambda_{\Gamma} \asymp \sqrt{ \frac{ \log d }{ \min \left\{ n_X, n_Y \right\} } },
\end{equation*}
where $\lambda_{\Gamma}$ is the penalization parameter of D-Trace loss. Then for any $\tau \geq \tau_3$, we have
\begin{equation*}
\sup_{ (\Omega_X, \Omega_Y) \in \mathcal{M}} \mathbb{E} \left[ \left\Vert \check{\Delta}(\tau) - \Delta \right\Vert^2_{\mathrm{F}} \right] \lesssim \frac{h \log d}{ \min \left\{ n_X, n_Y \right\} }.
\end{equation*}
\end{theorem}

Combine Theorem~\ref{thm:diff-network-low-bd} and Theorem~\ref{thm:diff-network-upp-bd}, we then have the final main result of this section.
\begin{theorem}
\label{thm:diff-network-minimax-optimal}
Assume that the conditions of Theorem~\ref{thm:diff-network-low-bd} and Theorem~\ref{thm:diff-network-upp-bd} hold. We have
\begin{equation*}
\min_{\widehat{\Delta}} \sup_{ (\Omega_X, \Omega_Y) \in \mathcal{M} } \mathbb{E} \left[ \left\Vert \widehat{\Delta} - \Delta \right\Vert^2_{\mathrm{F}} \right] \asymp \frac{h \log d}{ \min \left\{ n_X, n_Y \right\} }.
\end{equation*}

\end{theorem}

\section{Upper Bound for Expected Error}
\label{sec:upp-bd-expe-err}

In this section, we develop theoretical guarantees for the expected error. We start with the analysis of the \myalgMT.
The subsequent theorem provides the expected error measured in the Frobenius norm.
\begin{theorem}
\label{thm:Frob-err-expec}
Suppose that Assumption~\ref{assump:model-structure} and Assumption~\ref{assump:eig-element-upp-bd} hold.
Assume that $2 \left( K + 2 \right) \leq d^{\tau_1}$ and $N \leq d^{\tau_2}$, for some universal constants $\tau_1,\tau_2>0$.
In addition, assume that
\begin{equation*}
\frac{ \log d }{ \min_{0 \leq k \leq K} n_k } \lesssim 1.
\end{equation*}
Let $\widebar{n}={N}/(K+1)$, $M_{\mathrm{op}} \geq M_{\Omega}$ and $M_{\mathrm{op}}=O(1)$.
If $\lambda_{\text{M}} \asymp \sqrt{ { \log d }/{ N } }$, we have
\begin{equation*}
\mathbb{E} \left[ \sum^K_{k=0} \alpha_k \left\Vert \widehat{\Omega}_{k} - \Omega^{(k)} \right\Vert^2_{\mathrm{F}} \right] \lesssim \left( \frac{ s }{ N } + \frac{ h }{ \widebar{n} } \right) \log d.
\end{equation*}
\end{theorem}

Refer to the proof in Appendix~\ref{sec:proof-thm-Frob-err-expec}. The rate described in Theorem~\ref{thm:Frob-err-expec} consists of two parts. The first part, which is of the order $(s \log d)/{ N }$, refers to the estimation of the shared component. The second part, of the order $( h \log d ) / { \widebar{n} }$, relates to the estimation of the individual components.

As discussed in Section~\ref{sec:thm-trans-glasso},  the differential network estimate $\diffestimate$ is considered the result of a black-box algorithm, with the presumption that its estimation errors are appropriately controlled.
Let $\mathcal{B}_{\mathrm{F}} ( R ) \coloneqq \left\{ A \in \mathbb{R}^{d \times d} : \left\Vert A \right\Vert_{\mathrm{F}} \leq R \right\}$ be the ball with radius $R$. For $\tau>0$ and $0 \leq k \leq K$, we define
\begin{equation}
\label{eq:proj-Gamma}
\diffproj (\tau) \in \arg \min_{ \Gamma \in \mathcal{B}_{\mathrm{F}} ( d^{\tau} ) } \left\Vert \Gamma - \diffestimate  \right\Vert_{\mathrm{F}} .
\end{equation}
$\diffproj(\tau)$ is utilized solely for theoretical reasons.
Suppose that by choosing $\tau$ appropriately, we have
\begin{equation}
\label{eq:diff-network-high-prob-err-Frob-expec}
\mathbb{E} \left[ \left\Vert \diffproj (\tau) - \diff \right\Vert^2_{\mathrm{F}} \right] \lesssim \, \widebar{g}^{(k)}_{\mathrm{F}} (  n_0, n_k, d, h, M_{\Gamma} ) \coloneqq \widebar{g}^{(k)}_{\mathrm{F}} \quad \text{for all} \ 0 \leq k \leq K.
\end{equation}

We define
\begin{equation}
\label{eq:proj-Omega-tilde}
\targetproj (\tau) = \sum^K_{k=0} \alpha_k \left( \steponeoutput - \diffproj (\tau) \right),
\end{equation}
$\diffproj (\tau)$ is defined in~\eqref{eq:proj-Gamma}.
Combining~\eqref{eq:diff-network-high-prob-err-Frob-expec} with Theorem~\ref{thm:Frob-err-expec} yields the following theorem with expected error upper bound for \myalg.
\begin{theorem}
\label{thm:expec-error-upp-general}
If the conditions of Theorem~\ref{thm:Frob-err-expec} are satisfied
and $\tau$ is chosen so that~\eqref{eq:diff-network-high-prob-err-Frob-expec} holds, we have
\begin{equation*}
\mathbb{E} \left[ \left\Vert \targetproj (\tau) - \Omega^{(0)} \right\Vert^2_{\mathrm{F}} \right] \lesssim \left( \frac{ s }{ N } + \frac{ h }{ \widebar{n} } \right) \log d + \sum^K_{k=0} \alpha_k \widebar{g}^{(k)}_{\mathrm{F}} \, .
\end{equation*}
\end{theorem}

We then characterize $\widebar{g}^{(k)}_{\mathrm{F}}$ in the case where the D-Trace loss estimator in~\eqref{eq:diff-network-est-Dtrace} is used.
Recall that $\diffproj (\tau)$ is defined by~\eqref{eq:diff-network-high-prob-err-Frob-expec}.
\begin{theorem}
\label{thm:expected-frob-err-dtrace}
Suppose that Assumption~\ref{assump:model-structure} and Assumption~\ref{assump:eig-element-upp-bd} hold.
Assume that $2 (K+1) \leq d^{\tau_1}$, $N \leq d^{\tau_2}$ and $M_{\Gamma} \leq d^{\tau_3}$ for some universal constants $\tau_1,\tau_2,\tau_3>0$.
Let
\begin{equation*}
\lambda^{(k)}_{\Gamma} \asymp M_{\Gamma} \sqrt{ \frac{ \log d }{ \min \left\{ n_k,n_0 \right\} } }.
\end{equation*}
Then, for any $\tau \geq \tau_3$, we have
\begin{equation*}
\mathbb{E} \left[ \left\Vert \diffproj (\tau) - \diff \right\Vert^2_{\mathrm{F}} \right] \lesssim \frac{ M^2_{\Gamma} \, h \log d }{ \min \left\{ n_k,n_0 \right\} }.
\end{equation*}
\end{theorem}
\begin{proof}
Note that $\Vert A \Vert_{\mathrm{F}} \leq \vert A \vert_1$ for any matrix $A$. 
Thus we have $\Vert \diff \Vert_{\mathrm{F}} \leq \Vert \diff \Vert_1  \leq 2 M_{\Gamma}$.
The rest of the proof is similar to the proof of Theorem~\ref{thm:Frob-err-expec}.
\end{proof}

By Theorem~\ref{thm:expected-frob-err-dtrace}, we have
\begin{equation*}
\widebar{g}^{(k)}_{\mathrm{F}} =  \frac{ M^2_{\Gamma} \, h \log d }{ \min \left\{ n_k,n_0 \right\} }
\end{equation*}
for D-Trace loss estimator. Plugging the above results into Theorem~\ref{thm:expec-error-upp-general}, we then have the following corollary.

\begin{corollary}
\label{corollary:convg-Trans-GLasso-Frob-expec-error}
Let $\widehat{\Omega}^{(0)}$ be obtained by \myalg~\eqref{eq:Trans-GLasso} with the D-Trace loss estimator used in Step~1. 
Suppose that Assumption~\ref{assump:model-structure} and Assumption~\ref{assump:eig-element-upp-bd} hold, and that the conditions in Theorem~\ref{thm:Frob-err-expec} and Theorem~\ref{thm:expected-frob-err-dtrace} are satisfied.
If $2 \left( K + 2 \right) \leq d^{\tau_1}$, $N \leq d^{\tau_2}$, $M_{\Gamma} \leq d^{\tau_3}$ for some universal constants $\tau_1,\tau_2,\tau_3>0$, and 
\begin{equation*}
\quad \lambda_{\text{M}} \asymp \sqrt{ \frac{ \log d }{ N } }, \quad \lambda^{(k)}_{\Gamma} \asymp M_{\Gamma} \sqrt{ \frac{ \log d }{ \min \left\{ n_k,n_0 \right\} } } \quad \text{ for all } k \in [K],
\end{equation*}
then for any $\tau \geq \tau_3$,
we have
\begin{equation}
\label{eq:expected-error-trans-glasso}
\mathbb{E} \left[ \left\Vert \targetproj (\tau) - \Omega^{(0)} \right\Vert^2_{\mathrm{F}} \right] \lesssim\left( \frac{ s }{ N } + (1 + M^2_{\Gamma}) \cdot \frac{ h }{ \widebar{n} } + M^2_{\Gamma} \cdot \frac{h}{n_0} \right) \log d \, .
\end{equation}
\end{corollary}
The estimation error in~\eqref{eq:expected-error-trans-glasso} is comprised of three parts: shared component estimation, individual component estimation, and differential network estimation. If $\widebar{n} \geq n_0$ and $M_{\Gamma}$ is bounded by a universal constant, the error scales as $\frac{s \log d}{N} + \frac{h \log d}{n_0}$. When using only target samples, the lowest error rate achievable is $\frac{ (s+h) \log d}{n_0}$ as stated in Lemma~\ref{lemma:low-bd-shared}. Therefore, if $N \gg n_0$, the error rate can be significantly reduced compared to the optimal rate obtained with only the target samples. Moreover, as demonstrated in Section~\ref{sec:thm-lower-bd}, the rate $\frac{s \log d}{N} + \frac{h \log d}{n_0}$ is minimax optimal under certain conditions.

\section{Proof of Theorem~\ref{thm:Frob-err-expec}}
\label{sec:proof-thm-Frob-err-expec}

Note that for any matrix $A \in \mathbb{R}^{d \times d}$, we have $\left\Vert A \right\Vert_{\mathrm{F}} \leq \sqrt{d} \left\Vert A \right\Vert_2$.
Since by our assumptions that $\left\Vert \Omega^{(k)} \right\Vert_2 = O(1)$ and $\left\Vert \steponeoutput \right\Vert_2 = O(1)$, we thus have
\begin{equation*}
\left\Vert \Omega^{(k)} \right\Vert_{\mathrm{F}} = O \left( d^{\frac{1}{2}} \right) \quad\text{and}\quad \left\Vert \steponeoutput \right\Vert_{\mathrm{F}} = O \left( d^{\frac{1}{2}} \right).
\end{equation*}


For $\delta \in (0,1]$, let
\begin{equation*}
\lambda_{\text{M}} \geq C_1C_3 \sqrt{ \frac{ \log (2 \left( K + 2 \right) d^2 / \delta) }{ 2 N } }, 
\end{equation*}
where $C_1 = 160$ and $C_3=M_{\Sigma}$.
Besides, by the proof of Theorem~\ref{thm:multi-task-frob}, 
when
\begin{equation*}
\frac{ \log (2 \left( K + 2 \right) d^2 / \delta) }{ \min_{0 \leq k \leq K} n_k } \leq \frac{1}{2},
\end{equation*}
then by~\eqref{eq:proof-thm-frob-9}, we have $\mathbb{P}\left\{\mathbb{G}(\lambda_{\text{M}})\right\} \geq 1 - \delta$; or equivalently, we have $\mathbb{P}\left\{\widebar{\mathbb{G}}(\lambda_{\text{M}})\right\} \leq \delta$, where $\widebar{\mathbb{G}}(\lambda_{\text{M}})$ denotes the event that $\mathbb{G}(\lambda_{\text{M}})$ does not hold.

Recall that by assumption we have $2 \left( K + 2 \right) \leq d^{\tau_1}$. Let $\delta=d^{-\tau^{\prime}}$, where $\tau^{\prime}$ will be specified later, then
by letting
\begin{equation}
\label{eq:proof-thm-Frob-err-expec-3}
\lambda_{\text{M}} = C_1 C_3 \sqrt{ \frac{ (\tau^{\prime} + \tau_1 + 2) \log d }{ 2 N } },
\end{equation}
\begin{equation}
\label{eq:proof-thm-Frob-err-expec-2}
\frac{ (\tau^{\prime} + \tau_1 + 2) \log d }{ \min_{0 \leq k \leq K} n_k } \leq \frac{1}{2},
\end{equation}
we have $\mathbb{P}\left\{\widebar{\mathbb{G}}(\lambda_{\text{M}})\right\} \leq d^{-\tau^{\prime}}$. 
Besides, by Section~\ref{sec:remaining-proof}, when~\eqref{eq:proof-thm-Frob-err-expec-3}--\eqref{eq:proof-thm-Frob-err-expec-2} are true, $\mathbb{G}(\lambda_{\text{M}})$ then implies that
\begin{align*}
\sum^K_{k=0} \alpha_k \left\Vert \steponeoutput - \Omega^{(k)} \right\Vert^2_{\mathrm{F}} & \leq \frac{9 \left( s +  ( K+1 ) h \right)\lambda^2_{\text{M}}}{4 \kappa^2} \\
& \leq \frac{9  C^2_1 C^2_3 (\tau^{\prime} + \tau_1 + 2) }{8} \left( \frac{ s }{ N } + \frac{ h }{ \widebar{n} } \right) \log d.
\end{align*}

Note that
\begin{equation}
\label{eq:proof-thm-Frob-err-expec-4}
\begin{aligned}
& \quad \mathbb{E} \left[ \sum^K_{k=0} \alpha_k \left\Vert \steponeoutput - \Omega^{(k)} \right\Vert^2_{\mathrm{F}} \right] \\
& = \mathbb{P} \left\{ \mathbb{G}(\lambda_{\text{M}}) \right\} \mathbb{E} \left[ \left. \sum^K_{k=0} \alpha_k \left\Vert \steponeoutput - \Omega^{(k)} \right\Vert^2_{\mathrm{F}} \right\vert \mathbb{G}(\lambda_{\text{M}}) \right] + \mathbb{P} \left\{ \widebar{\mathbb{G}}(\lambda_{\text{M}}) \right\} \mathbb{E} \left[ \left. \sum^K_{k=0} \alpha_k \left\Vert \steponeoutput - \Omega^{(k)} \right\Vert^2_{\mathrm{F}} \right\vert \widebar{\mathbb{G}}(\lambda_{\text{M}}) \right].
\end{aligned}
\end{equation}
Given~\eqref{eq:proof-thm-Frob-err-expec-3}--\eqref{eq:proof-thm-Frob-err-expec-2} are true, we have
\begin{equation}
\label{eq:proof-thm-Frob-err-expec-5}
\mathbb{E} \left[ \left. \sum^K_{k=0} \alpha_k \left\Vert \steponeoutput - \Omega^{(k)} \right\Vert^2_{\mathrm{F}} \right\vert \mathbb{G}(\lambda_{\text{M}}) \right] \leq \frac{9  C^2_1 C^2_3 (\tau^{\prime} + \tau_1 + 2) }{8} \left( \frac{ s }{ N } + \frac{ h }{ \widebar{n} } \right) \log d.
\end{equation}
Besides, since $\Vert \steponeoutput \Vert_{\mathrm{F}}, \Vert \Omega^{(k)} \Vert_{\mathrm{F}} = O (d^{\frac{1}{2}})$ for all $0 \leq k \leq K$, we have
\begin{equation*}
\sum^K_{k=0} \alpha_k \left\Vert \steponeoutput - \Omega^{(k)} \right\Vert^2_{\mathrm{F}} \leq 
C^{\prime} d,
\end{equation*}
for some constant $C^{\prime} > 0$,
and thus we have
\begin{equation*}
\mathbb{P} \left\{ \widebar{\mathbb{G}}(\lambda_{\text{M}}) \right\} \mathbb{E} \left[ \left. \sum^K_{k=0} \alpha_k \left\Vert \steponeoutput - \Omega^{(k)} \right\Vert^2_{\mathrm{F}} \right\vert \widebar{\mathbb{G}}(\lambda_{\text{M}}) \right] \leq C^{\prime} d \cdot d^{-\tau^{\prime}} = C^{\prime} d^{-(\tau^{\prime}-1)}.
\end{equation*}
By Assumption that we have $N \leq d^{\tau_2}$ where $\tau_2>0$, then when we choose $\tau^{\prime}$ such that
\begin{equation*}
\tau^{\prime} \geq \tau_2 + 1 + \frac{ \log \left( 
8 C^{\prime}  / 9  C^2_1 C^2_3 (\tau^{\prime} + \tau_1 + 2) \right) }{ \log d },
\end{equation*}
we then have
\begin{align*}
C^{\prime} d^{-(\tau^{\prime}-1)} & \leq \frac{9  C^2_1 C^2_3 (\tau^{\prime} + \tau_1 + 2) }{8} \cdot \frac{ 1 }{ N } \\
& \leq \frac{9  C^2_1 C^2_3  (\tau^{\prime} + \tau_1 + 2) }{8} \left( \frac{ s }{ N } + \frac{ h }{ \widebar{n} } \right) \log d,
\end{align*}
which then implies that
\begin{equation*}
\mathbb{P} \left\{ \widebar{\mathbb{G}}(\lambda_{\text{M}}) \right\} \mathbb{E} \left[ \left. \sum^K_{k=0} \alpha_k \left\Vert \steponeoutput - \Omega^{(k)} \right\Vert^2_{\mathrm{F}} \right\vert \widebar{\mathbb{G}}(\lambda_{\text{M}}) \right] \leq \mathbb{E} \left[ \left. \sum^K_{k=0} \alpha_k \left\Vert \steponeoutput - \Omega^{(k)} \right\Vert^2_{\mathrm{F}} \right\vert \mathbb{G}(\lambda_{\text{M}}) \right].
\end{equation*}
Combine the above inequality with~\eqref{eq:proof-thm-Frob-err-expec-4} and~\eqref{eq:proof-thm-Frob-err-expec-5}, we finally have
\begin{align*}
\quad \mathbb{E} \left[ \sum^K_{k=0} \alpha_k \left\Vert \steponeoutput - \Omega^{(k)} \right\Vert^2_{\mathrm{F}} \right] & \leq 2 \mathbb{E} \left[ \left. \sum^K_{k=0} \alpha_k \left\Vert \steponeoutput - \Omega^{(k)} \right\Vert^2_{\mathrm{F}} \right\vert \mathbb{G}(\lambda_{\text{M}}) \right] \\
& \leq \frac{18  C^2_1 C^2_3 (\tau^{\prime} + \tau_1 + 2) }{8} \left( \frac{ s }{ N } + \frac{ h }{ \widebar{n} } \right) \log d \\
& \lesssim \left( \frac{ s }{ N } + \frac{ h }{ \widebar{n} } \right) \log d .
\end{align*}

{\newText

\section{Graph Recovery Guarantee for Gaussian Graphical Models}

While this paper primarily focuses on precision matrix estimation, it is also of great interest to explore the implications of \myalg for graph recovery in Gaussian graphical models (GGMs), given the close relationship between the precision matrix and the underlying graph structure. In this section, we introduce a graph estimation procedure based on thresholding the output of \myalg, establish a theoretical guarantee for its accuracy, and demonstrate its advantage over graph estimators that rely solely on target data.

Let $E^{(0)} = \left[ E^{(0)}_{jl} \right]_{1 \leq j,l \leq d} \in \{0,1\}^{d \times d}$ denote the adjacency matrix of the undirected Gaussian graphical model for the target data. Specifically, for any $j \neq l$, we set $E^{(0)}_{jl} = 1$ if nodes $j$ and $l$ are connected, and $E^{(0)}_{jl} = 0$ otherwise. By definition, $E^{(0)}_{jj} = 0$ for all $1 \leq j \leq d$. Given the well-established connection between GGMs and the precision matrix~\citep{lauritzen1996graphical}, we have, for any $j \neq l$,
\[
E^{(0)}_{jl} =
\begin{cases}
1 & \text{if } \Omega^{(0)}_{jl} \neq 0, \\
0 & \text{if } \Omega^{(0)}_{jl} = 0.
\end{cases}
\]

We propose a graph estimator based on thresholding the output of \myalg, $\widehat{\Omega}^{(0)}$. Specifically, let $\widehat{E}^{(0)} = \left[ \widehat{E}^{(0)}_{jl} \right]_{1 \leq j,l \leq d} \in \{0,1\}^{d \times d}$ denote the estimated adjacency matrix. We set $\widehat{E}^{(0)}_{jj} = 0$ for all $1 \leq j \leq d$, and for any $j \neq l$, define
\[
\widehat{E}^{(0)}_{jl} =
\begin{cases}
1 & \text{if } \vert \widehat{\Omega}^{(0)}_{jl} \vert > \epsilon_n, \\
0 & \text{if } \vert \widehat{\Omega}^{(0)}_{jl} \vert \leq \epsilon_n,
\end{cases}
\]
where $\epsilon_n$ is a thresholding parameter.

We measure the graph estimation error using the Hamming distance between $E^{(0)}$ and $\widehat{E}^{(0)}$, defined as
\begin{equation*}
\text{Err} \left( \widehat{E}^{(0)}  \right) \coloneqq \frac{1}{2} \sum_{1 \leq j,l \leq d} \mathds{1} \left\{ \widehat{E}^{(0)}_{jl} \neq E^{(0)}_{jl} \right\}.
\end{equation*}
Additionally, let $\omega_{\min} = \min_{(j,l): E^{(0)}_{jl}=1 } \left\vert \Omega^{(0)}_{jl} \right\vert >0$ denote the minimum signal strength in the graphical model. The following deterministic bound then holds for $\text{Err} \left( \widehat{E}^{(0)}  \right)$.
\begin{proposition}
\label{prop:sum-edge-error-bd}
If $0<\epsilon_n \leq \frac{1}{2}\omega_{\min}$, then
\begin{equation*}
\text{Err} \left( \widehat{E}^{(0)} \right) \leq \frac{1}{2 \epsilon^2_n} \left\Vert \widehat{\Omega}^{(0)} - \Omega^{(0)} \right\Vert^2_{\text{F}}.
\end{equation*}
Moreover, if the right-hand side is smaller than $1$, perfect graph recovery is achieved, i.e., $\widehat{E}^{(0)} = E^{(0)}$.
\end{proposition}
\begin{proof}
Let
\[
S_{\mathrm{I}} = \left\{ (j,l) \, : \, E^{(0)}_{jl}=0, \,\widehat{E}^{(0)}_{jl} = 1 \right\} \quad \text{and} \quad S_{\mathrm{II}} = \left\{ (j,l) \, : \, E^{(0)}_{jl}=1, \,\widehat{E}^{(0)}_{jl} = 0 \right\}
\]
denote the sets of type-I error and type-II error. We then have
\begin{equation*}
\text{Err} \left( \widehat{E}^{(0)}  \right) = \frac{1}{2} \sum_{1 \leq j,l \leq d} \mathds{1} \left\{ \widehat{E}^{(0)
}_{jl} \neq E_{jl} \right\} = \frac{1}{2} \left( \left\vert S_{\mathrm{I}} \right\vert + \left\vert S_{\mathrm{II}} \right\vert \right) .
\end{equation*}

Note that for any $(j,l) \in S_{\mathrm{I}}$, we have
\begin{equation}
\label{eq:proof-prop-sum-edge-error-bd-1}
\left\vert \widehat{\Omega}^{(0)}_{jl} - \Omega^{(0)}_{jl} \right\vert = \left\vert \widehat{\Omega}^{(0)}_{jl} \right\vert > \epsilon_n.
\end{equation}
Besides, for any $(j,l) \in S_{\mathrm{II}}$, we have
\begin{equation}
\label{eq:proof-prop-sum-edge-error-bd-2}
\left\vert \widehat{\Omega}^{(0)}_{jl} - \Omega^{(0)}_{jl} \right\vert \geq \left\vert \Omega^{(0)}_{jl} \right\vert - \left\vert \widehat{\Omega}^{(0)}_{jl} \right\vert \geq \left\vert \Omega^{(0)}_{jl} \right\vert - \epsilon_n \geq \omega_{\min} - \epsilon_n \geq  \epsilon_n,
\end{equation}
where the last inequality follows from the assumption that $\epsilon_n \leq \frac{1}{2} \omega_{\min}$.

Combining~\eqref{eq:proof-prop-sum-edge-error-bd-1} and~\eqref{eq:proof-prop-sum-edge-error-bd-2}, we have
\begin{align*}
\left\Vert \widehat{\Omega}^{(0)} - \Omega^{(0)} \right\Vert^2_{\mathrm{F}} & \geq \sum_{(j,l) \in S_{\mathrm{I}}} \left(  \widehat{\Omega}^{(0)}_{jl} - \Omega^{(0)}_{jl} \right)^2 + \sum_{(j,l) \in S_{\mathrm{II}}} \left(  \widehat{\Omega}^{(0)}_{jl} - \Omega^{(0)}_{jl} \right)^2 \\
& \geq \epsilon^2_n \left( \left\vert S_{\mathrm{I}} \right\vert + \left\vert S_{\mathrm{II}} \right\vert \right) \\
& = 2 \epsilon^2_n \cdot \text{Err} \left( \widehat{E}^{(0)}  \right),
\end{align*}
which then implies the final result.
\end{proof}

To illustrate the advantage of \myalg in graph recovery over methods relying solely on target data, we consider a special case of the parameter space defined in~\eqref{eq:param-space}, where $\Omega^{(0)}_{jj} = 1$ for all $1 \leq j \leq d$. In this scenario, the sample-optimal graph recovery algorithm proposed by~\cite{misra2020information}, which uses only target data, requires
\begin{equation}
\label{eq:graph-recovery-sample-1}
n_0 \gg (s + h) \log d / \omega^2_{\min}
\end{equation}
to achieve perfect graph recovery.

In contrast, applying Proposition~\ref{prop:sum-edge-error-bd} and Corollary~\ref{corollary:convg-Trans-GLasso-Frob-Dtrace}, we choose $\epsilon_n$ such that $\epsilon_n \asymp \omega_{\min}$ and $\epsilon_n \leq \frac{1}{2}\omega_{\min}$. Given that the conditions of Corollary~\ref{corollary:convg-Trans-GLasso-Frob-Dtrace} hold, and assuming $\bar{n} \geq n_0$ and $K \leq d^{\tau}$ for some universal constant $\tau > 0$, perfect graph recovery requires only that
\begin{equation}
\label{eq:graph-recovery-sample-2}
n_0 \gg h \log d / \omega_{\min}^2 \quad \text{and} \quad N \gg s \log d / \omega_{\min}^2.
\end{equation}

Comparing~\eqref{eq:graph-recovery-sample-1} and~\eqref{eq:graph-recovery-sample-2}, we see that the target sample complexity requirement for $n_0$ in \myalg is significantly lower than that of methods using only target samples when $s \gg h$. This highlights the advantage of transfer learning in Gaussian graphical model estimation.

}

\clearpage
\putbib[boxinz-papers]
\end{bibunit}

\end{document}